\documentclass[11pt]{article}
\pdfoutput=1

\usepackage{mathrsfs}
\usepackage{amsmath,amssymb}
\usepackage{bm}
\usepackage{natbib}
\usepackage[usenames]{color}
\usepackage{amsthm}
\usepackage{mathtools}
\usepackage{multirow} 
\usepackage{enumitem}
\usepackage{graphicx}
\usepackage{subcaption}
\usepackage{booktabs}
\usepackage[colorlinks,
linkcolor=red,
anchorcolor=blue,
citecolor=blue
]{hyperref}

\usepackage{mylatexstyle}

\usepackage{setspace}
\usepackage[left=1in, right=1in, top=1in, bottom=1in]{geometry}

\usepackage{xcolor}

\ifdefined\final
\usepackage[disable]{todonotes}
\else
\usepackage[textsize=tiny]{todonotes}
\fi
\title{\huge One-Layer Transformers Provably Learn In-Context K-Nearest Neighbor Prediction with Chain-of-Thought}

\author
{
Lyumin Wu\thanks{
\scriptsize School of Computing \& Data Science, The University of Hong Kong; {\tt lyuminwu001@connect.hku.hk}}
\qquad
Yuan Cao\thanks{
\scriptsize School of Computing \& Data Science, 
	The University of Hong Kong;  {\tt yuancao@hku.hk}}
}

\date{}

\begin{document}

\title{\huge One-Layer Transformers Provably Learn In-Context K-Nearest Neighbor Prediction with Chain-of-Thought}

\maketitle

\begin{abstract}

Chain-of-Thought (CoT) enables transformers to solve complex reasoning tasks by generating intermediate reasoning steps. Despite strong empirical success, the mechanisms by which such reasoning abilities emerge during training remain poorly understood, particularly from a theoretical perspective. In this work, we study this question in a stylized yet fully analyzable setting: in-context $\mathrm{K}$-nearest-neighbor ($\mathrm{K}$-NN) prediction with a one-layer transformer. Earlier work has demonstrated that one-layer transformers can be trained to perform $1$-NN without CoT \citep{Li2024One} by having the softmax attention attend to the nearest neighbor in context. However, this result on $1$-NN cannot be extended to $\mathrm{K}$-NN for $K>1$, as softmax attention is much less naturally suited to attending to the 2nd through $K$-th nearest neighbors. We first give a rigorous negative result showing that this limitation already appears on a well-separated class of binary classification tasks for odd $K>1$. We then give an explicit construction of a one-layer transformer, and show that it can solve $\mathrm{K}$-NN via CoT reasoning. We further prove that, under a stylized training setup, gradient descent can recover this construction, thereby showing that transformers can acquire the capability to solve $\mathrm{K}$-NN through training. In addition, we show that the resulting trained model applies to a broader test class than that assumed in the training analysis. Our results shed light on how CoT expands the ability of transformers to learn and solve tasks that are otherwise hard to solve. 
\end{abstract}

\section{Introduction}
Transformer-based \citep{Vaswani2017Attention} large language models (LLMs) have achieved remarkable success across various tasks, including natural language processing \citep{Brown2020Language,Touvron2023Llama,Touvron2023Llama2}, image and video generation \citep{Liu2024Sora}, and multimodal learning \citep{Achiam2023Gpt,Singh2025Openai}. Notably, these models have also achieved state-of-the-art performance on complex reasoning tasks \citep{Kojima2022Large,Zhou2022Least,Chowdhery2023Palm,Zhong2024Evaluation}. A key factor driving this progress is Chain-of-Thought (CoT), which enables LLMs to generate intermediate reasoning steps autoregressively rather than producing only a final answer. This process enhances the models’ ability to solve complex problems. However, the mechanisms underlying CoT reasoning remain poorly understood, posing fundamental theoretical challenges.

Theoretical studies on transformers with CoT have recently progressed along three main directions: expressive power, statistical properties, and training theory. On the expressivity side, prior work shows that intermediate reasoning steps can expand the effective computational power of transformers, enabling shallow architectures to solve problems beyond direct-answer prediction and establishing reasoning length as a key computational resource \citep{Feng2023Towards,Merrill2023expressive,Li2024Chain,Amiri2025Lower}. On the statistical side, recent analyses formalize CoT as a structured decomposition of complex inference into simpler intermediate subproblems, leading to improvements in sample efficiency, robustness, and generalization \citep{Prystawski2023Why, Hu2024Unveiling,Li2025Training}. Complementing these results, a growing line of training-theoretic work characterizes how transformers learn multi-step reasoning capabilities under gradient-based optimization, focusing on tasks such as in-context linear regression, symbolic path finding, and synthetic state tracking \citep{Huang2025Transformers,Yang2025Multi,Huang2025Length}.

Among these directions, the training-theoretic perspective is most closely aligned with practical LLM development, as it aims to explain how transformers acquire CoT-style multi-step reasoning capabilities through gradient-based training. Recent work has made encouraging progress in this direction by analyzing several structured settings, but current theoretical understanding remains limited. Following this perspective, we investigate the in-context $\mathrm{K}$-nearest-neighbor ($\mathrm{K}$-NN) problem. Prior work has shown that one-layer transformers can be trained to perform in-context $1$-NN without CoT \citep{Li2024One}. However, this result does not directly extend to $\mathrm{K}$-NN with $K>1$. We prove a proposition showing that direct inference already fails for odd $K>1$, which motivates our study of the CoT setting. We then show that a one-layer transformer admits an explicit construction for solving the $\mathrm{K}$-NN problem via CoT by giving both an asymptotic expressivity guarantee and a finite-margin correctness statement. We further show that gradient descent can recover this construction, and that the resulting trained model applies to a broader test class than the one specified by the training distributional assumptions.

\vspace{0.1cm}
\noindent\textbf{Contributions.} \hspace{0.05cm} We summarize the main contributions of this work as follows.
\begin{itemize}[left=5mm, labelsep=5pt, itemsep=1pt, topsep=1pt]

\item We prove that a one-layer transformer without CoT cannot perform $\mathrm{K}$-NN via direct inference for odd $K>1$ (Proposition~\ref{prop: one layer gap}). In contrast, a one-layer transformer admits an explicit construction for solving the $\mathrm{K}$-NN problem via CoT (Theorem~\ref{thm:expressive power} and Proposition~\ref{prop:positive_margin}). Together, these results provide a concrete example, in a controlled theoretical setting, of how CoT can expand the computational capabilities of transformers.

\item We further provide a convergence analysis of the population loss for a one-layer transformer trained by gradient descent (Theorem~\ref{thm:training convergence}), showing that gradient descent recovers the construction in Theorem~\ref{thm:expressive power}. This extends the in-context 1-NN result of~\citet{Li2024One} to $\mathrm{K}$-NN prediction with CoT. Combining the training and construction analyses further shows that the resulting trained model applies beyond the distributional assumptions used in the training analysis.

\item Our work contributes to a more comprehensive theoretical understanding of how Transformers with CoT can perform complex tasks. Prior work has shown that CoT can implement gradient descent for in-context linear regression~\citep{Huang2025Transformers}, repeatedly apply transition rules to update a state~\citep{Huang2025Length}, and autoregressively trace paths between nodes in trees~\citep{Yang2025Multi}. In contrast, our work theoretically identifies and formalizes a distinct CoT mechanism based on iterative token selection and masking, thereby broadening the current theoretical understanding of CoT.



\item Our work introduces several proof ideas and techniques that may be of independent interest. We identify a mechanism where the conventionally non-sequential $\mathrm{K}$-NN task can be reformulated as a sequential procedure based on iterative token selection and masking, thereby enabling CoT-based prediction, as formalized in Lemma~\ref{lem:expressive power mechanism}. We also develop technical tools for analyzing softmax-attention training, including contradiction and compactness arguments that establish coefficient divergence (Appendix~\ref{sec:parameter divergence}) and drift and crossing arguments that characterize the precise limiting ratios among the coefficients (Appendix~\ref{sec:convergence of the population loss}). Together, these tools reveal how a model can simultaneously learn token selection and masking.

\end{itemize}

\textbf{Notation.} For a scalar $x$, let $\mathrm{sign}(x) = 1$ if $x > 0$, $0$ if $x = 0$, and $-1$ otherwise. For a vector $\vb\in\RR^d$ and a matrix $\Ab\in\RR^{m\times n}$, let $\|\vb\|$ denote the $\ell_2$-norm of $\vb$ and $\|\Ab\|$ the spectral norm of $\Ab$.
We write $\Ab_{:,j}\in\RR^m$ and $\Ab_{i,:}\in\RR^n$ for the $j$-th column and $i$-th row of $\Ab$, respectively.
The symbols $\mathbf{0}_n$ and $\mathbf{0}_{m\times n}$ denote the zero vector in $\RR^n$ 
and the zero matrix in $\RR^{m\times n}$, respectively. 
We denote by $\Ib_d$ the $d$-dimensional identity matrix, 
and by $\Ib$ the identity matrix when the dimension is clear. For an integer $n$, we define $[n] = \{ 1,2, \ldots, n \}$. 
For two sequences $\{a_n\}$ and $\{b_n\}$, we denote $a_n = O(b_n)$ if there exists a constant $C> 0$ such that $|a_n|\leq C |b_n|$ for all large enough $n$. We denote $a_n = \Omega(b_n)$ if $b_n = O(a_n)$. We say $a_n = \Theta(b_n)$ if both $a_n = O(b_n)$ and $a_n = \Omega(b_n)$ hold.

\section{Related Work}
\noindent \textbf{In-context learning.} In-context learning describes the ability of transformers to perform previously unseen tasks without updating their model parameters. \citet{Bai2023Transformers} demonstrated the in-context learning capacities of transformers by constructing multi-head ReLU transformers capable of performing a variety of learning methods, including ridge regression, lasso regression, generalized linear models, and learning with shallow neural networks. \citet{Zhang2024Trained} studied the gradient-flow dynamics of a linear self-attention model. \citet{Huang2024context} analyzed linear prediction tasks using a softmax-attention layer trained via gradient descent. \citet{Li2024One} investigated how transformers perform in-context learning when solving one-nearest-neighbor selection problems. \citet{zhang2024context} showed that the MLP component of a transformer block can learn to simulate the initialization step of in-context gradient descent. \citet{siyu2024training} investigated the gradient-flow dynamics of training multi-head softmax attention and showed that all heads exhibit distinct patterns and interact jointly to solve in-context linear regression.  \citet{zhang2026logistic} showed that softmax transformers can perform in-context logistic regression via normalized gradient descent by recurrently applying a trained self-attention layer.

\noindent \textbf{Transformers as more general algorithm learners.} Beyond standard in-context learning settings, a growing body of work studies how transformers learn more general algorithmic procedures. A line of work studies transformers on Markovian or Markov-like sequence models. \citet{Edelman2024evolution} empirically demonstrated that transformers trained on samples generated by a Markov chain in a bigram setting can develop statistical induction heads. \citet{Nichani2024How} proved that gradient descent can train transformers to learn a generalized induction head on latent causal structures. \citet{Chen2024Unveiling} analyzed a modified two-layer transformer on $n$-gram Markov data. \citet{shi2025towards} showed that gradient-trained softmax attention acts as a token selector for random-walk prediction. A complementary body of work examines optimization in more specialized transformer architectures and training regimes. This includes analyses of vision transformers and their sample complexity \citep{Jelassi2022Vision}, the role of adaptive optimizers in transformer training \citep{pan2023toward, zhang2020adaptive}, the mechanisms through which transformers acquire semantic or topical structure \citep{li2023transformers}, and the optimization of looped transformer architectures, where \citet{wu2026looped} proved that gradient descent on looped linear transformers with layer normalization converges to solutions implementing the power method for principal-component prediction. 
Other studies investigate the joint training dynamics of attention, decoder, and MLP components \citep{Tian2023Scan, tianjoma2024}, characterize the implicit bias of softmax attention by linking its optimization dynamics to max-margin or SVM-like solutions \citep{tarzanagh2023transformers, ataee2023max, Huang2025How}.
Finally, several recent results demonstrate the capability of trained softmax transformers to perform token selection and identify task-relevant structure: \citet{Wang2024Transformers} proved that a one-layer transformer trained by gradient descent can learn sparse token selection, while \citet{zhang2025group,zhang2026teacher} established optimal variable selection in group-sparse classification and teacher-student learning setups.

\noindent \textbf{Understanding Chain-of-Thought Reasoning.} Recently, several studies have sought to understand CoT from both experimental and theoretical perspectives. \citet{Sprague2025CoT} conducted a large-scale meta-analysis showing that CoT is most effective on mathematical, symbolic, and algorithmic tasks, whereas \citet{Stechly2024Chain} showed, in planning domains, that the gains from CoT rely heavily on problem-specific demonstrations, casting doubt on whether CoT yields robust procedural generalization. \citet{Liu2025Mind} identified task regimes in which CoT can systematically degrade performance. \citet{Chen2024Unlocking} proposed a reasoning-boundary framework to quantify CoT capability and analyze how different CoT strategies affect reasoning performance. \citet{Ton2025Understanding} developed an information-theoretic framework for evaluating the contribution of individual reasoning steps and diagnosing CoT failure modes without requiring annotated intermediate traces. \citet{Hou2023Towards,Cabannes2024Iteration} investigated CoT through the lens of mechanistic interpretability. On the theoretical side, most existing studies have examined how CoT enhances transformer expressiveness \citep{Feng2023Towards,Merrill2023expressive,Li2024Chain,Amiri2025Lower} and have explored CoT through its statistical properties \citep{Prystawski2023Why,Hu2024Unveiling,Li2025Training}. Other works have characterized the intrinsic limits of CoT-based reasoning, showing that transformers face inherent barriers to compositional computation and that some problems provably require long CoT traces \citep{Peng2024limitations}. Complementing these results, \citet{Huang2025Transformers,Yang2025Multi,Huang2025Length} characterized how transformers learn multi-step reasoning capabilities under gradient-based optimization, focusing on tasks such as in-context linear regression, symbolic path finding, and synthetic state tracking.
\section{Preliminaries}
\label{sec:preliminaries}
In this section, we first define the $\mathrm{K}$-NN problem and introduce the transformer architecture studied in this paper. Based on these definitions, we then show that direct inference by a one-layer transformer already fails on a well-separated class of inputs for odd $K>1$.


\subsection{\texorpdfstring{$\mathrm{K}$}{K}-Nearest Neighbor Problem}
\label{subsec:knn_problem}
Given a dataset \(\mathcal{D}=\{(\xb_i,y_i)\}_{i=1}^N\), where \(\xb_i\in\RR^d\) is the input feature vector and \(y_i\in\mathcal{Y}\) is the corresponding label, and a query instance \(\xb_q\in\RR^d\), we define the $\mathrm{K}$-NN problem as follows.

\begin{definition}[\(\mathrm{K}\)-nearest neighbor ordering and classification]
\label{def:knn}
Given a query \(\mathbf{x}_q\) and a dataset \(\mathcal{D}=\{(\mathbf{x}_i,y_i)\}_{i=1}^N\), a permutation \(\pi:[N]\to[N]\) is said to order the dataset points by increasing distance to \(\mathbf{x}_q\) if
\[
\|\mathbf{x}_q-\mathbf{x}_{\pi(1)}\|
<
\|\mathbf{x}_q-\mathbf{x}_{\pi(2)}\|
<
\cdots
<
\|\mathbf{x}_q-\mathbf{x}_{\pi(N)}\|.
\]
Then \(\xb_{\pi(k)}\) is the \(k\)-th nearest neighbor of \(\xb_q\). The $\mathrm{K}$-NN problem is to predict the label of \(\xb_q\) from the labels of its top-\(K\) nearest neighbors. Formally, the $\mathrm{K}$-NN target is defined as
$
\arg\max_{c\in\mathcal{Y}}
\sum_{k=1}^K \mathbf{1}\{y_{\pi(k)}=c\},
$
namely the majority label among the top-\(K\) nearest neighbors.
\end{definition}
In the definition above, we implicitly assume that \(\|\mathbf{x}_q-\mathbf{x}_i\| \neq \|\mathbf{x}_q-\mathbf{x}_j\|\) for all \(i\neq j\), so that the ordering permutation \(\pi\) is well defined. This assumption is satisfied almost surely whenever \((\mathbf{x}_1,\dots,\mathbf{x}_N,\mathbf{x}_q)\) are independent and continuously distributed, because the event of a distance tie then has probability zero. 
In this work, we focus on the binary-label case \(\mathcal Y=\{\pm1\}\). In this case, the \(\mathrm K\)-NN target reduces to
$
\sign(\sum_{k=1}^K y_{\pi(k)}),
$
where we take \(\sign(0)=0\) to represent abstention when the top-\(K\) set has no strict majority.

\subsection{Transformer Architecture}
\label{subsec:transformer_arch}
In this subsection, we introduce the stochastic positional encoding module, the input embedding matrix, and the transformer model considered in our analysis. We also give a negative result for using transformers without CoT to solve $\mathrm{K}$-NN tasks.

\noindent\textbf{Stochastic positional encodings.} Let $\{e_j\}_{j=1}^{N+1}$ denote the canonical basis of $\RR^{N+1}$. During each training epoch and at test time, a new set of stochastic positional encodings is generated as follows:
\begin{enumerate}[leftmargin = *, nosep]
\item a uniformly random permutation $\tau \sim \mathrm{Unif}(S_{N+1})$;
\item independent Rademacher signs $s_1, \ldots, s_{N+1} \overset{\mathrm{i.i.d.}}{\sim} \mathrm{Rad}(\pm 1)$.
\end{enumerate}
Here, $S_{N+1}$ denotes the set of all permutations of $[N+1]$, so that $\tau(i)$ specifies the permuted position assigned to the token $i$. We then define the one-hot signed positional vectors by
\begin{equation*}
\pb_i \coloneqq s_i\, e_{\tau(i)} \quad (i \in [N+1]).
\end{equation*}

Stochastic positional encodings have been employed in several prior works~\cite{Shen2023Positional,Wang2024Transformers,Dwivedi2026Relational}. Most closely related to our work, \cite{Wang2024Transformers} theoretically analyzes the training dynamics of one-layer transformers with near-orthogonal stochastic positional encodings, showing that these encodings improve length generalization while inducing distributional symmetries that simplify the characterization of gradient-descent trajectories. Our construction similarly induces a signed-permutation symmetry, which we exploit in Appendix~\ref{subsec:training_dynamics_reduced_k1} to characterize the training dynamics. Importantly, stochastic positional encodings are not essential to our expressive power guarantees (Theorem~\ref{thm:expressive power} and Proposition~\ref{prop:positive_margin}); these results extend directly to any set of mutually orthogonal positional encodings.

\noindent\textbf{Embedding construction.}
We construct the context matrix $\Hb \in \RR^{D \times N}$ and the initial query vector $\qb^{(0)} \in \RR^D$ by concatenating the features, labels, and positional encodings. Specifically, for each training set $\{(\xb_i, y_i)\}_{i=1}^N$ and query input $\xb_q$, the embedding $\hb_i$ (the $i$-th column of $\Hb$) and the initial query vector $\qb^{(0)}$ are given by $\hb_i = \big[ \xb_i^\top,\ y_i,\ \pb_i^\top \big]^\top$ and $\qb^{(0)} = \big[ \xb_q^\top,\ 0,\ \pb_{N+1}^\top \big]^\top$, respectively. Consequently, the initial input embedding matrix is in the following format:
\begin{equation}\label{eq:Z_definition}
    \Zb^{(0)} = [\Hb, \qb^{(0)}]=\begin{bmatrix} \xb_1 &\cdots &\xb_N& \xb_q\\ y_1 &\cdots &y_N&0\\ \pb_1&\cdots & \pb_N &\pb_{N+1} \end{bmatrix} \in \RR^{D\times (N+1)}.
\end{equation}

\noindent\textbf{One-layer self-attention transformer.}
Following the setup of \citet{Wang2024Transformers,Huang2025Transformers,Yang2025Multi}, we consider a one-layer self-attention transformer in which the projection and value matrices are combined into $\Vb \in \RR^{D \times D}$, and the key and query matrices are combined into $\Wb \in \RR^{D \times D}$. The transformer layer, parameterized by $\btheta = (\Vb, \Wb)$, is defined by
\begin{equation}\label{eq:TF_def}
    \mathrm{TF}(\Zb, \btheta) = \Zb + \Vb \Zb\cdot\mathrm{softmax}\left(\Zb^\top \Wb \Zb\right),
\end{equation}
where $\mathrm{softmax}(\Ab)$ applies the softmax operator on each column of the matrix $\Ab$, i.e., $[\mathrm{softmax}(\Ab)]_{ij} \allowbreak = \mathrm{exp}(\Ab_{ij} )/\sum_i \mathrm{exp}(\Ab_{ij} )$.

The transformer definition in \eqref{eq:TF_def} is slightly simplified relative to practical implementations, most notably in its use of merged key and query matrices. These simplifications reflect the limitations of our theoretical framework. However, similar simplifications are common in theoretical analyses of transformers \citep{Li2024One,Wang2024Transformers,Yang2025Multi}. Among these existing works, \citet{Li2024One} showed that transformers with frozen $\Vb = \Ib$ can be trained to perform $1$-NN prediction in context via direct inference. However, their result cannot be directly extended to the case of $\mathrm{K}$-NN with $K>1$, as softmax attention is much less naturally suited to attending to the 2nd through $K$-th nearest neighbors than to attending to the single nearest neighbor. 

To illustrate the limitation of direct inference by transformers, we define a class of well-separated input instances, and show that transformers in \eqref{eq:TF_def} cannot perform $\mathrm{K}$-NN even on such instances. Specifically, for \(\gamma>0\), let
\(\cX_\gamma\) be the class of input instances
\(\{(\xb_i,y_i)_{i=1}^N,\xb_q\}\) such that
\(\xb_1,\ldots,\xb_N,\xb_q\in\mathbb S^{d-1}\), \(y_i\in\{\pm1\}\), and
\[
\min\left\{
\min_{1\le k\le N-1}
\big(\|\xb_q-\xb_{\pi(k+1)}\|-\|\xb_q-\xb_{\pi(k)}\|\big),
\,
2-\|\xb_q-\xb_{\pi(N)}\|
\right\}\ge \gamma.
\]
Then we have the following proposition. 
\begin{proposition}\label{prop: one layer gap}
Assume \(d\ge 2\), \(3\le K\le N-2\), and \(K\) is odd. 
For any \(0<\gamma\le 2/(N+K+1)\), any one-layer transformer
\eqref{eq:TF_def} with fixed finite matrices
\(\Vb,\Wb\in\mathbb R^{D\times D}\) and any realization of the stochastic positional encodings, there exists an input instance in
\(\cX_{\gamma}\) such that the corresponding input embedding matrix $\Zb^{(0)}$ in \eqref{eq:Z_definition} satisfies
\[
\operatorname{sign}\big(\mathrm{TF}(\Zb^{(0)},\btheta)_{d+1,-1}\big)
\ne
\operatorname{sign}\Big(\sum_{k=1}^K y_{\pi(k)}\Big).
\]
\end{proposition}
We defer the detailed proof to Appendix~\ref{appendix: proof of one layer gap}. Proposition \ref{prop: one layer gap} shows that, although one-layer transformers are capable of performing in-context $1$-NN prediction \citep{Li2024One}, this direct-inference mechanism already breaks down for odd $K>1$, even on a well-separated class of inputs. This limitation motivates our study of the CoT setting.

\section{Construction of Transformers}\label{sec:construction}
In this section, we first introduce the CoT process, which enables transformers to generate intermediate results autoregressively before producing the final answer. We then show that there exists a parameter setting $\btheta$ such that the one-layer transformer \eqref{eq:TF_def} can solve the $\mathrm{K}$-NN problem.

\noindent\textbf{Prediction with CoT.} With CoT, we enable the transformer to generate $K-1$ steps before it outputs the final prediction. Specifically, at the $k$-th generation step, given the input sequence $\hat{\Zb}^{(k)}$, we obtain $\mathrm{TF}(\hat{\Zb}^{(k)}, \btheta)_{:, -1}$ as the prediction of the $(k+1)$-th token and append it to the end of the current sequence such that $\hat{\Zb}^{(k+1)} = [\hat{\Zb}^{(k)},\mathrm{TF}(\hat{\Zb}^{(k)}, \btheta)_{:, -1}]$. After $K-1$ reasoning steps, the CoT process induces $K-1$ sequences $\{\hat{\Zb}^{(k)}\}_{k=1}^{K-1}$ in the following form:
\begin{equation*}
    \hat{\Zb}^{(k)} =[\Zb^{(0)}, \mathrm{TF}({\Zb}^{(0)}, \btheta)_{:, -1},\ldots,\mathrm{TF}(\hat{\Zb}^{(k-1)}, \btheta)_{:, -1}]
    \in \mathbb{R}^{D \times (N+k+1)}, k \in [K-1].
\end{equation*}
 Note that $\hat{\Zb}^{(0)} = \Zb^{(0)}=[\Hb,\qb^{(0)}]$.

 We impose the following parameters $\btheta = (\Vb, \Wb)$ in \eqref{eq:TF_def}: 
 \begin{equation*}
\Vb=
\begin{bmatrix}
\mathbf{0}_{d\times d} & \mathbf{0}_{d\times 1} & \mathbf{0}_{d\times (N+1)} \\
\mathbf{0}_{1\times d} & 1 & \mathbf{0}_{1\times (N+1)} \\
\mathbf{0}_{(N+1)\times d} & \mathbf{0}_{(N+1)\times 1} & \Ib_{N+1}.
\end{bmatrix},\quad \Wb=
\begin{bmatrix}
\alpha\Ib_d &  \mathbf{0}_{d\times 1} & \mathbf{0}_{d\times (N+1)} \\
\mathbf{0}_{1\times d} & 0 & \mathbf{0}_{1\times (N+1)} \\
\mathbf{0}_{(N+1)\times d} &  \mathbf{0}_{(N+1)\times 1} & \beta\Ib_{N+1},
\end{bmatrix}
 \end{equation*}
 where $\alpha$ and $\beta$ are scalars. 
The following theorem provides a construction of a one-layer transformer that solves the $\mathrm{K}$-NN problem.

\begin{theorem}\label{thm:expressive power}
    Suppose that, almost surely, $\xb_1,\dots,\xb_N,\xb_q\in \mathbb S^{d-1}$, $\|\xb_q-\xb_{\pi(1)}\|<\cdots<\|\xb_q-\xb_{\pi(N)}\|<2$, and $y_i\in\{\pm1\}\,\text{for every }i\in[N]$. For any $K \leq N$, let $\hat{y}_K = \mathrm{sign}(\sum_{k=1}^K y_{\pi(k)})$. Then, as $\alpha \to \infty$, $\beta \to -\infty$, and $\beta / \alpha \to -2$, it holds that
    \begin{equation*}
        \mathbb P\left(
\mathrm{TF}(\hat{\mathbf{Z}}^{(K-1)}, \boldsymbol{\theta})_{d+1, -1}\cdot\hat{y}_K<0
\right)\to0,
    \end{equation*}
    where the probability is taken with respect to the joint law of
\((\mathcal{D},\xb_q,\omega)\), where \(\mathcal{D}=\{(\xb_i,y_i)\}_{i=1}^N\) and
\(\omega=(\tau,s_1,\dots,s_{N+1})\) denotes the randomness in the stochastic positional encodings.
\end{theorem}
We defer the detailed proof to Appendix~\ref{appendix: proof of expressive power}. Theorem~\ref{thm:expressive power} gives an asymptotic expressivity result: under the sharp-attention limit $\alpha \to \infty$, $\beta \to -\infty$, and $\beta / \alpha \to -2$, a one-layer transformer can solve the $\mathrm{K}$-NN problem via CoT reasoning. If we further restrict attention to the $\gamma$-separated input instances in $\cX_{\gamma}$ defined in Section~\ref{subsec:transformer_arch}, this asymptotic construction can be strengthened to a finite-parameter correctness statement, which we state next.

\begin{proposition}[Uniform correctness under a fixed margin]\label{prop:positive_margin}
Under the construction in Theorem~\ref{thm:expressive power}, the following
holds. For any \(K\le N\) and \(\gamma>0\), there exist constants
\(\alpha_0>0\) and \(\delta_0>0\) depending only on \((\gamma,N,K)\) such
that if
\(\alpha\ge \alpha_0\) and \(\left|\beta/\alpha+2\right|\le \delta_0\), then 
\[
\operatorname{sign}\!\left(
\mathrm{TF}(\widehat{\Zb}^{(K-1)},\btheta)_{d+1,-1}
\right)
\equiv
\operatorname{sign}\Big(\sum_{k=1}^K y_{\pi(k)}\Big)
\]
for all instances in \(\cX_\gamma\) with
\(\sum_{k=1}^K y_{\pi(k)}\neq 0\) and all realizations of the stochastic positional encodings.
\end{proposition}
We defer the detailed proof to Appendix~\ref{appendix: proof of positive_margin}. Proposition~\ref{prop:positive_margin} is directly comparable with Proposition~\ref{prop: one layer gap}, showing that a transformer with exactly the same architecture can perform $\mathrm{K}$-NN with $K$ reasoning steps. These two propositions together highlight the strength of CoT over direct inference in handling complicated tasks.

\noindent\textbf{Key intuition.} The key intuition behind Theorem~\ref{thm:expressive power} and Proposition~\ref{prop:positive_margin} is that our construction implements a sequential nearest-neighbor selection procedure. The following lemma summarizes the outcome of this selection procedure and serves as a key step in the proof of Theorem~\ref{thm:expressive power}.
\begin{lemma}\label{lem:expressive power mechanism}
Suppose that $\xb_1,\dots,\xb_N,\xb_q\in \mathbb S^{d-1}$, $\|\xb_q-\xb_{\pi(1)}\|<\cdots<\|\xb_q-\xb_{\pi(N)}\|<2$, and $y_i\in\{\pm1\}\,\text{for every }i\in[N]$. For any $K \leq N$, as $\alpha \to \infty$, $\beta \to -\infty$, and $\beta / \alpha \to -2$, it holds that
$
     \mathrm{TF}(\hat{\mathbf{Z}}^{(K-1)}, \boldsymbol{\theta})_{d+1, -1} - \sum_{k=1}^K y_{\pi(k)}\to 0
$.
\end{lemma}
To explain Lemma~\ref{lem:expressive power mechanism}, consider the following idealized setting. Suppose that after $k$ CoT steps, all previously generated tokens are exact, so that
$
\hat{\Zb}^{(k)}=[\Hb,\qb^{(0)},\dots,\qb^{(k)}],
$
where
$
\qb^{(k)}=\big[\xb_q^\top,\ \sum_{i=1}^k y_{\pi(i)},\ \pb_{N+1}^\top+\sum_{i=1}^k \pb_{\pi(i)}^\top\big]^\top.
$
Then the attention logits for the last token take the form
\[
\big(\hat{\Zb}^{(k)}\big)^\top \Wb\hat{\Zb}^{(k)}_{:,-1}
=
\begin{cases}
\alpha \xb_i^\top \xb_q,
& \text{for an unselected context token } h_i,\ i\notin\{\pi(1),\dots,\pi(k)\},\\[1mm]
\alpha \xb_{\pi(i)}^\top \xb_q+\beta,
& \text{for a selected context token } h_{\pi(i)},\ 1\le i\le k,\\[1mm]
\alpha+(m+1)\beta,
& \text{for a previously generated token } \qb^{(m)},\ 0\le m\le k.
\end{cases}
\]
The key point is that the positional component leaves unselected context tokens unchanged, penalizes each previously selected neighbor by an additive factor of \(\beta\), and penalizes each previously generated token \(\qb^{(m)}\) by an additive factor of \((m+1)\beta\). 
Dividing by $\alpha$ and letting $-\beta/\alpha\to 2$, the corresponding asymptotic scores become $\xb_i^\top \xb_q$, $\xb_{\pi(i)}^\top \xb_q-2$, and $-(2m+1)$. Combined with
$
\xb_{\pi(1)}^\top \xb_q>\cdots>\xb_{\pi(N)}^\top \xb_q>-1,
$
which is implied by $\|\xb_q-\xb_{\pi(1)}\|<\cdots<\|\xb_q-\xb_{\pi(N)}\|<2$, the largest score is attained by the next unretrieved neighbor $\pi(k+1)$. Therefore,
$
\mathrm{softmax}((\hat{\Zb}^{(k)})^\top \Wb\hat{\Zb}^{(k)}_{:,-1})
$
converges to the one-hot vector corresponding to \(h_{\pi(k+1)}\). With the residual connection, the transformer prediction $\mathrm{TF}(\hat{\mathbf{Z}}^{(k)}, \boldsymbol{\theta})_{:, -1}$ converges to $
\qb^{(k+1)}=\big[\xb_q^\top,\ \sum_{i=1}^{k+1} y_{\pi(i)},\ \pb_{N+1}^\top+\sum_{i=1}^{k+1} \pb_{\pi(i)}^\top\big]^\top.
$


The discussion above explains the one-step mechanism underlying the construction under the idealization that all previously generated tokens are exact. In the actual CoT rollout, however, the previously generated tokens only approach their ideal counterparts, and the resulting perturbations to subsequent attention scores may affect which token receives the largest attention weight. Appendix~\ref{appendix: proof of expressive power} addresses this challenge through a refined induction argument, thereby establishing Lemma~\ref{lem:expressive power mechanism}. Since the \(\mathrm{K}\)-NN target is given by \(\sign(\sum_{i=1}^K y_{\pi(i)})\), Lemma~\ref{lem:expressive power mechanism} implies that
$
\mathbb{P}(
\mathrm{TF}(\hat{\mathbf{Z}}^{(K-1)}, \boldsymbol{\theta})_{d+1,-1}\cdot \hat{y}_K < 0
)\to 0,
$
which is exactly the statement of Theorem~\ref{thm:expressive power}.

\section{Training Dynamics and Generalization of Transformers}
Section~\ref{sec:construction} establishes the existence of a one-layer transformer capable of solving the $\mathrm{K}$-NN problem using CoT. However, those results do not address whether such a transformer can be learned in practice. In this section, we analyze the training dynamics of the transformer and, interestingly, our result shows that gradient descent can recover the construction in Theorem~\ref{thm:expressive power} under one-step state supervision.


\noindent\textbf{Training.}   
Following prior analyses of training dynamics
\citep{Huang2024How,Li2024One,Gatmiry2024Can,Yang2025Multi},
we keep $\Vb$ fixed throughout training. This simplification is commonly adopted in recent theoretical work to isolate the learning dynamics of softmax attention. Motivated by the construction in Section~\ref{sec:construction}, we set
$
\Vb=\diag(\mathbf{0},1,\Ib_{N+1}).
$
Accordingly, we rewrite \(\mathrm{TF}(\Zb,\btheta)\) in \eqref{eq:TF_def} as \(\mathrm{TF}(\Zb,\Wb)\). We then train the transformer to predict the one-nearest-neighbor ground-truth vector $[\xb_q^\top, y_{\pi(1)},\pb_{N+1}^\top+\pb_{\pi(1)}^\top]^\top$
given the initial embedding matrix \(\Zb^{(0)}\).

The training loss for a single sample is:
\begin{equation*}
    \ell(\Wb) =  \frac{1}{2}\left\| \mathrm{TF}({\Zb}^{(0)}, \Wb)_{:, -1} - [\xb_q^\top, y_{\pi(1)},\pb_{N+1}^\top+\pb_{\pi(1)}^\top]^\top \right\|^2.
\end{equation*}
The population loss is then defined as 
\begin{equation}\label{eq:training target}
\mathcal{L}(\Wb) = \mathbb{E}_{\mathcal{D}, \xb_q} \mathbb{E}_{\omega}\left[ \ell( \Wb) \right].
\end{equation}

The learning objective in \eqref{eq:training target} is optimized via GD with learning rate $\eta > 0$. At $t=0$, we initialize $\Wb^{(0)}=\mathbf{0}$ as the zero matrix. The parameter is updated as follows:
\begin{equation*}\label{GD}
\begin{aligned}
    \Wb^{(t+1)}&=\Wb^{(t)}-\eta \nabla_{\Wb} \mathcal{L}(\Wb^{(t)}).
\end{aligned}
\end{equation*}
\noindent For this training setup, we make the following assumption on the training set.
\begin{assumption}[Training Distribution]
\label{ass:distribution}
For a given dataset $\{(\xb_i, y_i)\}_{i=1}^N \cup \{\xb_q\}$, let $\xb_{1:q} \coloneqq  (\xb_1,\ldots,\xb_N,\xb_q)$ and $y_{1:N} \coloneqq  (y_1,\ldots,y_N)$. We focus on the following underlying distribution: (i) The input vectors $\{\xb_i\}_{i=1}^N$ and the query $\xb_q$ are sampled independently and uniformly from the unit sphere $\mathbb{S}^{d-1} \subset \RR^d$. (ii) The labels $\{y_i\}_{i=1}^N$ satisfy $y_i\in\{\pm 1\}$ and $\EE[y_i y_j \mid \xb_{1:q}] = 0$ for all $i\neq j,i,j\in[N]$. (iii) We have $\mathbb{P}(y_{1:N}\mid \xb_{1:q})=\mathbb{P}(-y_{1:N}\mid \xb_{1:q})$ and $\mathbb{P}(y_{1:N}\mid \xb_{1:q})=\mathbb{P}(y_{1:N}\mid \Rb\xb_{1:q})$ for any rotation matrix $\Rb\in \RR^{d\times d}$.
\end{assumption}
Assumption \ref{ass:distribution} is adapted from \citet{Li2024One}. As discussed in \citet{Li2024One}, it covers the setting in which the $y_i$'s are independent of the $\xb_i$’s and uniformly distributed over $\{\pm 1\}$. Such cases can be challenging because they may induce a complicated $\mathrm{K}$-NN prediction rule and a highly oscillatory optimization landscape.

In the following theorem, we prove that for suitable parameters $(d,N,\eta)$, the loss function \eqref{eq:training target} will converge to zero under gradient descent.

\begin{theorem}
\label{thm:training convergence}
Suppose that
$d\ge 2$, $N=\Omega(\sqrt d)$ and $\eta=O(N^{-2})$. Under Assumption \ref{ass:distribution}, we train the transformer model \eqref{eq:TF_def} by gradient descent to minimize the population loss \eqref{eq:training target}. Then we have
\begin{enumerate}[leftmargin = *,nosep]
    \item  For all training steps $t\geq 0$, 
\[
\Wb^{(t)}=\diag\bigl(\alpha^{(t)}\Ib_d,\;0,\;\beta^{(t)}\Ib_{N+1}\bigr),
\]
with the scalar coefficients satisfying $\beta^{(t)}\to -\infty$, $\alpha^{(t)}\to +\infty$, and ${\beta^{(t)}}/{\alpha^{(t)}}\to -2$.
    \item The population loss $\mathcal{L}( \Wb^{(t)})$ converges to zero as $t\to \infty$.
\end{enumerate}

\end{theorem}
We defer the detailed proof to Appendix~\ref{appendix: proof of gradient descent}. Theorem~\ref{thm:training convergence} shows that gradient descent recovers the construction in Theorem~\ref{thm:expressive power} using supervision only on the first-step target state vector. Prior theoretical studies of one-layer transformers with CoT~\citep{Huang2025Transformers,Huang2025Length,Yang2025Multi} typically supervise multiple intermediate states via teacher forcing, effectively providing repeated examples of a shared transition rule. Such supervision can facilitate optimization analysis, but need not be necessary for identifying the underlying rule. Our result shows that, for $\mathrm{K}$-NN, supervising only the first-step target state vector suffices to learn the mechanism detailed in Section~\ref{sec:construction}, which can then be applied autoregressively over multiple steps.


The first-step target supervision essentially also directly teaches the model to perform 1-NN retrieval, making our setting more comparable to that of \citet{Li2024One}. However, the two settings differ in several important ways. \citet{Li2024One} only use a binary indicator to distinguish the query token from the context tokens and supervise the scalar target $y_{\pi(1)}$. Their GD trajectory has the form $\diag(\xi_1^{(t)} I_d, 0, -\xi_2^{(t)})$, with a single scalar coefficient for the indicator dimension. In our setting, each token instead carries a stochastic positional vector, and we supervise the full target $[\xb_q^\top, y_{\pi(1)}, \pb_{N+1}^\top + \pb_{\pi(1)}^\top]^\top$. The positional component is essential for CoT because it records the retrieved neighbor, allowing that neighbor to be penalized in subsequent reasoning steps. A further difference is that the analysis of \citet{Li2024One} requires $-\xi_2^{(0)}$ to be sufficiently negative, whereas we use $\alpha^{(0)}=\beta^{(0)}=0$, so query-token suppression must emerge from GD. Moreover, their analysis establishes the qualitative separation $\xi_1^{(t)}\to+\infty$ and $\xi_2^{(t)}-\xi_1^{(t)}\to+\infty$. Our analysis provides a finer characterization of the learned coefficients: starting from zero initialization, we prove that $\alpha^{(t)}\to+\infty$ and $\beta^{(t)}\to-\infty$ and further identify their relative growth through $\beta^{(t)}/\alpha^{(t)}\to-2$. Establishing these coupled asymptotics from zero initialization requires new analytical arguments. In particular, we develop contradiction and compactness arguments to establish coefficient divergence (Appendix~\ref{sec:parameter divergence}), and drift and crossing arguments to characterize the limiting ratio (Appendix~\ref{sec:convergence of the population loss}). We provide a proof sketch in Section~\ref{sec:proof sketch}.


\noindent\textbf{Testing.} Compared with the conditions of Theorem~\ref{thm:expressive power}, Assumption~\ref{ass:distribution} is more restrictive. This is natural, as Assumption~\ref{ass:distribution} enables training guarantees that are stronger than the pure expressive power results in Theorem~\ref{thm:expressive power}. Once training has recovered the construction, however, the learned model can still be applied to more general in-context $\mathrm{K}$-NN prediction tasks. Combining the training and construction results therefore yields the following corollary.

\begin{corollary}\label{corollary: test result}
Suppose that, almost surely, $\xb_1,\dots,\xb_N,\xb_q\in \mathbb S^{d-1}$, $\|\xb_q-\xb_{\pi(1)}\|<\cdots<\|\xb_q-\xb_{\pi(N)}\|<2$, and $y_i\in\{\pm1\}\,\text{for every }i\in[N]$. Let the single-layer model \eqref{eq:TF_def} be trained for $t$ steps under the conditions of Theorem~\ref{thm:training convergence}. For any $K \leq N$, let $\hat{y}_K = \mathrm{sign}(\sum_{k=1}^K y_{\pi(k)})$. Then,
    \begin{equation*}
        \mathbb P\left(
\mathrm{TF}(\hat{\mathbf{Z}}^{(K-1)}, \boldsymbol{\theta}^{(t)})_{d+1, -1}\cdot\hat{y}_K <0
\right)\to0\,\,\, \text{as} \,\,\, t\to \infty,
    \end{equation*}
    where the probability is taken with respect to the joint law of
\((\mathcal{D},\xb_q,\omega)\), where \(\mathcal{D}=\{(\xb_i,y_i)\}_{i=1}^N\) and
\(\omega=(\tau,s_1,\dots,s_{N+1})\) denotes the randomness in the stochastic positional encodings. Moreover, $
     \mathrm{TF}(\hat{\mathbf{Z}}^{(K-1)}, \boldsymbol{\theta}^{(t)})_{d+1, -1} - \sum_{k=1}^K y_{\pi(k)}\to 0
$ almost surely.
\end{corollary}
Corollary~\ref{corollary: test result} summarizes the consequence of first learning the sharp-attention construction under the training assumptions and then applying that learned construction to the broader test class, and is proved by directly combining Theorems~\ref{thm:training convergence} and~\ref{thm:expressive power}. 

\section{Proof Sketch}\label{sec:proof sketch}
In this section, we sketch the proof of Theorem \ref{thm:training convergence} and highlight the techniques we used. The full proof is deferred to Appendix \ref{appendix: proof of gradient descent}.

We first characterize the structure of $\Wb^{(t)}$ preserved throughout training. The key observation is that the population loss is invariant under three symmetries: feature rotations, label sign flips, and transformations of the positional block of the form \(S = DP\), where \(P \in \RR^{(N+1)\times(N+1)}\) is a permutation matrix and \(D \in \RR^{(N+1)\times(N+1)}\) is a diagonal matrix with entries in \(\{\pm 1\}\). Specifically, for \(R \in O(d)\), \(s \in \{\pm 1\}\), and such an \(S\), define \(G_R\coloneqq \diag(R,1,\Ib_{N+1})\), \(G_s\coloneqq \diag(\Ib_d,s,\Ib_{N+1})\), and \(G_S\coloneqq \diag(\Ib_d,1,S)\). These invariances follow, respectively, from the rotational invariance of \((\xb_1,\ldots,\xb_N,\xb_q)\), the sign symmetry of \(y_1,\ldots,y_N\), and the construction of the stochastic positional encodings. Hence, for each \(G\in\{G_R,G_s,G_S\}\),
\begin{equation}\label{sketch: invariance}
\mathcal L(\Wb)=\mathcal L(G\Wb G^{-1})
\quad\text{and}\quad
\nabla \mathcal L(G\Wb G^{-1})=G\bigl(\nabla \mathcal L(\Wb)\bigr)G^{-1}.
\end{equation}
Since \(\Wb^{(0)}=\mathbf 0\), \eqref{sketch: invariance} implies we always have $G\Wb^{(t)} G^{-1}=\Wb^{(t)}$. A direct characterization of this property shows that $\Wb^{(t)}$ takes the form \(\diag(\alpha \Ib_d,\gamma,\beta \Ib_{N+1})\). Moreover, we can prove by induction that the scalar \(\gamma\) remains equal to \(0\) throughout training. We summarize our result in the following lemma.
\begin{lemma}
\label{lem:sketch-structure}
Under the same conditions as in Theorem \ref{thm:training convergence}, for every \(t\ge 0\), there exist scalars \(\alpha^{(t)},\beta^{(t)}\in\RR\) such that
\[
\Wb^{(t)}=\diag\bigl(\alpha^{(t)}\Ib_d,\;0,\;\beta^{(t)}\Ib_{N+1}\bigr).
\]
\end{lemma}

Lemma~\ref{lem:sketch-structure} reduces the population objective to two scalar parameters. For clarity, we first write the explicit form of the attention weights. Let
$
\mathrm{softmax}\left(\Zb^{(0)\top} \Wb^{(t)} \qb^{(0)}\right)
=
\bigl(a_1^{(t)},\dots,a_N^{(t)},a_0^{(t)}\bigr)^\top.
$
Then
\[
a_i^{(t)}
=
\frac{\exp\bigl(\alpha^{(t)}\,\xb_i^\top \xb_q\bigr)}
{\exp\bigl(\alpha^{(t)}+\beta^{(t)}\bigr)+\sum_{j=1}^N \exp\bigl(\alpha^{(t)}\,\xb_j^\top \xb_q\bigr)},
a_0^{(t)}
=
\frac{\exp\bigl(\alpha^{(t)}+\beta^{(t)}\bigr)}
{\exp\bigl(\alpha^{(t)}+\beta^{(t)}\bigr)+\sum_{j=1}^N \exp\bigl(\alpha^{(t)}\,\xb_j^\top \xb_q\bigr)}.
\]
Moreover, taking the conditional expectation with respect to \(\sigma(\xb_1,\ldots,\xb_N,\xb_q)\) eliminates the dependence of the terms \(a_i^{(t)}a_j^{(t)}y_i y_j\) in the population loss on the labels. Under Assumption~\ref{ass:distribution}, which gives \(\EE[y_i y_j\mid \xb_1,\ldots,\xb_N,\xb_q]=0\) for \(i\neq j\) and \(\EE[y_i^2\mid \xb_1,\ldots,\xb_N,\xb_q]=1\), and using the identity \(\pb_i^\top \pb_j=\delta_{ij}\), we obtain
\begin{equation}\label{sketch: reduced loss}
\mathcal L(\Wb^{(t)})
=
\EE\big[
(1-a_{\pi(1)}^{(t)})^2+\sum_{i\neq\pi(1)}(a_i^{(t)})^2+\frac12(a_0^{(t)})^2
\big].
\end{equation}
Equation~\eqref{sketch: reduced loss} makes the objective transparent: \((1-a_{\pi(1)}^{(t)})^2\) rewards attention to the nearest neighbor, \(\sum_{i\neq\pi(1)}(a_i^{(t)})^2\) penalizes mass assigned to incorrect context tokens, and \(\frac12(a_0^{(t)})^2\) penalizes attention to the query token. 
The convergence analysis therefore reduces to tracking two scalar parameters: \(\beta^{(t)}\) must suppress attention to the query token, while \(\alpha^{(t)}\) must sharpen the ranking of context tokens according to \(\xb_i^\top \xb_q\). The following lemma presents our first result on the dynamics of \(\alpha^{(t)}\) and \(\beta^{(t)}\).
\begin{lemma}\label{lem:sketch-parameter-divergence}
    Under the same conditions as in Theorem \ref{thm:training convergence}, we have
    \[
    \alpha^{(t)}\to +\infty,
    \qquad
    \beta^{(t)}\to -\infty.
    \]
\end{lemma}
Establishing these results is nontrivial, particularly under the zero initialization of $\Wb$. In contrast, \citet{Li2024One} initialize the entry of $\Wb$ associated with the query token to a sufficiently negative value, thereby suppressing the corresponding attention weight from the outset. Under our zero initialization $\Wb^{(0)}=0$, this suppression must instead arise from the GD dynamics. We first establish, via an inductive argument, the forward invariance of the region defined by $\alpha^{(t)}\geq 0$ and $\beta^{(t)}\leq 0$ throughout training, which further allows us to prove that $\beta^{(t)}$ is strictly decreasing (Lemma~\ref{prop:invariant}).  Building on these properties, we develop a contradiction argument to establish $\beta^{(t)}\to-\infty$ (Lemma~\ref{prop:beta-divergence}). Finally, once $\beta^{(t)}\to-\infty$ has been established, we show that the gradient with respect to $\alpha$ is eventually uniformly negative on every compact interval, which implies $\alpha^{(t)}\to+\infty$ (Lemma~\ref{prop:alpha-divergence}).

However, Lemma \ref{lem:sketch-parameter-divergence} alone does not determine \(a_0^{(t)}\), which depends on the competition between the logit \(\alpha^{(t)}+\beta^{(t)}\) and the top context logit \(\alpha^{(t)}\,\xb_{\pi(1)}^\top \xb_q\). The key issue is therefore the relative scale of \(\alpha^{(t)}\) and \(\beta^{(t)}\).
To identify the relevant threshold, it is natural to compare the weight of the query token with that of the top context token. Since
$
{a_0^{(t)}}/{a_{\pi(1)}^{(t)}}
=
\exp(\alpha^{(t)}(1-\xb_{\pi(1)}^\top \xb_q+{\beta^{(t)}}/{\alpha^{(t)}})),
$
the competition is governed by the sign of \((1-\xb_{\pi(1)}^\top \xb_q+\beta^{(t)}/\alpha^{(t)})\). Because \(1-\xb_{\pi(1)}^\top \xb_q \in (0,2)\, a.s.\), the value \(2\) emerges as the critical threshold for ensuring that \(a_{\pi(1)}^{(t)}\) dominates. This observation motivates splitting the analysis into two separate regimes: 
\begin{equation*}
    -\beta^{(t)}/\alpha^{(t)} \in (0,2)\quad \text{and} \quad-\beta^{(t)}/\alpha^{(t)} > 2.
\end{equation*}
 The next two lemmas formalize these behaviors and show that they imply
\(\liminf_{t\to\infty}{-\beta^{(t)}}/{\alpha^{(t)}}\ge 2\) and \(\limsup_{t\to\infty}{-\beta^{(t)}}/{\alpha^{(t)}}\le 2\).
\begin{lemma}
\label{lem:sketch-interior-drift}
Denote $\lambda_t\coloneqq {-\beta^{(t)}}/{\alpha^{(t)}}$. Let $I=[a,b]\subset(0,2)$.
Then there exist constants
$
0<\underline d_I\le \overline d_I<\infty$, $\delta_I>0,
$
and an integer \(T_I\) such that for every \(t\ge T_I\), if $\lambda_t\in I$, then
\[
\frac{\underline d_I}{(\alpha^{(t)})^2}
\le
\lambda_{t+1}-\lambda_t
\le
\frac{\overline d_I}{(\alpha^{(t)})^2},\quad
\mathcal{L}(\Wb^{(t)})-\mathcal{L}(\Wb^{(t+1)})
\ge
\frac{\delta_I}{(\alpha^{(t)})^2}.
\]
\end{lemma}
Lemma~\ref{lem:sketch-interior-drift} establishes two key facts. First, on any interval \(I\subset(0,2)\), \(\lambda_t\) exhibits a uniform positive drift: the increment \(\lambda_{t+1}-\lambda_t\) is of order \((\alpha^{(t)})^{-2}\). Second, each such step yields a decrease in the loss of at least order \((\alpha^{(t)})^{-2}\). Together, these facts imply that every sufficiently late upcrossing of \([a,b]\subset(0,2)\) by the sequence \((\lambda_t)_{t\ge 0}\) induces a fixed positive decrease in the loss; therefore, only finitely many such upcrossings can occur. Combining this conclusion with two results from the appendix, we obtain the desired claim (Theorem \ref{thm:liminf-lambda-ge-2-final}). Specifically, Lemma \ref{lem:zero-not-subseq-final} shows that \(0\) is not a subsequential limit of \(\{\lambda_t\}\), while Lemma~\ref{lem:no-tail-trap-final} shows that the trajectory cannot remain indefinitely in any compact interval \(I\subset(0,2)\). Therefore,
$
\liminf_{t\to\infty}\lambda_t\ge 2.
$

\begin{lemma}
\label{lem:sketch-exterior-control}
Denote $\lambda_t\coloneqq {-\beta^{(t)}}/{\alpha^{(t)}}$. For every \(\varepsilon>0\), there exist constants \(c_\varepsilon>0\), \(C_\varepsilon>0\), and an integer \(T_\varepsilon\) such that for all \(t\ge T_\varepsilon\), the following hold:
\begin{enumerate}[leftmargin=*, nosep]
    \item If
    $
    \lambda_t\ge 2+\frac{\varepsilon}{2},
    $
    then
    \[
    \lambda_{t+1}-\lambda_t
    \le
    -\frac{c_\varepsilon}{(\alpha^{(t)})^{d+1}},
    \quad
    \alpha^{(t+1)}-\alpha^{(t)}
    \ge
    \frac{c_\varepsilon}{(\alpha^{(t)})^{d}},\quad
    0\le
    \beta^{(t)}-\beta^{(t+1)}
    \le
    C_\varepsilon e^{-(\varepsilon/2)\alpha^{(t)}}.
    \]
    \item If
    $
    \lambda_t\le 2+\varepsilon,
    $
    then
     $
    |\lambda_{t+1}-\lambda_t|<{\varepsilon}/{2}.
     $
\end{enumerate}
\end{lemma}
Lemma~\ref{lem:sketch-exterior-control} reveals two important facts. First, when \(\lambda_t \ge 2+\varepsilon/2\), \(\lambda_t\) exhibits a uniform negative drift, while \(\alpha^{(t)}\) continues to increase and the \(\beta\)-updates become exponentially small. These estimates imply that the trajectory cannot remain indefinitely in the region \(\{\lambda_t \ge 2+\varepsilon/2\}\). Otherwise, \(\alpha^{(t)}\) would grow at least polynomially fast, the increments \(\beta^{(t)}-\beta^{(t+1)}\) would be summable, and \(\beta^{(t)}\) would converge to a finite limit, contradicting the fact that \(\beta^{(t)}\to -\infty\). Second, once \(t\) is sufficiently large and \(\lambda_t \le 2+\varepsilon\), the final estimate shows that the dynamics cannot cross back above \(2+\varepsilon\) in a single step. Combined with the negative drift in the regime \(\lambda_t \ge 2+\varepsilon/2\), this implies that, after the last time the trajectory exits the exterior region, it must remain below \(2+\varepsilon\) forever. Therefore, for every \(\varepsilon > 0\), we have
$
\lambda_t < 2+\varepsilon
$
for all sufficiently large \(t\), which yields
$
\limsup_{t\to\infty}\lambda_t \le 2.
$

\textit{Proof of Theorem \ref{thm:training convergence}}. Combining Lemma~\ref{lem:sketch-interior-drift} and Lemma~\ref{lem:sketch-exterior-control}, we obtain
$
-{\beta^{(t)}}/{\alpha^{(t)}}\to 2.
$
Together with Lemma \ref{lem:sketch-parameter-divergence},
this implies that the top context token eventually dominates the query token and all other context tokens. Hence \(a_{\pi(1)}^{(t)}\to 1\), \(a_0^{(t)}\to 0\), and \(a_i^{(t)}\to 0\) for \(i\neq \pi(1)\). Recalling the representation \eqref{sketch: reduced loss},
we conclude that \(\mathcal L(\Wb^{(t)})\to 0\).
\section{Numerical Experiments}\label{sec:numerical experiment}
In this section, we conduct numerical experiments to empirically validate our theoretical results. Specifically, we first show that the convergence of the training loss and the learned parameter structure of the transformer match the predictions of Theorem~\ref{thm:training convergence}. We then report the test accuracy of the trained transformer to verify the claims of Theorem~\ref{thm:expressive power} and Proposition~\ref{prop:positive_margin}. In addition, we conduct an experiment in which $\Wb$ and $\Vb$ are jointly trained from random initialization. The results show that fixing $\Vb$ is not essential for recovering the theoretical construction in Theorem~\ref{thm:expressive power}, but is used only for analytical tractability. Due to space constraints, we defer these results to Appendix~\ref{app:joint_training}.

\noindent\textbf{Experimental setup.}
We conduct experiments to assess three theory-facing questions: whether the training loss decreases, whether the learned parameter matrix develops the predicted block structure, and whether the trained model can execute $\mathrm{K}$-NN prediction at test time via CoT. We train the transformer on synthetic data generated according to Assumption~\ref{ass:distribution}, with $\xb_1,\ldots, \xb_N, \xb_q$ independently and uniformly sampled from the unit sphere, and $y_1,\ldots, y_N$ sampled as independent Rademacher variables. For each run, we sample fresh training instances independently, use the stochastic positional encoding procedure described in Section~\ref{subsec:transformer_arch}, and evaluate the trained model on a separate test set. We independently sample this test set of size 1000 consisting of well-separated instances with distance margin $\gamma=0.1$ via rejection sampling, and we report test accuracy on the $\mathrm{K}$-NN problem for $K=1,3,5$. Test accuracy is computed as the fraction of test instances satisfying
$\mathrm{TF}(\hat{\mathbf{Z}}^{(K-1)}, \boldsymbol{\theta})_{d+1, -1}\cdot\hat{y}_K> 0$.


To probe behavior beyond the theorem setup, we use random Gaussian initialization with variance 0.5 and SGD with batch size 1000 and learning rate 1, rather than zero initialization and exact gradient descent on the population loss. We train the model for 10000 online SGD iterations. Accordingly, these experiments should be interpreted as qualitative checks rather than as a direct numerical verification of Theorem~\ref{thm:training convergence}.


\begin{figure*}[t]
    \centering
    \begin{subfigure}[t]{0.31\textwidth}
        \centering
        \includegraphics[width=\linewidth]{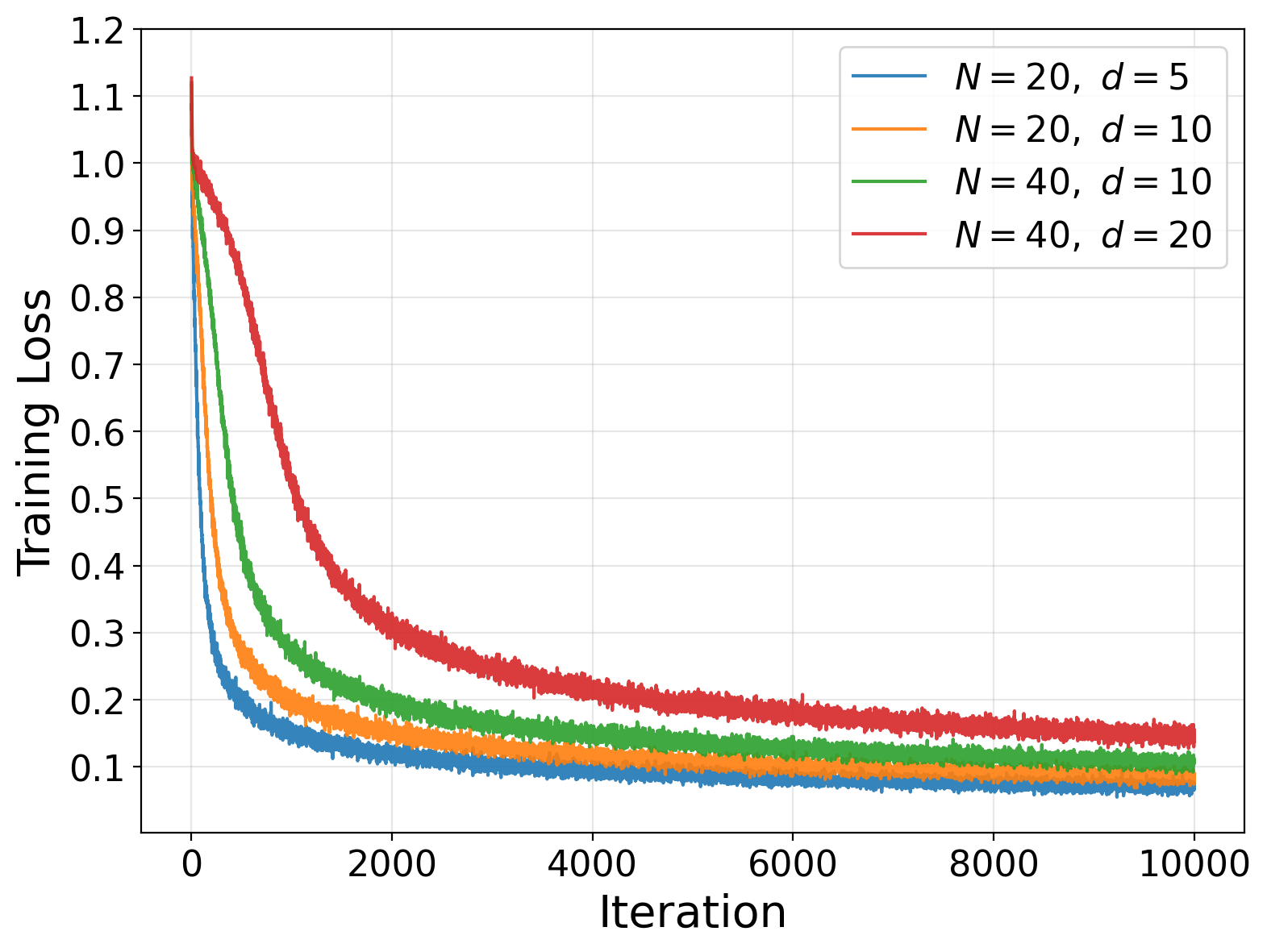}
        \caption{Training loss}
        \label{fig:training loss}
    \end{subfigure}
    \hfill
    \begin{subfigure}[t]{0.31\textwidth}
        \centering
        \includegraphics[width=\linewidth]{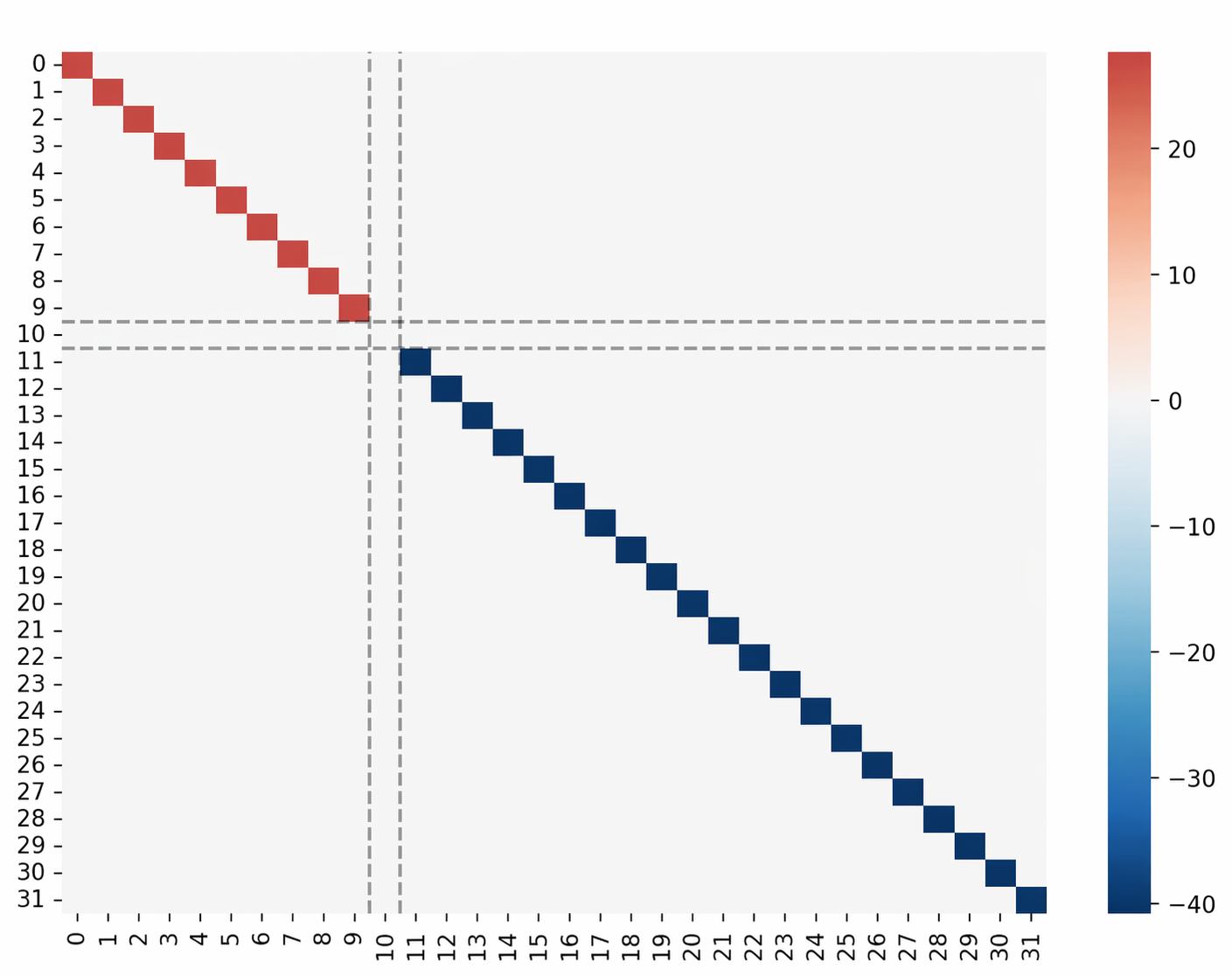}
        \caption{Heatmap of $\Wb$}
        \label{fig:heatmap}
    \end{subfigure}
    \hfill
    \begin{subfigure}[t]{0.36\textwidth}
        \centering
        \includegraphics[width=\linewidth]{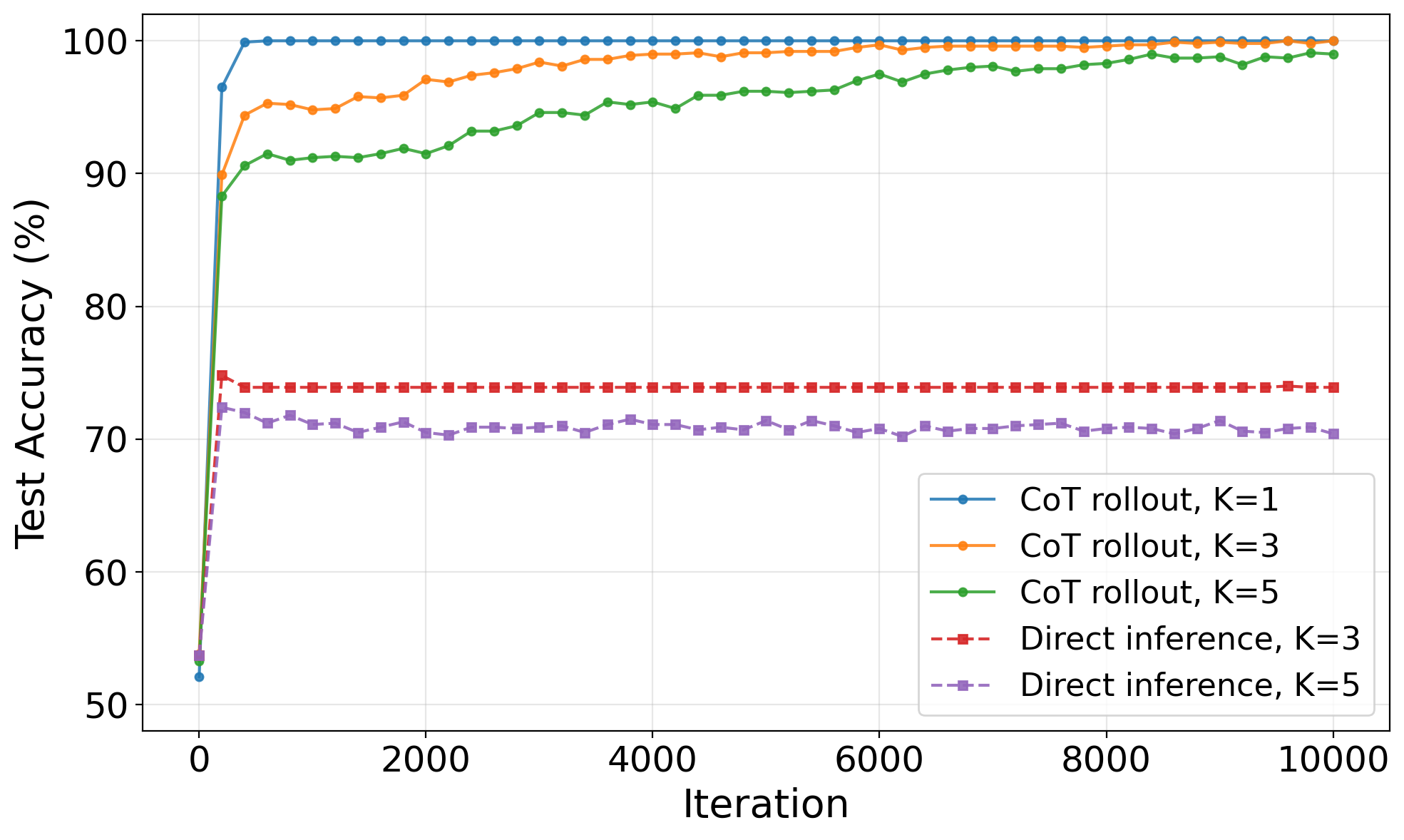}
        \caption{Test accuracy}
        \label{fig:test accuracy}
    \end{subfigure}
    \caption{Training loss, learned parameter heatmap, and test accuracy. The heatmap and test accuracy are shown for the trained transformer with $d=10$ and $N=20$.}
    \label{fig:results}
    \vspace{-15pt}
\end{figure*}
\noindent\textbf{Loss convergence.}
We train the model under various configurations with $d\in\{5,10,20\}$ and $N\in\{20,40\}$. Figure~\ref{fig:training loss} shows that the training loss decreases consistently across all tested configurations, with all curves displaying stable downward trends over training.

\noindent\textbf{Learned parameter structure.}
We plot a heatmap of the learned matrix $\Wb$ after 10000 iterations for $d=10$ and $N=20$. As shown in Figure~\ref{fig:heatmap}, most entries of $\Wb$ are close to zero except for the diagonal entries in the top-left block and the bottom-right block, matching the qualitative block structure predicted by Theorem~\ref{thm:training convergence}. The top-left diagonal entries are positive, while the bottom-right entry is negative, and the magnitude ratio is close to the theoretical value. 



\noindent\textbf{Test accuracy.}
We set $d=10$ and $N=20$, and plot test accuracy against the number of training iterations. As shown in Figure~\ref{fig:test accuracy}, the CoT accuracy for $K=1,3,5$ approaches 1 on the test distribution by the end of training. For comparison, we also train direct-inference baselines for $K=3$ and $K=5$. For each $K$, we train a separate transformer with the same architecture and optimization setup, but supervise it to predict the corresponding final $K$-NN target state $q^{(K)}=[x_q^\top,\sum_{i=1}^{K}y_{\pi(i)},p_{N+1}^\top+\sum_{i=1}^{K}p_{\pi(i)}^\top]^\top$ directly from the initial input in a single forward pass. Despite being trained specifically for each value of $K$, these direct-inference baselines achieve substantially lower test accuracy than their CoT counterparts for both $K=3$ and $K=5$ (Figure~\ref{fig:test accuracy}). These results suggest that CoT enables a sequential retrieval mechanism for $K$-NN prediction that generalizes across reasoning depths, while direct inference appears substantially less effective for $K\ge3$. We further examine direct inference under a larger context size and with multi-head and multi-layer architectures in
Appendix~\ref{app:direct_inference}

\section{Conclusions and Limitations}
This paper gives a concrete theoretical example in which CoT changes what a one-layer transformer can do in context. Without CoT, we show that direct inference fails for odd $\mathrm{K}$-NN on a well-separated class of binary classification tasks. In contrast, a one-layer transformer admits an explicit construction for solving the $\mathrm{K}$-NN problem via CoT. We further show that, in a stylized population-loss training setting with one-step state supervision, gradient descent converges to the same sharp-attention parameter pattern underlying this construction.

The contribution of the paper is primarily theoretical rather than empirical. Our analysis is based on a simplified transformer architecture, frozen $\Vb$ in the training theorem, synthetic data, and population-loss arguments. These choices make the problem tractable and allow us to isolate the role of intermediate reasoning steps, but they also limit how directly the results transfer to practical large-scale transformers. Extending this analysis to richer transformer blocks, less symmetric data models, and finite-sample training remains an important direction for future work.


\appendix

\section{Proof of Proposition \ref{prop: one layer gap}}\label{appendix: proof of one layer gap}
\begin{proof}
Fix arbitrary finite matrices \(\Vb,\Wb\in\mathbb R^{D\times D}\), a realization of \(\pb_1,\ldots,\pb_{N+1}\), and
\(0<\gamma\le 2/(N+K+1)\). We prove that the one-layer transformer
\eqref{eq:TF_def} cannot be correct on all instances in \(\cX_{\gamma}\).
Fix \(\xb_q\in\mathbb S^{d-1}\), and choose a unit vector
\(\ub\perp\xb_q\). For \(t\in(-1,1)\), define
\[
\xb(t)=t\xb_q+\sqrt{1-t^2}\,\ub .
\]
Then \(\xb(t)\in\mathbb S^{d-1}\) and
\[
\|\xb(t)-\xb_q\|^2=2-2t.
\]
Thus larger \(t\) corresponds to smaller distance from \(\xb_q\).

For each context position \(i\), write
\[
\hb_i(t,y)=
\begin{bmatrix}
\xb(t)\\
y\\
\pb_i
\end{bmatrix},
\qquad y\in\{\pm1\}.
\]
For an input with \(\xb_i=\xb(t_i)\), the single-pass last-column label
prediction is
\[
\mathrm{TF}(\Zb^{(0)},\btheta)_{d+1,-1}
=
\frac{
e^{(\qb^{(0)})^\top\Wb\qb^{(0)}}[\Vb\qb^{(0)}]_{d+1}
+
\sum_{i=1}^N
e^{\hb_i(t_i,y_i)^\top\Wb\qb^{(0)}}[\Vb\hb_i(t_i,y_i)]_{d+1}
}{
e^{(\qb^{(0)})^\top\Wb\qb^{(0)}}
+
\sum_{i=1}^N e^{\hb_i(t_i,y_i)^\top\Wb\qb^{(0)}}
}.
\]
The denominator is strictly positive. Hence the sign of
\(\mathrm{TF}(\Zb^{(0)},\btheta)_{d+1,-1}\) is the sign of the numerator.

For \(y\in\{\pm1\}\), define
\[
A_i(t)
=
\frac{
e^{\hb_i(t,1)^\top\Wb\qb^{(0)}}[\Vb\hb_i(t,1)]_{d+1}
+
e^{\hb_i(t,-1)^\top\Wb\qb^{(0)}}[\Vb\hb_i(t,-1)]_{d+1}
}{2},
\]
and
\[
B_i(t)
=
\frac{
e^{\hb_i(t,1)^\top\Wb\qb^{(0)}}[\Vb\hb_i(t,1)]_{d+1}
-
e^{\hb_i(t,-1)^\top\Wb\qb^{(0)}}[\Vb\hb_i(t,-1)]_{d+1}
}{2}.
\]
Then
\[
e^{\hb_i(t,y)^\top\Wb\qb^{(0)}}[\Vb\hb_i(t,y)]_{d+1}
=
A_i(t)+B_i(t)y.
\]
Therefore, on this restricted family of inputs, the prediction has sign
\[
\operatorname{sign}\left(
A(t_1,\ldots,t_N)+\sum_{i=1}^N B_i(t_i)y_i
\right),
\]
where
\[
A(t_1,\ldots,t_N)
=
e^{(\qb^{(0)})^\top\Wb\qb^{(0)}}[\Vb\qb^{(0)}]_{d+1}
+
\sum_{i=1}^N A_i(t_i)
\]
is independent of the labels.

We now switch to distance coordinates. If
\[
d_i=\|\xb(t_i)-\xb_q\|,
\]
then
\[
t_i=1-\frac{d_i^2}{2}.
\]
For notational simplicity, write \(B_i(d_i)\) for
\(B_i(1-d_i^2/2)\), and similarly write \(A(d_1,\ldots,d_N)\).
Thus the prediction has sign
\[
\operatorname{sign}\left(
A(d_1,\ldots,d_N)+\sum_{i=1}^N B_i(d_i)y_i
\right).
\]

Assume for contradiction that the fixed matrices \(\Vb,\Wb\) are correct for
every instance in \(\cX_{\gamma}\).

Fix any distance configuration with strict ordering, and let \(S\) be the set
of the \(K\) closest indices. Since \(K\) is odd, the KNN target is
\[
\operatorname{sign}\left(\sum_{i\in S}y_i\right)\in\{\pm1\}.
\]
Let
\[
R=\sum_{j\notin S}|B_j(d_j)|.
\]
For correctness under all choices of the outside labels, the following
conditions are necessary:
\[
\sum_{i\in S}y_i>0
\quad\Longrightarrow\quad
A(d_1,\ldots,d_N)+\sum_{i\in S}B_i(d_i)y_i>R,
\]
and
\[
\sum_{i\in S}y_i<0
\quad\Longrightarrow\quad
A(d_1,\ldots,d_N)+\sum_{i\in S}B_i(d_i)y_i<-R.
\]
Indeed, the outside labels contribute
\[
\sum_{j\notin S}B_j(d_j)y_j,
\]
whose possible perturbation has magnitude at most \(R\). The inequalities are
strict because \(\operatorname{sign}(0)=0\), while the KNN target is \(\pm1\).

We first record two consequences of these necessary conditions.

First, for every \(r\in S\),
\[
B_r(d_r)>R.
\]
To prove this, choose labels on \(S\setminus\{r\}\) such that
\[
\sum_{i\in S\setminus\{r\}}y_i=0,
\]
which is possible because \(K-1\) is even. Keeping these labels fixed, the KNN
target is \(+1\) when \(y_r=+1\), and \(-1\) when \(y_r=-1\). Hence
\[
A+\sum_{i\in S\setminus\{r\}}B_i(d_i)y_i+B_r(d_r)>R,
\]
and
\[
A+\sum_{i\in S\setminus\{r\}}B_i(d_i)y_i-B_r(d_r)<-R.
\]
Subtracting the second inequality from the first gives
\[
B_r(d_r)>R.
\]

Second, let
\[
m=\frac{K+1}{2}.
\]
For every \(P\subset S\) with \(|P|=m\), one has
\[
\sum_{i\in P}B_i(d_i)-\sum_{i\in S\setminus P}B_i(d_i)
>
R+|A(d_1,\ldots,d_N)|.
\]
Indeed, set \(y_i=+1\) for \(i\in P\) and \(y_i=-1\) for
\(i\in S\setminus P\). Since \(|P|=m\) and \(|S\setminus P|=m-1\), the
top-\(K\) label sum equals \(1\). Therefore,
\[
A+\sum_{i\in P}B_i(d_i)-\sum_{i\in S\setminus P}B_i(d_i)>R.
\]
Flipping all top-\(K\) labels gives target \(-1\), so
\[
A-\sum_{i\in P}B_i(d_i)+\sum_{i\in S\setminus P}B_i(d_i)<-R.
\]
Combining these two inequalities gives
\[
\sum_{i\in P}B_i(d_i)-\sum_{i\in S\setminus P}B_i(d_i)
>
R+|A(d_1,\ldots,d_N)|.
\]

We now construct fixed-margin configurations leading to a contradiction.

Choose disjoint index sets \(L,P\subset[N]\) such that
\[
|L|=m-1,
\qquad
|P|=m.
\]
This is possible because
\[
|L|+|P|=2m-1=K\le N-2<N.
\]

First define a distance configuration in which \(L\cup P\) is exactly the
top-\(K\) set. Enumerate
\[
L=\{\ell_1,\ldots,\ell_{m-1}\},
\qquad
P=\{p_1,\ldots,p_m\}.
\]
Assign
\[
d_{\ell_r}=(K+r)\gamma,
\qquad r=1,\ldots,m-1,
\]
and
\[
d_{p_s}=(K+m-1+s)\gamma,
\qquad s=1,\ldots,m.
\]
Thus the \(P\)-indices are farther than the \(L\)-indices. Assign the remaining
\(N-K\) indices distances
\[
(2K+1)\gamma,\ (2K+2)\gamma,\ \ldots,\ (2K+N-K)\gamma.
\]
All pairwise distance gaps are at least \(\gamma\). The largest distance used is
\[
(2K+N-K)\gamma=(N+K)\gamma\le 2-\gamma.
\]
Hence this is a valid instance in \(\cX_{\gamma}\). In this configuration,
the top-\(K\) set is \(L\cup P\).

Applying the balance consequence to the subset \(P\subset L\cup P\) gives
\begin{equation}
\sum_{p\in P}B_p(d_p)
>
\sum_{\ell\in L}B_\ell(d_\ell).
\label{eq:balance-P-L}
\end{equation}
Applying the dominance consequence to each top-\(K\) index gives
\begin{equation}
B_i(d_i)>0,
\qquad i\in L\cup P.
\label{eq:positive-LP}
\end{equation}
In particular,
\begin{equation}
B_p(d_p)>0,
\qquad p\in P.
\label{eq:positive-P}
\end{equation}

Next we cover \(P\) by subsets associated with indices in \(L\). Since
\(N-K\ge2\),
\[
|L|(N-K)=(m-1)(N-K)\ge 2(m-1)\ge m=|P|,
\]
where the last inequality uses \(m\ge2\). Hence, for each \(\ell\in L\), we may
choose a subset \(T_\ell\subseteq P\) such that
\[
|T_\ell|\le N-K
\]
and
\begin{equation}
\bigcup_{\ell\in L}T_\ell=P.
\label{eq:cover-P}
\end{equation}

Now fix \(\ell\in L\). We construct another fixed-margin distance configuration.
Keep token \(\ell\) at the same distance \(d_\ell\), and keep every token
\(p\in T_\ell\) at the same distance \(d_p\). Choose any \(K-1\) indices from
\[
[N]\setminus(\{\ell\}\cup T_\ell),
\]
which is possible because
\[
N-1-|T_\ell|\ge N-1-(N-K)=K-1.
\]
Assign these \(K-1\) indices the distances
\[
d_\ell-\gamma,\ d_\ell-2\gamma,\ \ldots,\ d_\ell-(K-1)\gamma.
\]
These distances are positive because the smallest possible value of \(d_\ell\)
is \((K+1)\gamma\), and hence
\[
d_\ell-(K-1)\gamma\ge 2\gamma>0.
\]
Assign all remaining indices distances
\[
(2K+1)\gamma,\ (2K+2)\gamma,\ \ldots
\]
using consecutive increments of \(\gamma\).

This configuration also has all pairwise distance gaps at least \(\gamma\).
Its largest distance is at most
\[
(N+K)\gamma\le 2-\gamma.
\]
Therefore it also belongs to \(\cX_{\gamma}\). In this configuration,
token \(\ell\) is exactly the \(K\)-th nearest neighbor, and every token
\(p\in T_\ell\) lies outside the top-\(K\) set because \(d_p>d_\ell\).

Applying the dominance consequence to token \(\ell\) gives
\[
B_\ell(d_\ell)
>
\sum_{p\in T_\ell}|B_p(d_p)|.
\]
Here the distances of the tokens \(p\in T_\ell\) are kept equal to the same
values \(d_p\) used in the first configuration. Hence the coefficient values
\(B_p(d_p)\) are the same as those in \eqref{eq:positive-P}. By
\eqref{eq:positive-P}, \(B_p(d_p)>0\) for every \(p\in P\), and therefore
\begin{equation}
B_\ell(d_\ell)
>
\sum_{p\in T_\ell}B_p(d_p).
\label{eq:ell-dominates-Tell}
\end{equation}

Summing \eqref{eq:ell-dominates-Tell} over all \(\ell\in L\) yields
\[
\sum_{\ell\in L}B_\ell(d_\ell)
>
\sum_{\ell\in L}\sum_{p\in T_\ell}B_p(d_p).
\]
By the covering property \eqref{eq:cover-P} and positivity
\eqref{eq:positive-P},
\[
\sum_{\ell\in L}\sum_{p\in T_\ell}B_p(d_p)
\ge
\sum_{p\in P}B_p(d_p).
\]
Therefore,
\begin{equation}
\sum_{\ell\in L}B_\ell(d_\ell)
>
\sum_{p\in P}B_p(d_p).
\label{eq:L-dominates-P}
\end{equation}
This contradicts \eqref{eq:balance-P-L}.

Thus the assumption that the fixed matrices \(\Vb,\Wb\) are correct on all
instances in \(\cX_{\gamma}\) is false. Consequently, for every finite
\(\Vb,\Wb\), there exists an instance in \(\cX_{\gamma}\) such that
\[
\operatorname{sign}\big(\mathrm{TF}(\Zb^{(0)},\btheta)_{d+1,-1}\big)
\ne
\operatorname{sign}\left(\sum_{k=1}^K y_{\pi(k)}\right).
\]
This proves the proposition.
\end{proof}

\section{Proof of Theorem \ref{thm:expressive power}}\label{appendix: proof of expressive power}

\begin{proof}
We first prove the following stronger vector-valued statement of Lemma~\ref{lem:expressive power mechanism}. By the sequential characterization of limits, it suffices to show that for every sequence $(\alpha_n,\beta_n)_{n\ge 1}$ such that $\alpha_n\to +\infty$, $\beta_n\to -\infty$, and $-\beta_n/\alpha_n\to 2$, one has
\[
\mathrm{TF}(\hat{\mathbf{Z}}_n^{(K-1)},\theta_n)_{:,-1}-\mathbf{q}^{(K)}\to 0\quad \text{a.s.}
\]
where $\qb^{(K)} = [\xb_q^\top,\sum_{k=1}^K y_{\pi(k)},\pb_{N+1}^\top+\sum_{k=1}^K\pb_{\pi(k)}^\top]^\top$. Fix such a sequence \((\alpha_n,\beta_n)_{n\ge1}\). For each \(n\), define
\[
\mathbf W_n\coloneqq \diag(\alpha_n I_d,0,\beta_n I_{N+1}),
\qquad
\boldsymbol\theta_n\coloneqq (\mathbf V,\mathbf W_n),
\]
where
$
\mathbf V=\diag(0,1,I_{N+1}).
$
Define the append-style rollout recursively by
\[
\widehat{\mathbf q}^{(0)}_n\coloneqq \mathbf q^{(0)},
\qquad
\widehat{\mathbf Z}^{(0)}_n\coloneqq \mathbf Z^{(0)}=[\mathbf H,\mathbf q^{(0)}],
\]
and, for \(r\ge1\),
\begin{equation}\label{eq:append-recursion-final}
\widehat{\mathbf q}^{(r)}_n
\coloneqq
\mathrm{TF}(\widehat{\mathbf Z}^{(r-1)}_n,\boldsymbol\theta_n)_{:,-1},
\qquad
\widehat{\mathbf Z}^{(r)}_n
\coloneqq
[\mathbf H,\widehat{\mathbf q}^{(0)}_n,\widehat{\mathbf q}^{(1)}_n,\dots,\widehat{\mathbf q}^{(r)}_n].
\end{equation}
Thus, in particular,
\[
\widehat{\mathbf q}^{(K)}_n
=
\mathrm{TF}(\widehat{\mathbf Z}^{(K-1)}_n,\boldsymbol\theta_n)_{:,-1}.
\]

For \(r=0,\dots,K\), write
\begin{equation}\label{eq:qr-def-final}
\mathbf q^{(r)}
=
\begin{bmatrix}
\mathbf x_q\\
\sum_{m=1}^r y_{\pi(m)}\\
\mathbf p_{N+1}+\sum_{m=1}^r \mathbf p_{\pi(m)}
\end{bmatrix}.
\end{equation}

We now establish the convergence through four lemmas.

\begin{lemma}
\label{lem:expr_x_block_exact_polished}
For every \(r\ge0\) and every \(n\),
\[
\bigl(\widehat{\mathbf q}^{(r)}_n\bigr)_x=\mathbf x_q.
\]
\end{lemma}

\begin{proof}
We prove by induction on \(r\). For \(r=0\), by definition,
\[
\widehat{\mathbf q}^{(0)}_n=\mathbf q^{(0)}
=
\begin{bmatrix}
\mathbf x_q\\
0\\
\mathbf p_{N+1}
\end{bmatrix},
\]
so \((\widehat{\mathbf q}^{(0)}_n)_x=\mathbf x_q\).

Assume now that \((\widehat{\mathbf q}^{(r-1)}_n)_x=\mathbf x_q\). Using the transformer formula,
\[
\widehat{\mathbf q}^{(r)}_n
=
\widehat{\mathbf q}^{(r-1)}_n
+
(\mathbf V\widehat{\mathbf Z}^{(r-1)}_n)\,
\mathrm{softmax}\!\left(
(\widehat{\mathbf Z}^{(r-1)}_n)^\top
\mathbf W_n
(\widehat{\mathbf Z}^{(r-1)}_n)_{:,-1}
\right).
\]
Since the \(x\)-row block of \(\mathbf V=\diag(0,1,I_{N+1})\) is identically zero, the additive attention correction has zero \(x\)-block. Hence
\[
(\widehat{\mathbf q}^{(r)}_n)_x
=
(\widehat{\mathbf q}^{(r-1)}_n)_x
=
\mathbf x_q.
\]
This completes the induction.
\end{proof}

\begin{lemma}
\label{lem:expr_pathwise_bounds_polished}
For every \(0\le r\le K\) and every \(n\),
\[
\left|(\widehat{\mathbf q}^{(r)}_n)_y\right|\le 2^r-1,
\qquad
\left\|(\widehat{\mathbf q}^{(r)}_n)_p\right\|\le 2^r.
\]
\end{lemma}

\begin{proof}
We prove the two bounds simultaneously by induction on \(r\). For \(r=0\), by definition,
\[
\widehat{\mathbf q}^{(0)}_n=\mathbf q^{(0)}
=
\begin{bmatrix}
\mathbf x_q\\
0\\
\mathbf p_{N+1}
\end{bmatrix}.
\]
Therefore
\[
\left|(\widehat{\mathbf q}^{(0)}_n)_y\right|=0=2^0-1,
\qquad
\left\|(\widehat{\mathbf q}^{(0)}_n)_p\right\|=\|\mathbf p_{N+1}\|=1=2^0.
\]

Assume that the claim holds at depth \(r-1\), where \(r\ge1\). Let
\[
a^{(r)}_n
\coloneqq
\mathrm{softmax}\!\left(
(\widehat{\mathbf Z}^{(r-1)}_n)^\top
\mathbf W_n
(\widehat{\mathbf Z}^{(r-1)}_n)_{:,-1}
\right).
\]
Then
\[
\widehat{\mathbf q}^{(r)}_n
=
\widehat{\mathbf q}^{(r-1)}_n
+
(\mathbf V\widehat{\mathbf Z}^{(r-1)}_n)a^{(r)}_n.
\]

For the \(y\)-block, we may write
\[
(\widehat{\mathbf q}^{(r)}_n)_y
=
(\widehat{\mathbf q}^{(r-1)}_n)_y
+
\sum_j a^{(r)}_{j,n}\,y^{\mathrm{ctx}}_{j,n},
\]
where \(y^{\mathrm{ctx}}_{j,n}\) denotes the \(y\)-block of the \(j\)-th context token in \(\widehat{\mathbf Z}^{(r-1)}_n\). Each context token is either an original token \(h_i\), whose \(y\)-block equals \(y_i\) and hence has magnitude at most 1, or a previously generated token \(\widehat{\mathbf q}^{(m)}_n\) with \(m\le r-1\), whose \(y\)-block has magnitude at most \(2^{r-1}-1\) by the induction hypothesis. Therefore
\[
|y^{\mathrm{ctx}}_{j,n}|
\le
\max\{1,2^{r-1}-1\}
\le 2^{r-1}
\qquad\text{for every }j.
\]
Since \(a^{(r)}_n\) is a probability vector,
\[
\left|\sum_j a^{(r)}_{j,n}y^{\mathrm{ctx}}_{j,n}\right|
\le
\sum_j a^{(r)}_{j,n}\,2^{r-1}
=
2^{r-1}.
\]
Using the induction hypothesis,
\[
|(\widehat{\mathbf q}^{(r-1)}_n)_y|\le 2^{r-1}-1,
\]
we obtain
\[
|(\widehat{\mathbf q}^{(r)}_n)_y|
\le
2^{r-1}-1+2^{r-1}
=
2^r-1.
\]

For the \(p\)-block, similarly,
\[
(\widehat{\mathbf q}^{(r)}_n)_p
=
(\widehat{\mathbf q}^{(r-1)}_n)_p
+
\sum_j a^{(r)}_{j,n}\,p^{\mathrm{ctx}}_{j,n},
\]
where \(p^{\mathrm{ctx}}_{j,n}\) denotes the \(p\)-block of the \(j\)-th context token. Each context token is either an original token \(h_i\), whose \(p\)-block is \(\mathbf p_i\) and has norm \(1\), or a previously generated token \(\widehat{\mathbf q}^{(m)}_n\) with \(m\le r-1\), whose \(p\)-block has norm at most \(2^{r-1}\) by the induction hypothesis. Hence
\[
\|p^{\mathrm{ctx}}_{j,n}\|
\le
\max\{1,2^{r-1}\}
=
2^{r-1}
\qquad\text{for every }j,
\]
and therefore
\[
\left\|\sum_j a^{(r)}_{j,n}p^{\mathrm{ctx}}_{j,n}\right\|
\le
\sum_j a^{(r)}_{j,n}\,2^{r-1}
=
2^{r-1}.
\]
Using the induction hypothesis once more,
\[
\|(\widehat{\mathbf q}^{(r-1)}_n)_p\|\le 2^{r-1},
\]
we conclude
\[
\|(\widehat{\mathbf q}^{(r)}_n)_p\|
\le
2^{r-1}+2^{r-1}
=
2^r.
\]
This completes the induction.
\end{proof}

\begin{lemma}
\label{lem:expr_one_step_consistency_polished}
Fix \(s\in\{1,\dots,K\}\). Let
\[
\Omega_*
\coloneqq
\left\{
\mathbf x_{\pi(1)}^\top \mathbf x_q>\cdots>\mathbf x_{\pi(N)}^\top \mathbf x_q
\right\}
\cap
\left\{
\mathbf x_{\pi(r)}^\top \mathbf x_q>-1\ \text{for all }1\le r\le K
\right\}.
\]
On \(\Omega_*\), define
\[
\Delta_s\coloneqq
\begin{cases}
\displaystyle
\min_{j\in [N]\setminus\{\pi(1),\dots,\pi(s)\}}
\bigl(\mathbf x_{\pi(s)}^\top \mathbf x_q-\mathbf x_j^\top \mathbf x_q\bigr),
& s<N,\\[2mm]
1,& s=N.
\end{cases}
\]
Then \(\Delta_s>0\) and \(1+\mathbf x_{\pi(s)}^\top \mathbf x_q>0\).

Assume that
\[
\widehat{\mathbf q}^{(m)}_n\to \mathbf q^{(m)}
\qquad\text{for every }0\le m\le s-1.
\]
Then on \(\Omega_*\),
\[
\widehat{\mathbf q}^{(s)}_n\to \mathbf q^{(s)}.
\]
\end{lemma}

\begin{proof}
We work pointwise on \(\Omega_*\). At inference step \(s\), the current input is
\[
\widehat{\mathbf Z}^{(s-1)}_n
=
[\mathbf H,\widehat{\mathbf q}^{(0)}_n,\widehat{\mathbf q}^{(1)}_n,\dots,\widehat{\mathbf q}^{(s-1)}_n],
\]
and the query token is the last token \(\widehat{\mathbf q}^{(s-1)}_n\).

\medskip
\noindent
\textbf{Step 1: exact score formulas.}

For an original token \(h_i\), the score at step \(s\) is
\begin{equation}\label{eq:orig-score-polished}
\widehat s^{(s)}_{i,n}
=
\alpha_n\,\mathbf x_i^\top \mathbf x_q
+
\beta_n\,\mathbf p_i^\top (\widehat{\mathbf q}^{(s-1)}_n)_p.
\end{equation}
For a previously generated token \(\widehat{\mathbf q}^{(m)}_n\) with \(0\le m\le s-1\), the score is
\begin{equation}\label{eq:gen-score-polished}
\widehat s^{(s)}_{m,n}
=
\alpha_n
+
\beta_n\,\bigl\langle (\widehat{\mathbf q}^{(m)}_n)_p,(\widehat{\mathbf q}^{(s-1)}_n)_p\bigr\rangle.
\end{equation}
These identities follow directly from the block form
\[
\mathbf W_n=\diag(\alpha_n I_d,0,\beta_n I_{N+1}),
\]
together with Lemma~\ref{lem:expr_x_block_exact_polished}, which implies that every generated token has \(x\)-block exactly equal to \(\mathbf x_q\).

\medskip
\noindent
\textbf{Step 2: normalized score limits.}

Define
\[
\widehat \sigma^{(s)}_{i,n}\coloneqq \frac{\widehat s^{(s)}_{i,n}}{\alpha_n}
\qquad (i\in[N]),
\qquad
\widehat \sigma^{(s)}_{m,n}\coloneqq \frac{\widehat s^{(s)}_{m,n}}{\alpha_n}
\qquad (0\le m\le s-1).
\]
Using \eqref{eq:orig-score-polished},
\[
\widehat \sigma^{(s)}_{i,n}
=
\mathbf x_i^\top \mathbf x_q
-
\left(-\frac{\beta_n}{\alpha_n}\right)\mathbf p_i^\top (\widehat{\mathbf q}^{(s-1)}_n)_p.
\]
By the induction hypothesis,
\[
(\widehat{\mathbf q}^{(s-1)}_n)_p
\to
\mathbf p_{N+1}+\sum_{r=1}^{s-1}\mathbf p_{\pi(r)}.
\]
Since the signed one-hot vectors are orthonormal,
\[
\mathbf p_i^\top
\left(
\mathbf p_{N+1}+\sum_{r=1}^{s-1}\mathbf p_{\pi(r)}
\right)
=
\begin{cases}
1,& i\in \{\pi(1),\dots,\pi(s-1)\},\\
0,& i\notin \{\pi(1),\dots,\pi(s-1)\}.
\end{cases}
\]
Together with \(-\beta_n/\alpha_n\to 2\), this implies
\begin{equation}\label{eq:orig-limit-polished}
\widehat \sigma^{(s)}_{i,n}\to
\begin{cases}
\mathbf x_i^\top \mathbf x_q-2,& i\in \{\pi(1),\dots,\pi(s-1)\},\\[1mm]
\mathbf x_i^\top \mathbf x_q,& i\notin \{\pi(1),\dots,\pi(s-1)\}.
\end{cases}
\end{equation}
In particular, since \(\pi(s)\notin \{\pi(1),\dots,\pi(s-1)\}\),
\begin{equation}\label{eq:desired-limit-polished}
\widehat \sigma^{(s)}_{\pi(s),n}\to \mathbf x_{\pi(s)}^\top \mathbf x_q.
\end{equation}

For \(0\le m\le s-1\), using \eqref{eq:gen-score-polished},
\[
\widehat \sigma^{(s)}_{m,n}
=
1-\left(-\frac{\beta_n}{\alpha_n}\right)
\bigl\langle (\widehat{\mathbf q}^{(m)}_n)_p,(\widehat{\mathbf q}^{(s-1)}_n)_p\bigr\rangle.
\]
Again by the induction hypothesis,
\[
(\widehat{\mathbf q}^{(m)}_n)_p
\to
\mathbf p_{N+1}+\sum_{r=1}^{m}\mathbf p_{\pi(r)},
\qquad
(\widehat{\mathbf q}^{(s-1)}_n)_p
\to
\mathbf p_{N+1}+\sum_{r=1}^{s-1}\mathbf p_{\pi(r)}.
\]
Since
\[
\left\langle
\mathbf p_{N+1}+\sum_{r=1}^{m}\mathbf p_{\pi(r)},
\,
\mathbf p_{N+1}+\sum_{r=1}^{s-1}\mathbf p_{\pi(r)}
\right\rangle
=
m+1,
\]
we obtain
\begin{equation}\label{eq:gen-limit-polished}
\widehat \sigma^{(s)}_{m,n}\to 1-2(m+1)=-(2m+1).
\end{equation}

\medskip
\noindent
\textbf{Step 3: deterministic margin separation.}

We compare the limiting normalized score \(\mathbf x_{\pi(s)}^\top \mathbf x_q\) of the desired token with every competitor.

If \(j\in [N]\setminus\{\pi(1),\dots,\pi(s)\}\), then \(j\neq \pi(s)\), and by definition of \(\Delta_s\),
\[
\mathbf x_{\pi(s)}^\top \mathbf x_q-\mathbf x_j^\top \mathbf x_q\ge \Delta_s>0.
\]

If \(j\in \{\pi(1),\dots,\pi(s-1)\}\), then by \eqref{eq:orig-limit-polished},
\[
\mathbf x_{\pi(s)}^\top \mathbf x_q-(\mathbf x_j^\top \mathbf x_q-2)
=
2+\mathbf x_{\pi(s)}^\top \mathbf x_q-\mathbf x_j^\top \mathbf x_q
\ge
1+\mathbf x_{\pi(s)}^\top \mathbf x_q
>0.
\]

If \(0\le m\le s-1\), then by \eqref{eq:gen-limit-polished},
\[
\mathbf x_{\pi(s)}^\top \mathbf x_q-(-(2m+1))
=
\mathbf x_{\pi(s)}^\top \mathbf x_q+2m+1
\ge
1+\mathbf x_{\pi(s)}^\top \mathbf x_q
>0.
\]

Thus every competitor lies at least
\[
\min\Bigl\{\Delta_s,\ 1+\mathbf x_{\pi(s)}^\top \mathbf x_q\Bigr\}
\]
below the desired token at the level of limiting normalized scores. Since there are only finitely many competitors, \eqref{eq:orig-limit-polished} and \eqref{eq:gen-limit-polished} imply that there exists \(N_s\) such that for all \(n\ge N_s\),
\begin{equation}\label{eq:uniform-gap-polished}
\widehat \sigma^{(s)}_{\pi(s),n}-\widehat \sigma^{(s)}_{j,n}
\ge
\frac12\min\Bigl\{\Delta_s,\ 1+\mathbf x_{\pi(s)}^\top \mathbf x_q\Bigr\}
\qquad\text{for every competitor }j\neq \pi(s).
\end{equation}
Multiplying by \(\alpha_n\), we obtain
\begin{equation}\label{eq:raw-gap-polished}
\widehat s^{(s)}_{\pi(s),n}-\widehat s^{(s)}_{j,n}
\ge
\frac{\alpha_n}{2}\min\Bigl\{\Delta_s,\ 1+\mathbf x_{\pi(s)}^\top \mathbf x_q\Bigr\}
\qquad\text{for every competitor }j\neq \pi(s).
\end{equation}

\medskip
\noindent
\textbf{Step 4: softmax concentration.}

Let \(a^{(s)}_{j,n}\) denote the step-\(s\) attention weights on the context tokens of \(\widehat{\mathbf Z}^{(s-1)}_n\). By the softmax formula and \eqref{eq:raw-gap-polished},
\[
1-a^{(s)}_{\pi(s),n}
=
\frac{\sum_{j\neq\pi(s)} e^{\widehat s^{(s)}_{j,n}-\widehat s^{(s)}_{\pi(s),n}}}
{1+\sum_{j\neq\pi(s)} e^{\widehat s^{(s)}_{j,n}-\widehat s^{(s)}_{\pi(s),n}}}
\le
(N+s)\exp\!\left(
-\frac{\alpha_n}{2}\min\Bigl\{\Delta_s,\ 1+\mathbf x_{\pi(s)}^\top \mathbf x_q\Bigr\}
\right).
\]
In particular,
\begin{equation}\label{eq:competitor-weight-polished}
a^{(s)}_{j,n}
\le
\exp\!\left(
-\frac{\alpha_n}{2}\min\Bigl\{\Delta_s,\ 1+\mathbf x_{\pi(s)}^\top \mathbf x_q\Bigr\}
\right)
\qquad(j\neq\pi(s)).
\end{equation}
Hence
\begin{equation}\label{eq:weight-limit-polished}
a^{(s)}_{\pi(s),n}\to1,
\qquad
a^{(s)}_{j,n}\to0\quad(j\neq\pi(s)).
\end{equation}

\medskip
\noindent
\textbf{Step 5: convergence of the output token.}

The output token satisfies
\[
\widehat{\mathbf q}^{(s)}_n
=
\widehat{\mathbf q}^{(s-1)}_n
+
(\mathbf V\widehat{\mathbf Z}^{(s-1)}_n)a^{(s)}_n.
\]
The \(x\)-block is exact by Lemma~\ref{lem:expr_x_block_exact_polished}. We therefore only need to treat the \(y\)- and \(p\)-blocks.

For the \(y\)-block,
\[
(\widehat{\mathbf q}^{(s)}_n)_y
=
(\widehat{\mathbf q}^{(s-1)}_n)_y
+
a^{(s)}_{\pi(s),n}y_{\pi(s)}
+
\sum_{j\neq\pi(s)} a^{(s)}_{j,n}y^{\mathrm{ctx}}_{j,n}.
\]
By the induction hypothesis,
\[
(\widehat{\mathbf q}^{(s-1)}_n)_y\to \sum_{i=1}^{s-1} y_{\pi(i)},
\]
and by \eqref{eq:weight-limit-polished},
\[
a^{(s)}_{\pi(s),n}y_{\pi(s)}\to y_{\pi(s)}.
\]
By Lemma~\ref{lem:expr_pathwise_bounds_polished}, every competitor context token has \(y\)-block magnitude at most \(2^{s-1}\). Therefore, using \eqref{eq:competitor-weight-polished},
\[
\left|\sum_{j\neq\pi(s)} a^{(s)}_{j,n}y^{\mathrm{ctx}}_{j,n}\right|
\le
(N+s)2^{s-1}
\exp\!\left(
-\frac{\alpha_n}{2}\min\Bigl\{\Delta_s,\ 1+\mathbf x_{\pi(s)}^\top \mathbf x_q\Bigr\}
\right)\to 0.
\]
Thus
\[
(\widehat{\mathbf q}^{(s)}_n)_y\to \sum_{i=1}^{s-1} y_{\pi(i)}+y_{\pi(s)}=\sum_{i=1}^{s} y_{\pi(i)}.
\]

For the \(p\)-block,
\[
(\widehat{\mathbf q}^{(s)}_n)_p
=
(\widehat{\mathbf q}^{(s-1)}_n)_p
+
a^{(s)}_{\pi(s),n}\mathbf p_{\pi(s)}
+
\sum_{j\neq\pi(s)} a^{(s)}_{j,n}p^{\mathrm{ctx}}_{j,n}.
\]
By the induction hypothesis,
\[
(\widehat{\mathbf q}^{(s-1)}_n)_p
\to
\mathbf p_{N+1}+\sum_{r=1}^{s-1}\mathbf p_{\pi(r)},
\]
and by \eqref{eq:weight-limit-polished},
\[
a^{(s)}_{\pi(s),n}\mathbf p_{\pi(s)}\to \mathbf p_{\pi(s)}.
\]
Again by Lemma~\ref{lem:expr_pathwise_bounds_polished}, every competitor context token has \(p\)-block norm at most \(2^{s-1}\). Hence
\[
\left\|\sum_{j\neq\pi(s)} a^{(s)}_{j,n}p^{\mathrm{ctx}}_{j,n}\right\|
\le
(N+s)2^{s-1}
\exp\!\left(
-\frac{\alpha_n}{2}\min\Bigl\{\Delta_s,\ 1+\mathbf x_{\pi(s)}^\top \mathbf x_q\Bigr\}
\right)\to 0.
\]
Therefore
\[
(\widehat{\mathbf q}^{(s)}_n)_p
\to
\mathbf p_{N+1}+\sum_{r=1}^{s}\mathbf p_{\pi(r)}.
\]

Combining the \(x\)-, \(y\)-, and \(p\)-block conclusions gives
\[
\widehat{\mathbf q}^{(s)}_n\to \mathbf q^{(s)}.
\]
This proves the lemma.
\end{proof}

\begin{lemma}
\label{lem:as_convergence}
Under the assumptions of Theorem~\ref{thm:expressive power},
\[
\widehat{\mathbf q}^{(K)}_n-\mathbf q^{(K)}
\to 0 \quad \text{a.s.}.
\]
Equivalently,
\[
\mathrm{TF}(\widehat{\mathbf Z}^{(K-1)}_n,\boldsymbol\theta_n)_{:,-1}
-
\mathbf q^{(K)}
\to 0 \quad \text{a.s.}.
\]
\end{lemma}

\begin{proof}
Under the assumptions of Theorem~\ref{thm:expressive power}, we have
\[
\|\mathbf x_q-\mathbf x_{\pi(1)}\|<\cdots<\|\mathbf x_q-\mathbf x_{\pi(N)}\|<2
\qquad\text{a.s.},
\]
and, since \(\mathbf x_i,\mathbf x_q\in \SSS^{d-1}\), the identity
\[
\|\mathbf x_q-\mathbf x_i\|^2=2-2\mathbf x_i^\top \mathbf x_q
\]
implies
\[
\mathbf x_{\pi(1)}^\top \mathbf x_q>\cdots>\mathbf x_{\pi(N)}^\top \mathbf x_q>-1
\qquad\text{a.s.}
\]
Hence \(\mathbb P(\Omega_*)=1\). 
By definition,
\[
\widehat{\mathbf q}^{(0)}_n=\mathbf q^{(0)}
\qquad\text{for all }n.
\]
Applying Lemma~\ref{lem:expr_one_step_consistency_polished} successively for \(s=1,2,\dots,K\), we obtain
\[
\widehat{\mathbf q}^{(r)}_n\to \mathbf q^{(r)}
\qquad\text{for every }0\le r\le K
\]
on \(\Omega_*\). Since \(\mathbb P(\Omega_*)=1\), this holds almost surely. In particular,
\[
\widehat{\mathbf q}^{(K)}_n\to \mathbf q^{(K)}
\qquad\text{almost surely}.
\]
The equivalent statement follows from the definition
\[
\widehat{\mathbf q}^{(K)}_n
=
\mathrm{TF}(\widehat{\mathbf Z}^{(K-1)}_n,\boldsymbol\theta_n)_{:,-1}.
\]



\end{proof}

By Lemma~\ref{lem:as_convergence}, for every sequence $(\alpha_n,\beta_n)_{n\ge1}$ such that
\[
\alpha_n\to+\infty,\qquad \beta_n\to-\infty,\qquad -\beta_n/\alpha_n\to2,
\]
we have
\[
\mathrm{TF}(\widehat{\mathbf Z}^{(K-1)}_n,\boldsymbol\theta_n)_{:,-1}
-
\mathbf q^{(K)}
\to 0
\qquad\text{a.s.}
\]
Taking the \((d+1)\)-st coordinate yields
\[
\mathrm{TF}(\widehat{\mathbf Z}^{(K-1)}_n,\boldsymbol\theta_n)_{d+1,-1}
-
(\mathbf q^{(K)})_{d+1}
\to 0
\qquad\text{a.s.}
\]
By the definition of \(\mathbf q^{(K)}\),
\[
(\mathbf q^{(K)})_{d+1}=\sum_{k=1}^K y_{\pi(k)}.
\]
Hence
\[
\mathrm{TF}(\widehat{\mathbf Z}^{(K-1)}_n,\boldsymbol\theta_n)_{d+1,-1}
-
\sum_{k=1}^K y_{\pi(k)}
\to 0
\qquad\text{a.s.}
\]

Since
\[
\hat y_K=\sign\!\left(\sum_{k=1}^K y_{\pi(k)}\right),
\]
it follows that
\[
\mathbf 1\left\{
\mathrm{TF}(\widehat{\mathbf Z}^{(K-1)}_n,\boldsymbol\theta_n)_{d+1,-1}\cdot\hat y_K<0
\right\}
\to 0
\qquad\text{a.s.}
\]
Indeed, if \(\sum_{k=1}^K y_{\pi(k)}=0\), then \(\hat y_K=0\), and the indicator above is identically zero for every \(n\). If \(\sum_{k=1}^K y_{\pi(k)}>0\), then
\[
\mathrm{TF}(\widehat{\mathbf Z}^{(K-1)}_n,\boldsymbol\theta_n)_{d+1,-1}
\to
\sum_{k=1}^K y_{\pi(k)}
>0
\qquad\text{a.s.},
\]
so
\[
\mathrm{TF}(\widehat{\mathbf Z}^{(K-1)}_n,\boldsymbol\theta_n)_{d+1,-1}>0
\]
for all sufficiently large \(n\) almost surely, and therefore
\[
\mathrm{TF}(\widehat{\mathbf Z}^{(K-1)}_n,\boldsymbol\theta_n)_{d+1,-1}\cdot\hat y_K<0
\]
cannot occur for all sufficiently large \(n\) almost surely. The case
\(\sum_{k=1}^K y_{\pi(k)}<0\) is analogous.

Since the indicator is bounded by \(1\), the dominated convergence theorem yields
\[
\mathbb P\!\left(
\mathrm{TF}(\widehat{\mathbf Z}^{(K-1)}_n,\boldsymbol\theta_n)_{d+1,-1}\cdot\hat y_K<0
\right)\to 0.
\]
Because the sequence \((\alpha_n,\beta_n)\) was arbitrary, the sequential characterization of limits implies
\[
\mathbb P\!\left(
\mathrm{TF}(\widehat{\mathbf Z}^{(K-1)},\boldsymbol\theta)_{d+1,-1}\cdot\hat y_K<0
\right)\to 0
\]
as \(\alpha\to+\infty\), \(\beta\to-\infty\), and \(-\beta/\alpha\to2\). This is exactly the statement of
Theorem~\ref{thm:expressive power}.
\end{proof}

\section{Proof of Proposition \ref{prop:positive_margin}}\label{appendix: proof of positive_margin}
\begin{proof}
We use the same rollout notation as in the proof of
Theorem~\ref{thm:expressive power}: \(\widehat{\qb}^{(r)}\) denotes the
actual token generated at step \(r\), \(\widehat{\Zb}^{(r)}\) denotes the
corresponding appended sequence, and \(\qb^{(r)}\) denotes the ideal token,
\[
\qb^{(r)}
=
\begin{bmatrix}
\xb_q\\
\sum_{m=1}^r y_{\pi(m)}\\
\pb_{N+1}+\sum_{m=1}^r\pb_{\pi(m)}
\end{bmatrix},
\qquad r=0,\ldots,K .
\]
We write \((\cdot)_p\) for the positional block. 
Fix an arbitrary instance in \(\cX_\gamma\) with
\(\sum_{k=1}^K y_{\pi(k)}\neq 0\). We first
convert the distance margin into an inner-product margin. Since
\(\xb_i,\xb_q\in \mathbb S^{d-1}\),
\[
\xb_i^\top\xb_q
=
1-\frac{\|\xb_q-\xb_i\|^2}{2}.
\]
Let
\[
d_k:=\|\xb_q-\xb_{\pi(k)}\|,\qquad k=1,\ldots,N .
\]
For \(k=1,\ldots,N-1\), the assumption \(d_{k+1}-d_k\ge \gamma\) implies
\[
\begin{aligned}
\xb_{\pi(k)}^\top\xb_q-\xb_{\pi(k+1)}^\top\xb_q
&=
\frac{d_{k+1}^2-d_k^2}{2}  \\
&=
\frac{(d_{k+1}-d_k)(d_{k+1}+d_k)}{2}
\ge
\frac{\gamma^2}{2}.
\end{aligned}
\]
Moreover, \(2-d_N\ge \gamma\) implies \(d_N\le 2-\gamma\), and hence
\[
\xb_{\pi(N)}^\top\xb_q
=
1-\frac{d_N^2}{2}
\ge
1-\frac{(2-\gamma)^2}{2}
=
-1+2\gamma-\frac{\gamma^2}{2}.
\]
Define
\[
\eta_\gamma
:=
\min\left\{
\frac{\gamma^2}{2},
\,2\gamma-\frac{\gamma^2}{2}
\right\}.
\]
If \(\cX_\gamma\) is nonempty,
then \(0<\gamma\le 2\), so \(\eta_\gamma>0\). Therefore,
\[
\xb_{\pi(k)}^\top\xb_q-\xb_{\pi(k+1)}^\top\xb_q
\ge \eta_\gamma,
\qquad k=1,\ldots,N-1,
\]
and
\[
\xb_{\pi(N)}^\top\xb_q
\ge -1+\eta_\gamma .
\]

The arguments proving Lemma~\eqref{lem:expr_x_block_exact_polished} and Lemma~\eqref{lem:expr_pathwise_bounds_polished} are purely algebraic and apply
verbatim to every finite choice of \((\alpha,\beta)\). By Lemma~\eqref{lem:expr_x_block_exact_polished}, the
feature block of every actual generated token equals \(\xb_q\). By
Lemma~\eqref{lem:expr_pathwise_bounds_polished}, for all generated tokens up to depth \(K\), the label block and
positional block are bounded by constants depending only on \(K\). Hence there
exists \(B_{N,K}>0\), depending only on \((N,K)\), such that every token
appearing in \(\widehat{\Zb}^{(r)}\), \(0\le r\le K-1\), and every original
token has norm at most \(B_{N,K}\).

Set
\[
C'_{N,K}:=2B_{N,K}(N+K).
\]
We will prove by induction that
\[
\|\widehat{\qb}^{(r)}-\qb^{(r)}\|
\le
C_{N,K}e^{-\alpha\eta_\gamma/8},
\qquad r=0,\ldots,K,
\]
for a constant \(C_{N,K}>0\) depending only on \((N,K)\), to be chosen below.

Let
\[
E_\alpha:=e^{-\alpha\eta_\gamma/8}.
\]
Assume that, for some \(1\le r\le K\),
\[
\|\widehat{\qb}^{(\ell)}-\qb^{(\ell)}\|
\le
C_{N,K}E_\alpha,
\qquad \ell=0,\ldots,r-1 .
\]

Let the attention scores at step \(r\) be computed with respect to the last
token \(\widehat{\qb}^{(r-1)}\). For an original token \(\hb_i\),
\[
\widehat s_i^{(r)}
=
\alpha \xb_i^\top\xb_q
+
\beta \pb_i^\top(\widehat{\qb}^{(r-1)})_p,
\]
and for a previously generated token \(\widehat{\qb}^{(m)}\),
\(0\le m\le r-1\),
\[
\widehat s_m^{(r)}
=
\alpha
+
\beta
\left\langle
(\widehat{\qb}^{(m)})_p,
(\widehat{\qb}^{(r-1)})_p
\right\rangle .
\]

We compare these normalized scores with the limiting normalized scores obtained
from the ideal rollout and the limiting ratio \(\beta/\alpha=-2\). For original
tokens, the corresponding limiting normalized score is
\[
\begin{cases}
\xb_i^\top\xb_q-2,
& i\in\{\pi(1),\ldots,\pi(r-1)\},\\
\xb_i^\top\xb_q,
& i\notin\{\pi(1),\ldots,\pi(r-1)\}.
\end{cases}
\]
For a previously generated token \(\qb^{(m)}\), \(0\le m\le r-1\), the
corresponding limiting normalized score is
\[
1-2(m+1)=-(2m+1),
\]
because
\[
\left\langle
(\qb^{(m)})_p,
(\qb^{(r-1)})_p
\right\rangle
=
m+1 .
\]

Assume \(|\beta/\alpha+2|\le\delta_0\le1\). Then
\[
|\beta/\alpha|\le 3 .
\]
For an original token \(\hb_i\), using
\[
\pb_i^\top(\qb^{(r-1)})_p
=
\begin{cases}
1,& i\in\{\pi(1),\ldots,\pi(r-1)\},\\
0,& i\notin\{\pi(1),\ldots,\pi(r-1)\},
\end{cases}
\]
we have
\[
\begin{aligned}
&\left|
\frac{\widehat s_i^{(r)}}{\alpha}
-
\left(
\xb_i^\top\xb_q
-
2\pb_i^\top(\qb^{(r-1)})_p
\right)
\right|  =
\left|
\frac{\beta}{\alpha}\pb_i^\top(\widehat{\qb}^{(r-1)})_p
+
2\pb_i^\top(\qb^{(r-1)})_p
\right| \\
& \le
\left|\frac{\beta}{\alpha}\right|
\left|
\pb_i^\top
\big(
(\widehat{\qb}^{(r-1)})_p-(\qb^{(r-1)})_p
\big)
\right|
+
\left|\frac{\beta}{\alpha}+2\right|
\left|
\pb_i^\top(\qb^{(r-1)})_p
\right| \le
3C_{N,K}E_\alpha+\delta_0 .
\end{aligned}
\]

For a previously generated token \(\widehat{\qb}^{(m)}\), \(0\le m\le r-1\),
we similarly compare \(\widehat s_m^{(r)}/\alpha\) with \(-(2m+1)\). Since
\[
\left\langle
(\qb^{(m)})_p,
(\qb^{(r-1)})_p
\right\rangle
=
m+1,
\]
we have
\[
\begin{aligned}
&\left|
\frac{\widehat s_m^{(r)}}{\alpha}
-
\bigl(-(2m+1)\bigr)
\right|  =
\left|
\frac{\beta}{\alpha}
\left\langle
(\widehat{\qb}^{(m)})_p,
(\widehat{\qb}^{(r-1)})_p
\right\rangle
+
2
\left\langle
(\qb^{(m)})_p,
(\qb^{(r-1)})_p
\right\rangle
\right| \\
& \le
3
\left|
\left\langle
(\widehat{\qb}^{(m)})_p,
(\widehat{\qb}^{(r-1)})_p
\right\rangle
-
\left\langle
(\qb^{(m)})_p,
(\qb^{(r-1)})_p
\right\rangle
\right|
+
\left|\frac{\beta}{\alpha}+2\right|
\left|
\left\langle
(\qb^{(m)})_p,
(\qb^{(r-1)})_p
\right\rangle
\right| .
\end{aligned}
\]
The first term is controlled by the induction hypothesis and the boundedness
from Lemma~\ref{lem:expr_pathwise_bounds_polished}. Indeed,
\[
\begin{aligned}
&\left|
\left\langle
(\widehat{\qb}^{(m)})_p,
(\widehat{\qb}^{(r-1)})_p
\right\rangle
-
\left\langle
(\qb^{(m)})_p,
(\qb^{(r-1)})_p
\right\rangle
\right| \\
& \le
\|(\widehat{\qb}^{(m)})_p-(\qb^{(m)})_p\|
\|(\widehat{\qb}^{(r-1)})_p\|
+
\|(\qb^{(m)})_p\|
\|(\widehat{\qb}^{(r-1)})_p-(\qb^{(r-1)})_p\| \le
A_K C_{N,K}E_\alpha,
\end{aligned}
\]
for some constant \(A_K>0\) depending only on \(K\). Also,
\[
\left|
\left\langle
(\qb^{(m)})_p,
(\qb^{(r-1)})_p
\right\rangle
\right|
=
m+1
\le K .
\]
Consequently, there exist constants \(L^0_K,D_K>0\), depending only on \(K\),
such that every normalized score at step \(r\) differs from its corresponding
limiting normalized score by at most
\[
L^0_K C_{N,K}E_\alpha+D_K\delta_0 .
\]

We now choose the constants. First choose
\[
C_{N,K}\ge K C'_{N,K}.
\]
Next define
\[
\delta_0
:=
\min\left\{
1,\frac{\eta_\gamma}{16D_K}
\right\}.
\]
Finally choose \(\alpha_0>0\) sufficiently large so that, for all
\(\alpha\ge\alpha_0\),
\[
L^0_K C_{N,K}e^{-\alpha\eta_\gamma/8}
\le
\frac{\eta_\gamma}{16}.
\]
With these choices, whenever \(\alpha\ge\alpha_0\) and
\(|\beta/\alpha+2|\le\delta_0\), every actual normalized score differs from
its limiting normalized score by at most
\[
\frac{\eta_\gamma}{16}+\frac{\eta_\gamma}{16}
=
\frac{\eta_\gamma}{8}.
\]

We now compare the limiting scores. The correct token at step \(r\) is
\(\hb_{\pi(r)}\), whose limiting normalized score is
\[
\xb_{\pi(r)}^\top\xb_q .
\]
For any later unselected original token \(\hb_{\pi(j)}\), \(j>r\),
\[
\xb_{\pi(r)}^\top\xb_q-\xb_{\pi(j)}^\top\xb_q
\ge
\eta_\gamma .
\]
For any previously selected original token \(\hb_{\pi(j)}\), \(j<r\), the
limiting normalized score is
\[
\xb_{\pi(j)}^\top\xb_q-2\le -1,
\]
whereas
\[
\xb_{\pi(r)}^\top\xb_q
\ge
\xb_{\pi(N)}^\top\xb_q
\ge
-1+\eta_\gamma .
\]
Thus the correct token beats every previously selected original token by at
least \(\eta_\gamma\). Finally, every previously generated token has limiting
normalized score \(-(2m+1)\le -1\), and hence the correct token also beats
every previously generated token by at least \(\eta_\gamma\). Therefore, before
perturbation, the correct token has normalized score gap at least
\(\eta_\gamma\) over every competitor.

Since each normalized score is perturbed by at most \(\eta_\gamma/8\), the
actual normalized score of \(\hb_{\pi(r)}\) exceeds the actual normalized score
of every other column by at least
\[
\eta_\gamma-2\cdot\frac{\eta_\gamma}{8}
=
\frac{3\eta_\gamma}{4}.
\]
Equivalently, the corresponding unnormalized score gap is at least
\(3\alpha\eta_\gamma/4\). Therefore, if \(\ab^{(r)}\) denotes the softmax
attention vector at step \(r\), then
\[
\sum_{\ell\in [N+r]\setminus\{\pi(r)\}}a_\ell^{(r)}
\le
(N+K)e^{-3\alpha\eta_\gamma/4}.
\]

Using the token norm bound and \(\|\Vb\|\le1\), we get
\[
\begin{aligned}
\left\|
(\Vb\widehat{\Zb}^{(r-1)})\ab^{(r)}
-
\Vb\hb_{\pi(r)}
\right\|
&=
\left\|
\sum_{\ell\in [N+r]\setminus\{\pi(r)\}}
a_\ell^{(r)}
\Vb\left(
\widehat{\Zb}^{(r-1)}_{:,\ell}-\hb_{\pi(r)}
\right)
\right\| \\
&\le
2B_{N,K}
\sum_{\ell\in [N+r]\setminus\{\pi(r)\}}a_\ell^{(r)} \\
&\le
C'_{N,K}e^{-3\alpha\eta_\gamma/4}.
\end{aligned}
\]
By the residual update and the ideal identity
\[
\qb^{(r)}=\qb^{(r-1)}+\Vb\hb_{\pi(r)},
\]
we obtain
\[
\|\widehat{\qb}^{(r)}-\qb^{(r)}\|
\le
\|\widehat{\qb}^{(r-1)}-\qb^{(r-1)}\|
+
C'_{N,K}e^{-3\alpha\eta_\gamma/4}.
\]
Since \(C_{N,K}\ge K C'_{N,K}\), iterating from \(1\) to \(r\) gives
\[
\|\widehat{\qb}^{(r)}-\qb^{(r)}\|
\le
rC'_{N,K}e^{-3\alpha\eta_\gamma/4}
\le
C_{N,K}e^{-\alpha\eta_\gamma/8},
\qquad r=0,\ldots,K,
\]
where the case \(r=0\) is exact. This closes the induction.

Taking \(r=K\) and the \((d+1)\)-st coordinate gives
\[
\left|
\mathrm{TF}(\widehat{\Zb}^{(K-1)},\btheta)_{d+1,-1}
-
\sum_{k=1}^K y_{\pi(k)}
\right|
\le
C_{N,K}e^{-\alpha\eta_\gamma/8}.
\]
Increase \(\alpha_0\) once more, if necessary, so that
\[
C_{N,K}e^{-\alpha_0\eta_\gamma/8}<1.
\]
Then, for all \(\alpha\ge\alpha_0\),
\[
\left|
\mathrm{TF}(\widehat{\Zb}^{(K-1)},\btheta)_{d+1,-1}
-
\sum_{k=1}^K y_{\pi(k)}
\right|
<1.
\]
Since \(y_i\in\{\pm1\}\), the sum
\[
S_K:=\sum_{k=1}^K y_{\pi(k)}
\]
is an integer. Under the assumption \(S_K\neq0\), we have \(|S_K|\ge1\). The
strict error bound above therefore implies that
\[
\operatorname{sign}\!\left(
\mathrm{TF}(\widehat{\Zb}^{(K-1)},\btheta)_{d+1,-1}
\right)
=
\operatorname{sign}(S_K)
=
\operatorname{sign}\Big(\sum_{k=1}^K y_{\pi(k)}\Big).
\]
The constants \(\delta_0\) and \(\alpha_0\) depend only on
\((\gamma,N,K)\), and the argument was deterministic for an arbitrary instance
satisfying the margin condition. Hence the conclusion holds uniformly over all
instances satisfying the stated margin condition.
\end{proof}

\section{Proof of Theorem \ref{thm:training convergence}}\label{appendix: proof of gradient descent}

In this section, we elaborate on the proof of Theorem~\ref{thm:training convergence}. Our proof can be broken down into the following steps:

\begin{enumerate}
\item 
We prove in Sections \ref{subsec:gradient_derivation_reduced_k1} and \ref{subsec:training_dynamics_reduced_k1} that gradient descent preserves a two-dimensional invariant manifold utilizing the symmetry of the data distribution and the gradient structure. More precisely, from zero initialization, for every iteration \(t\),
\[
\Wb^{(t)}=\diag\bigl(\alpha^{(t)}\Ib_d,\;0,\;\beta^{(t)}\Ib_{N+1}\bigr),
\]
where $\alpha^{(t)}$ and $\beta^{(t)}$ are scalar parameters, so the matrix recursion reduces exactly to a scalar recursion for the pair \((\alpha^{(t)},\beta^{(t)})\).
\item 
We rewrite the population loss to a reduced objective \(\mathcal L(\alpha,\beta)\). By analyzing the update dynamics of $\alpha^{(t)}$ and $\beta^{(t)}$, we prove in Section \ref{sec:parameter divergence} that, under the regime $d\ge 2$, $N\ge 100\sqrt d$, $0<\eta\le {1}/{840(N+1)^2}$ the region \([0,\infty)\times(-\infty,0]\) is forward invariant, and the iterates satisfy
\[
\beta^{(t)}\to -\infty,\qquad \alpha^{(t)}\to +\infty.
\]
\item 
We further provide an analysis of the precise slope
$
-{\beta^{(t)}}/{\alpha^{(t)}}
$ in Sections \ref{sec:exact asymptotic representation}, \ref{sec:boundary-layer estimates for certain subsequence} and \ref{sec:convergence of the population loss}.
 We demonstrate that \[
\liminf_{t\to\infty}\left(-\frac{\beta^{(t)}}{\alpha^{(t)}}\right)\ge 2
\qquad\text{and}\qquad
\limsup_{t\to\infty}\left(-\frac{\beta^{(t)}}{\alpha^{(t)}}\right)\le 2,
\]
and therefore
$-{\beta^{(t)}}/{\alpha^{(t)}}\to 2$.
With the asymptotic slope identified, we then show that the self-token contribution vanishes and establish the convergence of the loss function.
\end{enumerate}
\subsection{Derivation of Gradients}
\label{subsec:gradient_derivation_reduced_k1}
In this subsection, we provide a derivation of the gradients of the population loss $\mathcal{L}(\Wb)$ with respect to the parameter matrix $\Wb$, based on the architecture defined in Section \ref{subsec:transformer_arch}.

 Under Assumption~\ref{ass:distribution}, the random variables \(\{\xb_i^\top \xb_q\}_{i=1}^N\) have a continuous joint law, hence ties occur with probability zero. Therefore \(\pi(1)\) is almost surely uniquely defined. For notational simplicity, we write
\[
\pi\coloneqq\pi(1),\quad \qb\coloneqq\qb^{(0)},\quad \Zb\coloneqq\Zb^{(0)}=[\Hb,\qb].
\]

Recall that for a single sample \((\mathcal D,\xb_q,\omega)\), the training loss is
\begin{equation}
\label{eq:reduced_k1_sample_loss}
\ell(\Wb)
=
\frac12\left\|
\mathrm{TF}(\Zb;\Wb)_{:,-1}-\qb^{(1)}
\right\|^2,
\end{equation}
where $\qb^{(1)}\coloneqq[\xb_q^\top, y_{\pi(1)},\pb_{N+1}^\top+\pb_{\pi(1)}^\top]^\top$, and the population objective is
\begin{equation}
\label{eq:reduced_k1_population_loss}
\mathcal L(\Wb)
=
\EE_{\mathcal D,\xb_q}\EE_{\omega}\big[\ell(\Wb)\big].
\end{equation}

\noindent\textbf{Intermediate variables.}
Define
\begin{equation}
\label{eq:reduced_k1_intermediate_variables}
\ub\coloneqq \Zb^\top \Wb \qb\in\RR^{N+1},
\qquad
\ab\coloneqq \mathrm{softmax}(\ub)\in\RR^{N+1},
\qquad
\hat{\qb}\coloneqq \mathrm{TF}(\Zb;\Wb)_{:,-1}
=
\qb+\Vb\Zb\,\ab,
\end{equation}
and the residual
\begin{equation}
\label{eq:reduced_k1_residual}
\rb\coloneqq \hat{\qb}-\qb^{(1)}.
\end{equation}
 Then \(\ell(\Wb)=\frac12\|\rb\|^2\).

\begin{lemma}
\label{lem:reduced_k1_gradient}
For the single-sample loss \(\ell(\Wb)\) defined in \eqref{eq:reduced_k1_sample_loss}, we have
\begin{equation}
\label{eq:reduced_k1_gradient_formula}
\nabla_{\Wb}\ell(\Wb)
=
\Zb
\Bigl(
(\diag(\ab)-\ab\ab^\top)\,\Zb^\top \Vb^\top \rb
\Bigr)
\qb^\top.
\end{equation}
Consequently,
\begin{equation}
\label{eq:reduced_k1_population_gradient_formula}
\nabla_{\Wb}\mathcal L(\Wb)
=
\EE_{\mathcal D,\xb_q}\EE_{\omega}
\left[
\Zb
\Bigl(
(\diag(\ab)-\ab\ab^\top)\,\Zb^\top \Vb^\top \rb
\Bigr)
\qb^\top
\right].
\end{equation}
\end{lemma}

\begin{proof}
We apply the chain rule in three steps.

\noindent\textbf{Step 1: gradient with respect to the attention weights.}
From \eqref{eq:reduced_k1_residual},
\[
\rb
=
\qb+\Vb\Zb\,\ab-\qb^{(1)}.
\]
Hence
\[
\ell(\Wb)=\frac12\|\rb\|^2,
\qquad
\nabla_{\ab}\ell
=
(\Vb\Zb)^\top \rb
=
\Zb^\top \Vb^\top \rb.
\]

\noindent\textbf{Step 2: gradient with respect to the logits.}
Let \(\Jb_{\mathrm{softmax}}(\ub)\) denote the Jacobian of \(\mathrm{softmax}\) at \(\ub\). The standard softmax identity gives
\[
\Jb_{\mathrm{softmax}}(\ub)=\diag(\ab)-\ab\ab^\top.
\]
Therefore,
\[
\nabla_{\ub}\ell
=
\Jb_{\mathrm{softmax}}(\ub)\,\nabla_{\ab}\ell
=
(\diag(\ab)-\ab\ab^\top)\,\Zb^\top \Vb^\top \rb.
\]

\noindent\textbf{Step 3: gradient with respect to \(\Wb\).}
Since
\[
\ub=\Zb^\top \Wb \qb,
\]
its differential is
\[
d\ub=\Zb^\top (d\Wb)\,\qb.
\]
Thus, for any perturbation \(d\Wb\),
\begin{align*}
d\ell
&=
\langle \nabla_{\ub}\ell,\ d\ub\rangle \\
&=
\langle \nabla_{\ub}\ell,\ \Zb^\top (d\Wb)\,\qb\rangle \\
&=
\left\langle
\Zb(\nabla_{\ub}\ell)\qb^\top,\ d\Wb
\right\rangle.
\end{align*}
Hence
\[
\nabla_{\Wb}\ell(\Wb)
=
\Zb(\nabla_{\ub}\ell)\qb^\top
=
\Zb
\Bigl(
(\diag(\ab)-\ab\ab^\top)\,\Zb^\top \Vb^\top \rb
\Bigr)
\qb^\top,
\]
which proves \eqref{eq:reduced_k1_gradient_formula}. Taking expectation and using linearity of expectation gives \eqref{eq:reduced_k1_population_gradient_formula}.
\end{proof}

\begin{lemma}
\label{lem:reduced_k1_zero_y_column_gradient}
Let \(y\) denote the label coordinate in the \((x,y,p)\)-block decomposition of \(\RR^D\). Then, for every \(\Wb\in\RR^{D\times D}\),
\begin{equation}
\label{eq:reduced_k1_zero_y_column_sample_gradient}
(\nabla_{\Wb}\ell(\Wb))_{:,y}=\mathbf 0,
\end{equation}
and
\begin{equation}
\label{eq:reduced_k1_zero_y_column_population_gradient}
(\nabla_{\Wb}\mathcal L(\Wb))_{:,y}=\mathbf 0.
\end{equation}
Consequently, since gradient descent is initialized at \(\Wb_0=\mathbf 0\) and evolves according to
\begin{equation}
\label{eq:reduced_k1_gd}
\Wb^{(t+1)}
=
\Wb^{(t)}-\eta \nabla_{\Wb}\mathcal L(\Wb^{(t)}),
\qquad t\ge 0,
\end{equation}
we have
\begin{equation}
\label{eq:reduced_k1_zero_y_column_iterates}
(\Wb^{(t)})_{:,y}=\mathbf 0
\qquad\text{for all }t\ge 0.
\end{equation}
\end{lemma}

\begin{proof}
From the definition of the initial query token,
\[
\qb=
\begin{bmatrix}
\xb_q\\
0\\
\pb_{N+1}
\end{bmatrix},
\]
so its \(y\)-coordinate is zero:
\[
\qb_y=0.
\]
By \eqref{eq:reduced_k1_gradient_formula},
\[
\nabla_{\Wb}\ell(\Wb)
=
\Zb
\Bigl(
(\diag(\ab)-\ab\ab^\top)\,\Zb^\top \Vb^\top \rb
\Bigr)
\qb^\top.
\]
The \(y\)-th column of the right-hand side is obtained by multiplying by the scalar \(\qb_y\), hence
\[
(\nabla_{\Wb}\ell(\Wb))_{:,y}=\mathbf 0.
\]
Taking expectation gives
\[
(\nabla_{\Wb}\mathcal L(\Wb))_{:,y}
=
\EE_{\mathcal D,\xb_q}\EE_\omega
\bigl[(\nabla_{\Wb}\ell(\Wb))_{:,y}\bigr]
=
\mathbf 0.
\]
Finally, \((\Wb^{(0)})_{:,y}=\mathbf 0\) because \(\Wb^{(0)}=\mathbf 0\). Then by \eqref{eq:reduced_k1_gd} and \eqref{eq:reduced_k1_zero_y_column_population_gradient},
\[
(\Wb^{(t+1)})_{:,y}
=
(\Wb^{(t)})_{:,y}-\eta (\nabla_{\Wb}\mathcal L(\Wb^{(t)}))_{:,y}
=
\mathbf 0-\eta \mathbf 0
=
\mathbf 0.
\]
This proves \eqref{eq:reduced_k1_zero_y_column_iterates} by induction.
\end{proof}

\subsection{A low-dimensional block-diagonal structure preserved by GD}
\label{subsec:training_dynamics_reduced_k1}

In this subsection, we show that gradient descent preserves a low-dimensional block-diagonal structure of \(\Wb\).

%
%

\paragraph{Symmetry group.}
We use the same \((x,y,p)\)-block decomposition
\[
\RR^D=\RR^d\oplus\RR\oplus\RR^{N+1}.
\]
Define the following orthogonal transformations:
\begin{enumerate}[leftmargin=*]
\item \textbf{Data rotations:} for any \(R\in O(d)\),
\[
G_R\coloneqq \operatorname{diag}(R,1,\Ib_{N+1})\in\RR^{D\times D}.
\]
\item \textbf{Label flips:} for any \(s\in\{\pm1\}\),
\[
G_s\coloneqq \operatorname{diag}(\Ib_d,s,\Ib_{N+1})\in\RR^{D\times D}.
\]
\item \textbf{Signed permutations on the position block:}
Let $\mathcal{H}_{N+1}$ be the set of all \emph{signed permutation matrices} in $\RR^{(N+1)\times(N+1)}$, i.e.
matrices of the form $S = DP$ where $P$ is a permutation matrix and $D$ is diagonal with $\pm1$ entries. For any \(S\in\mathcal H_{N+1}\), define
\[
G_S\coloneqq \operatorname{diag}(\Ib_d,1,S)\in\RR^{D\times D}.
\]
\end{enumerate}
Let \(\mathbb G\) be the group generated by all such \(G_R\), \(G_s\), and \(G_S\). The group acts on \(\Wb\) by conjugation:
\[
\Wb\mapsto G\Wb G^{-1},
\qquad G\in\mathbb G.
\]

\begin{lemma}
\label{lem:reduced_k1_tf_covariance}
For any orthogonal \(G\in\mathbb G\), any \(\Zb\in\RR^{D\times (N+1)}\), and any \(\Wb\in\RR^{D\times D}\),
\begin{equation}
\label{eq:reduced_k1_tf_covariance}
\mathrm{TF}(G\Zb;G\Wb G^{-1})=G\,\mathrm{TF}(\Zb;\Wb).
\end{equation}
In particular,
\begin{equation}
\label{eq:reduced_k1_last_token_covariance}
\mathrm{TF}(G\Zb;G\Wb G^{-1})_{:,-1}
=
G\,\mathrm{TF}(\Zb;\Wb)_{:,-1}.
\end{equation}
\end{lemma}

\begin{proof}
We first note that the fixed value matrix
\[
\Vb=\operatorname{diag}(0,1,\Ib_{N+1})
\]
commutes with every generator of \(\mathbb G\), and hence with every \(G\in\mathbb G\). Indeed:
\[
G_R\Vb=\Vb G_R,\qquad
G_s\Vb=\Vb G_s,\qquad
G_S\Vb=\Vb G_S.
\]
Now compute:
\[
\Vb(G\Zb)=G(\Vb\Zb),
\]
and
\[
(G\Zb)^\top (G\Wb G^{-1})(G\Zb)
=
\Zb^\top G^\top G\,\Wb\,G^{-1}G\,\Zb
=
\Zb^\top \Wb \Zb,
\]
because \(G\) is orthogonal. Therefore the arguments of the column-wise softmax coincide:
\[
\mathrm{softmax}\big((G\Zb)^\top (G\Wb G^{-1})(G\Zb)\big)
=
\mathrm{softmax}(\Zb^\top \Wb \Zb).
\]
Substituting into the definition of \(\mathrm{TF}\),
\begin{align*}
\mathrm{TF}(G\Zb;G\Wb G^{-1})
&=
G\Zb+(\Vb G\Zb)\mathrm{softmax}\big((G\Zb)^\top (G\Wb G^{-1})(G\Zb)\big)\\
&=
G\Zb+\bigl(G(\Vb\Zb)\bigr)\mathrm{softmax}(\Zb^\top \Wb \Zb)\\
&=
G\Bigl(\Zb+(\Vb\Zb)\mathrm{softmax}(\Zb^\top \Wb \Zb)\Bigr)\\
&=
G\,\mathrm{TF}(\Zb;\Wb).
\end{align*}
This proves \eqref{eq:reduced_k1_tf_covariance}, and \eqref{eq:reduced_k1_last_token_covariance} follows by taking the last column.
\end{proof}

\begin{lemma}
\label{lem:reduced_k1_rotation_invariance}
For every \(R\in O(d)\),
\begin{equation}
\label{eq:reduced_k1_rotation_invariance}
\mathcal L(\Wb)=\mathcal L(G_R\Wb G_R^{-1}).
\end{equation}
\end{lemma}

\begin{proof}
Define rotated data
\[
\xb_i'\coloneqq R\xb_i,\qquad
\xb_q'\coloneqq R\xb_q,\qquad
y_i'\coloneqq y_i,
\]
and keep \(\omega\) unchanged. Since the uniform measure on \(\mathbb S^{d-1}\) is rotationally invariant,
\[
(\mathcal D',\xb_q',\omega)\overset{d}{=}(\mathcal D,\xb_q,\omega).
\]
The nearest-neighbor index is unchanged under rotations, because Euclidean distances are preserved. Moreover,
\[
\hb_i'=
\begin{bmatrix}
R\xb_i\\
y_i\\
\pb_i
\end{bmatrix}
=
G_R\hb_i,
\qquad
\qb'=
\begin{bmatrix}
R\xb_q\\
0\\
\pb_{N+1}
\end{bmatrix}
=
G_R\qb,
\]
and therefore
\[
\Zb'=G_R\Zb.
\]
Similarly,
\[
(\qb^{(1)})'
=
\begin{bmatrix}
R\xb_q\\
y_\pi\\
\pb_{N+1}+\pb_\pi
\end{bmatrix}
=
G_R\qb^{(1)}.
\]
Hence, by Lemma~\ref{lem:reduced_k1_tf_covariance},
\begin{align*}
&\ell(G_R\Wb G_R^{-1};\mathcal D',\xb_q',\omega)\\
&=
\frac12
\left\|
\mathrm{TF}(\Zb';G_R\Wb G_R^{-1})_{:,-1}-(\qb^{(1)})'
\right\|^2\\
&=
\frac12
\left\|
G_R\mathrm{TF}(\Zb;\Wb)_{:,-1}-G_R\qb^{(1)}
\right\|^2\\
&=
\frac12
\left\|
\mathrm{TF}(\Zb;\Wb)_{:,-1}-\qb^{(1)}
\right\|^2\\
&=
\ell(\Wb;\mathcal D,\xb_q,\omega),
\end{align*}
where we used the orthogonality of \(G_R\) in the third equality. Taking expectation over \((\mathcal D,\xb_q,\omega)\) gives \eqref{eq:reduced_k1_rotation_invariance}.
\end{proof}

\begin{lemma}
\label{lem:reduced_k1_label_invariance}
For every \(s\in\{\pm1\}\),
\begin{equation}
\label{eq:reduced_k1_label_invariance}
\mathcal L(\Wb)=\mathcal L(G_s\Wb G_s^{-1}).
\end{equation}
\end{lemma}

\begin{proof}
Define transformed labels
\[
y_i'\coloneqq s\,y_i,
\]
and keep \(\{\xb_i\}_{i=1}^N\), \(\xb_q\), and \(\omega\) unchanged. By Assumption \ref{ass:distribution}, we have 
\[
(\mathcal D',\xb_q,\omega)\overset{d}{=}(\mathcal D,\xb_q,\omega).
\]
The nearest-neighbor index \(\pi\) depends only on \(\{\xb_i\}_{i=1}^N\) and \(\xb_q\), so it is unchanged. Moreover,
\[
\hb_i'=
\begin{bmatrix}
\xb_i\\
s\,y_i\\
\pb_i
\end{bmatrix}
=
G_s\hb_i,
\qquad
\qb'=
\begin{bmatrix}
\xb_q\\
0\\
\pb_{N+1}
\end{bmatrix}
=
G_s\qb,
\]
hence
\[
\Zb'=G_s\Zb.
\]
Also,
\[
(\qb^{(1)})'
=
\begin{bmatrix}
\xb_q\\
s\,y_\pi\\
\pb_{N+1}+\pb_\pi
\end{bmatrix}
=
G_s\qb^{(1)}.
\]
Applying Lemma~\ref{lem:reduced_k1_tf_covariance} and using orthogonality of \(G_s\), we obtain
\[
\ell(G_s\Wb G_s^{-1};\mathcal D',\xb_q,\omega)
=
\ell(\Wb;\mathcal D,\xb_q,\omega).
\]
Taking expectation gives \eqref{eq:reduced_k1_label_invariance}.
\end{proof}

\begin{lemma}
\label{lem:reduced_k1_pe_invariance}
Let \(S\in\mathcal H_{N+1}\) be any signed permutation matrix. Define transformed position vectors
\[
(\pb_1',\ldots,\pb_N',\pb_{N+1}')
\coloneqq 
(S\pb_1,\ldots,S\pb_N,S\pb_{N+1}).
\]
Then
\[
(\pb_1',\ldots,\pb_N',\pb_{N+1}')
\overset{d}{=}
(\pb_1,\ldots,\pb_N,\pb_{N+1}).
\]
Equivalently, there exists \(\omega'=(\tau',s_1',\ldots,s_{N+1}')\) with \(\omega'\overset{d}{=}\omega\) such that
\[
\pb_i'=s_i'e_{\tau'(i)}\quad (i\in[N]),
\qquad
\pb_{N+1}'=s_{N+1}'e_{\tau'(N+1)}.
\]
\end{lemma}

\begin{proof}
Write
\[
S=DP,
\]
where \(P\) is a permutation matrix and \(D=\diag(\epsilon_1,\ldots,\epsilon_{N+1})\) with \(\epsilon_j\in\{\pm1\}\). Since
\[
\pb_i=s_i e_{\tau(i)},
\]
we have
\[
S\pb_i=(DP)(s_i e_{\tau(i)})=s_i D e_{P(\tau(i))}.
\]
Define
\[
\tau'\coloneqq P\circ \tau,
\qquad
s_i'\coloneqq s_i\,\epsilon_{\tau'(i)}
\quad (i=1,\ldots,N+1).
\]
Then
\[
S\pb_i=s_i' e_{\tau'(i)}\quad (i=1,\ldots,N),
\qquad
S\pb_{N+1}=s_{N+1}' e_{\tau'(N+1)}.
\]
Because \(\tau\) is uniform on \(S_{N+1}\) and \(P\) is fixed, \(\tau'=P\circ\tau\) is also uniform on \(S_{N+1}\).

It remains to verify that the transformed signs \((s_1',\ldots,s_{N+1}')\) are i.i.d.\ Rademacher and independent of \(\tau'\). Fix any \(\sigma\in S_{N+1}\). On the event \(\{\tau'=\sigma\}\), the quantities \(\epsilon_{\tau'(i)}=\epsilon_{\sigma(i)}\) are deterministic signs, and therefore
\[
s_i' = s_i\,\epsilon_{\sigma(i)}
\qquad (i=1,\ldots,N+1).
\]
Since each \(s_i\) is Rademacher, multiplication by a deterministic sign preserves the Rademacher law. Moreover, because this transformation acts coordinate-wise, the collection \((s_1',\ldots,s_{N+1}')\) remains independent under the conditioning \(\tau'=\sigma\). Hence
\[
(s_1',\ldots,s_{N+1}')\mid(\tau'=\sigma)
\overset{d}{=}
(s_1,\ldots,s_{N+1}),
\]
and this conditional law does not depend on \(\sigma\). Therefore \((s_1',\ldots,s_{N+1}')\) is a vector of i.i.d.\ Rademacher variables independent of \(\tau'\).

Consequently, \(\omega'=(\tau',s_1',\ldots,s_{N+1}')\) has the same distribution as \(\omega\), which proves the claim.
\end{proof}

\begin{lemma}
\label{lem:reduced_k1_signed_perm_invariance}
For every \(S\in\mathcal H_{N+1}\),
\begin{equation}
\label{eq:reduced_k1_signed_perm_invariance}
\mathcal L(\Wb)=\mathcal L(G_S\Wb G_S^{-1}).
\end{equation}
\end{lemma}

\begin{proof}
Fix \((\mathcal D,\xb_q)\), and let \(\omega\) generate the position vectors \((\pb_1,\ldots,\pb_N,\pb_{N+1})\). Let \(\omega'\) be as in Lemma~\ref{lem:reduced_k1_pe_invariance}, so that
\[
(\pb_1',\ldots,\pb_N',\pb_{N+1}')=(S\pb_1,\ldots,S\pb_N,S\pb_{N+1})
\]
and \(\omega'\overset{d}{=}\omega\).

Because only the \(p\)-block changes, we have
\[
\hb_i'=
\begin{bmatrix}
\xb_i\\
y_i\\
\pb_i'
\end{bmatrix}
=
\begin{bmatrix}
\xb_i\\
y_i\\
S\pb_i
\end{bmatrix}
=
G_S\hb_i,
\qquad
\qb'=
\begin{bmatrix}
\xb_q\\
0\\
\pb_{N+1}'
\end{bmatrix}
=
G_S\qb,
\]
hence
\[
\Zb'=G_S\Zb.
\]
The nearest-neighbor index \(\pi\) depends only on the data vectors and the query, so it is unchanged. Therefore
\[
(\qb^{(1)})'
=
\begin{bmatrix}
\xb_q\\
y_\pi\\
\pb_{N+1}'+\pb_\pi'
\end{bmatrix}
=
\begin{bmatrix}
\xb_q\\
y_\pi\\
S\pb_{N+1}+S\pb_\pi
\end{bmatrix}
=
G_S\qb^{(1)}.
\]
Applying Lemma~\ref{lem:reduced_k1_tf_covariance} and orthogonality of \(G_S\), we obtain
\[
\ell(G_S\Wb G_S^{-1};\mathcal D,\xb_q,\omega')
=
\ell(\Wb;\mathcal D,\xb_q,\omega).
\]
Finally, because \(\omega'\overset{d}{=}\omega\),
\[
\EE_\omega\big[\ell(G_S\Wb G_S^{-1};\mathcal D,\xb_q,\omega)\big]
=
\EE_{\omega'}\big[\ell(G_S\Wb G_S^{-1};\mathcal D,\xb_q,\omega')\big]
=
\EE_\omega\big[\ell(\Wb;\mathcal D,\xb_q,\omega)\big].
\]
Taking expectation over \((\mathcal D,\xb_q)\) proves \eqref{eq:reduced_k1_signed_perm_invariance}.
\end{proof}

\begin{lemma}
\label{lem:reduced_k1_gradient_equivariance}
Let \(G\in\mathbb G\). If
\[
\mathcal L(\Wb)=\mathcal L(G\Wb G^{-1})
\qquad\text{for all }\Wb,
\]
then
\begin{equation}
\label{eq:reduced_k1_gradient_equivariance}
\nabla_{\Wb}\mathcal L(G\Wb G^{-1})
=
G\bigl(\nabla_{\Wb}\mathcal L(\Wb)\bigr)G^{-1}.
\end{equation}
\end{lemma}

\begin{proof}
Let \(\Wb'\coloneqq G\Wb G^{-1}\), and let \(\Db\in\RR^{D\times D}\) be arbitrary. By invariance,
\[
\mathcal L(\Wb+\epsilon\Db)
=
\mathcal L\bigl(G(\Wb+\epsilon\Db)G^{-1}\bigr)
=
\mathcal L(\Wb'+\epsilon\,G\Db G^{-1})
\]
for every sufficiently small \(\epsilon\). Differentiating at \(\epsilon=0\) gives
\[
\langle \nabla_{\Wb}\mathcal L(\Wb),\Db\rangle_F
=
\left\langle
\nabla_{\Wb}\mathcal L(\Wb'),\ G\Db G^{-1}
\right\rangle_F.
\]
Using trace cyclicity,
\begin{align*}
\left\langle
\nabla_{\Wb}\mathcal L(\Wb'),\ G\Db G^{-1}
\right\rangle_F
&=
\tr\Bigl((\nabla_{\Wb}\mathcal L(\Wb'))^\top G\Db G^{-1}\Bigr)\\
&=
\tr\Bigl((G^{-1}\nabla_{\Wb}\mathcal L(\Wb')G)^\top \Db\Bigr)\\
&=
\left\langle
G^{-1}\nabla_{\Wb}\mathcal L(\Wb')G,\ \Db
\right\rangle_F.
\end{align*}
Since \(\Db\) is arbitrary,
\[
\nabla_{\Wb}\mathcal L(\Wb)
=
G^{-1}\nabla_{\Wb}\mathcal L(\Wb')G,
\]
which is equivalent to \eqref{eq:reduced_k1_gradient_equivariance}.
\end{proof}

\paragraph{Fixed-point set.}
Define
\[
\mathrm{Fix}(\mathbb G)
\coloneqq 
\{\Wb\in\RR^{D\times D}: \ G\Wb G^{-1}=\Wb,\ \forall G\in\mathbb G\}.
\]

\begin{lemma}
\label{lem:reduced_k1_gd_preserves_fix}
Let \((\Wb^{(t)})_{t\ge 0}\) be the gradient-descent iterates
\[
\Wb^{(t+1)}=\Wb^{(t)}-\eta \nabla_{\Wb}\mathcal L(\Wb^{(t)}),
\qquad
\Wb^{(0)}=\mathbf 0.
\]
Then
\[
\Wb^{(t)}\in\mathrm{Fix}(\mathbb G)
\qquad\text{for all }t\ge 0.
\]
\end{lemma}

\begin{proof}
By Lemmas~\ref{lem:reduced_k1_rotation_invariance}, \ref{lem:reduced_k1_label_invariance}, and \ref{lem:reduced_k1_signed_perm_invariance}, the population objective is invariant under conjugation by every element of \(\mathbb G\). Since \(\Wb^{(0)}=\mathbf 0\), clearly \(\Wb^{(0)}\in\mathrm{Fix}(\mathbb G)\).

Assume inductively that \(\Wb^{(t)}\in\mathrm{Fix}(\mathbb G)\). Fix any \(G\in\mathbb G\). Then
\[
G\Wb^{(t)} G^{-1}=\Wb^{(t)}.
\]
Applying Lemma~\ref{lem:reduced_k1_gradient_equivariance},
\[
\nabla_{\Wb}\mathcal L(\Wb^{(t)})
=
\nabla_{\Wb}\mathcal L(G\Wb^{(t)} G^{-1})
=
G\bigl(\nabla_{\Wb}\mathcal L(\Wb^{(t)})\bigr)G^{-1}.
\]
Therefore the gradient is also fixed by \(G\). Consequently,
\[
G\Wb^{(t+1)}G^{-1}
=
G\Bigl(\Wb^{(t)}-\eta\nabla_{\Wb}\mathcal L(\Wb^{(t)})\Bigr)G^{-1}
=
\Wb^{(t)}-\eta\nabla_{\Wb}\mathcal L(\Wb^{(t)})
=
\Wb^{(t+1)}.
\]
Since \(G\in\mathbb G\) was arbitrary, \(\Wb^{(t+1)}\in\mathrm{Fix}(\mathbb G)\). The claim follows by induction.
\end{proof}

\begin{lemma}
\label{lem:reduced_k1_fix_characterization}
If \(\Wb\in\mathrm{Fix}(\mathbb G)\), then there exist scalars \(\alpha,\gamma,\beta\in\RR\) such that
\begin{equation}
\label{eq:reduced_k1_fix_characterization}
\Wb=\operatorname{diag}(\alpha \Ib_d,\gamma,\beta \Ib_{N+1}).
\end{equation}
\end{lemma}

\begin{proof}
Write \(\Wb\) in \((x,y,p)\)-block form:
\[
\Wb=
\begin{bmatrix}
\Wb_{xx} & \Wb_{xy} & \Wb_{xp}\\
\Wb_{yx} & \Wb_{yy} & \Wb_{yp}\\
\Wb_{px} & \Wb_{py} & \Wb_{pp}
\end{bmatrix}.
\]

\noindent\textbf{Step 1: constraints from data rotations.}
Since \(G_R\Wb G_R^{-1}=\Wb\) for all \(R\in O(d)\),
\[
G_R\Wb G_R^{-1}
=
\begin{bmatrix}
R\Wb_{xx}R^\top & R\Wb_{xy} & R\Wb_{xp}\\
\Wb_{yx}R^\top & \Wb_{yy} & \Wb_{yp}\\
\Wb_{px}R^\top & \Wb_{py} & \Wb_{pp}
\end{bmatrix}
=
\Wb.
\]
Thus,
\[
R\Wb_{xx}R^\top=\Wb_{xx},
\qquad
R\Wb_{xy}=\Wb_{xy},
\qquad
R\Wb_{xp}=\Wb_{xp},
\]
\[
\Wb_{yx}R^\top=\Wb_{yx},
\qquad
\Wb_{px}R^\top=\Wb_{px},
\qquad \forall R\in O(d).
\]
We now solve these constraints.

First, \(R\Wb_{xy}=\Wb_{xy}\) for every \(R\in O(d)\) implies \(\Wb_{xy}=\mathbf 0\). Indeed, fix any unit vector \(u\in\RR^d\), and let
\[
R_u\coloneqq \Ib_d-2uu^\top
\]
be the reflection across \(u^\perp\). Then \(R_u u=-u\). Applying \(R_u\Wb_{xy}=\Wb_{xy}\) and left-multiplying by \(u^\top\), we obtain
\[
-u^\top \Wb_{xy}=u^\top \Wb_{xy},
\]
hence \(u^\top \Wb_{xy}=0\). Since \(u\) is arbitrary, \(\Wb_{xy}=0\). The same reflection argument yields
\[
\Wb_{xp}=\mathbf 0,
\qquad
\Wb_{yx}=\mathbf 0,
\qquad
\Wb_{px}=\mathbf 0.
\]

It remains to determine \(\Wb_{xx}\). The relation \(R\Wb_{xx}R^\top=\Wb_{xx}\) for all \(R\in O(d)\) implies that \(\Wb_{xx}\) commutes with all orthogonal transformations. In particular:
\begin{itemize}[leftmargin=*]
\item if \(R\) swaps the basis vectors \(e_i\) and \(e_j\), then the diagonal entries \((\Wb_{xx})_{ii}\) and \((\Wb_{xx})_{jj}\) must be equal;
\item if \(R\) flips the sign of \(e_i\) and fixes the other basis vectors, then all off-diagonal entries in row/column \(i\) must vanish.
\end{itemize}
Hence \(\Wb_{xx}=\alpha \Ib_d\) for some \(\alpha\in\RR\).

\noindent\textbf{Step 2: constraints from label flips.}
Since \(G_{-1}\Wb G_{-1}^{-1}=\Wb\),
\[
G_{-1}\Wb G_{-1}^{-1}
=
\begin{bmatrix}
\Wb_{xx} & -\Wb_{xy} & \Wb_{xp}\\
-\Wb_{yx} & \Wb_{yy} & -\Wb_{yp}\\
\Wb_{px} & -\Wb_{py} & \Wb_{pp}
\end{bmatrix}
=
\Wb.
\]
Therefore
\[
\Wb_{xy}=\Wb_{yx}=\Wb_{yp}=\Wb_{py}=\mathbf 0.
\]
Combined with Step 1, this shows that all cross-couplings between \(x\), \(y\), and \(p\) vanish, except that \(\Wb_{yy}\) and \(\Wb_{pp}\) remain unconstrained at this stage.

\noindent\textbf{Step 3: constraints from signed permutations on the position block.}
For every \(S\in\mathcal H_{N+1}\), the relation \(G_S\Wb G_S^{-1}=\Wb\) implies
\[
S\Wb_{pp}S^{-1}=\Wb_{pp}.
\]
We claim that this forces \(\Wb_{pp}\) to be a scalar multiple of the identity.

First, take \(S=D^{(j)}\), the diagonal sign-flip matrix that changes only the sign of the \(j\)-th coordinate. Then
\[
D^{(j)}\Wb_{pp}D^{(j)}=\Wb_{pp}
\]
implies that every off-diagonal entry in row \(j\) and column \(j\) must vanish. Since \(j\) is arbitrary, all off-diagonal entries of \(\Wb_{pp}\) vanish, so \(\Wb_{pp}\) is diagonal.

Next, take \(S=P_{ij}\), the permutation matrix that swaps coordinates \(i\) and \(j\). Then
\[
P_{ij}\Wb_{pp}P_{ij}^{-1}=\Wb_{pp}
\]
implies that the \(i\)-th and \(j\)-th diagonal entries of \(\Wb_{pp}\) are equal. Since \(i,j\) are arbitrary, all diagonal entries coincide. Hence
\[
\Wb_{pp}=\beta \Ib_{N+1}
\]
for some \(\beta\in\RR\).

Putting together Steps 1--3, we conclude that
\[
\Wb=\operatorname{diag}(\alpha \Ib_d,\gamma,\beta \Ib_{N+1})
\]
for some scalars \(\alpha,\gamma,\beta\in\RR\).
\end{proof}

\begin{theorem}
\label{thm:reduced_k1_training_structure}
Let \((\Wb^{(t)})_{t\ge 0}\) be the gradient-descent iterates \eqref{eq:reduced_k1_gd} initialized at \(\Wb^{(0)}=\mathbf 0\). Then for every \(t\ge 0\), there exist scalars \(\alpha_t,\beta_t\in\RR\) such that
\begin{equation}
\label{eq:reduced_k1_training_structure}
\Wb^{(t)}=\operatorname{diag}(\alpha^{(t)} \Ib_d,0,\beta^{(t)} \Ib_{N+1}).
\end{equation}
Equivalently, the gradient-descent trajectory remains in the two-dimensional linear subspace
\[
\mathcal M
=
\left\{
\operatorname{diag}(\alpha \Ib_d,0,\beta \Ib_{N+1})
:\ \alpha,\beta\in\RR
\right\}.
\]
\end{theorem}

\begin{proof}
By Lemma~\ref{lem:reduced_k1_gd_preserves_fix}, each iterate \(\Wb^{(t)}\) belongs to \(\mathrm{Fix}(\mathbb G)\). Therefore, by Lemma~\ref{lem:reduced_k1_fix_characterization}, there exist scalars \(\alpha^{(t)},\gamma^{(t)},\beta^{(t)}\) such that
\[
\Wb^{(t)}=\operatorname{diag}(\alpha^{(t)} \Ib_d,\gamma^{(t)},\beta^{(t)} \Ib_{N+1}).
\]
On the other hand, Lemma~\ref{lem:reduced_k1_zero_y_column_gradient} shows that the \(y\)-column of every iterate vanishes:
\[
(\Wb^{(t)})_{:,y}=\mathbf 0
\qquad\text{for all }t\ge 0.
\]
For a matrix of the form \(\operatorname{diag}(\alpha^{(t)} \Ib_d,\gamma^{(t)},\beta^{(t)} \Ib_{N+1})\), the only possibly nonzero entry in the \(y\)-th column is the diagonal entry \(\gamma^{(t)}\). Hence \((\Wb^{(t)})_{:,y}=\mathbf 0\) implies
\[
\gamma^{(t)}=0.
\]
Therefore
\[
\Wb^{(t)}=\operatorname{diag}(\alpha^{(t)} \Ib_d,0,\beta^{(t)} \Ib_{N+1})
\qquad\text{for all }t\ge 0,
\]
which proves the claim.
\end{proof}

\subsection{Parameter Divergence}
\label{sec:parameter divergence}

By Theorem~\ref{thm:reduced_k1_training_structure}, the population gradient-descent iterates remain in the two-dimensional manifold
\[
\mathcal M
=
\left\{
\diag(\alpha \Ib_d,0,\beta \Ib_{N+1})
:\ \alpha,\beta\in\RR
\right\}.
\]
Hence the matrix recursion for \(\Wb^{(t)}\) reduces exactly to a scalar recursion for the pair \((\alpha^{(t)},\beta^{(t)})\). In this subsection, we analyze this reduced two-parameter dynamical system and prove that
\[
\alpha^{(t)}\to+\infty,
\qquad
\beta^{(t)}\to-\infty.
\]
The analysis proceeds in three main steps:

\noindent\textbf{Step 1:}
We first rewrite the population loss as an exact expectation over the inner products
\[
U_i\coloneqq \xb_i^\top \xb_q,\qquad i\in[N],
\]
thereby integrating out the label randomness and the stochastic positional-encoding randomness. This yields a reduced objective \(\mathcal L(\alpha,\beta)\).

\noindent\textbf{Step 2:}
We show that the region
\[
[0,\infty)\times(-\infty,0]
\]
is forward invariant under gradient descent, and that throughout this region the \(\beta\)-gradient is strictly positive.

\noindent\textbf{Step 3:}
We then prove by contradiction that \(\beta^{(t)}\to -\infty\). Once \(\beta^{(t)}\to-\infty\) is known, we prove that \(\alpha^{(t)}\) cannot remain in any compact interval, and hence \(\alpha^{(t)}\to +\infty\).

Recall the notation
\[
\pi\coloneqq \pi(1),\qquad \Zb\coloneqq \Zb^{(0)},\qquad \qb\coloneqq \qb^{(0)},\qquad \hat{\qb}\coloneqq \hat{\qb}^{(1)}=\mathrm{TF}(\Zb;\Wb)_{:,-1}.
\]
Since all vectors lie on the unit sphere,
\[
\|\mathbf x_i-\mathbf x_q\|^2=2-2\,\mathbf x_i^\top \mathbf x_q,
\]
hence \(\pi\) is a nearest-neighbor index if and only if it maximizes \(U_i=\mathbf x_i^\top \mathbf x_q\). Conditional on \(\mathbf x_q\), the random variables \(U_i\) are i.i.d.\ with a continuous law on \((-1,1)\) when \(d\ge 2\); therefore \(\pi=\arg\max_{i\in[N]}U_i\) is unique almost surely. We work on this full-probability event from now on. 


\begin{lemma}
\label{lem:logits}
For every realization of \((\mathbf x_i,y_i,\mathbf p_i)_{i=1}^{N+1}\) and every \((\alpha,\beta)\in\mathbb R^2\), let
\[
\ab(\alpha,\beta)\coloneqq\mathrm{softmax}(\Zb^\top \Wb(\alpha,\beta)\qb)
=
(a_1(\alpha,\beta),\dots,a_N(\alpha,\beta),a_0(\alpha,\beta))^\top
\]
be the attention weights. Then
\[
a_i(\alpha,\beta)=\frac{e^{\alpha U_i}}{e^{\alpha+\beta}+\sum_{j=1}^N e^{\alpha U_j}},
\qquad
a_0(\alpha,\beta)=\frac{e^{\alpha+\beta}}{e^{\alpha+\beta}+\sum_{j=1}^N e^{\alpha U_j}}.
\]
Moreover,
\[
\hat{\qb}(\alpha,\beta)
=
\begin{bmatrix}
\mathbf x_q\\[2mm]
\sum_{i=1}^N a_i(\alpha,\beta)\,y_i\\[2mm]
\bigl(1+a_0(\alpha,\beta)\bigr)\mathbf p_{N+1}
+\sum_{i=1}^N a_i(\alpha,\beta)\,\mathbf p_i
\end{bmatrix}.
\]
\end{lemma}

\begin{proof}
Since
\[
\mathbf W(\alpha,\beta)=\operatorname{diag}(\alpha \Ib_d,\,0,\,\beta \Ib_{N+1}),
\]
the label coordinate does not contribute to the score. For \(i\in[N]\),
\[
\mathbf h_i^\top \mathbf W(\alpha,\beta)\mathbf q^{(0)}
=
\alpha\,\mathbf x_i^\top \mathbf x_q+\beta\,\mathbf p_i^\top \mathbf p_{N+1}.
\]
Because \(\tau\) is a permutation, the vectors \(\mathbf p_1,\dots,\mathbf p_{N+1}\) are orthonormal up to signs, so \(\mathbf p_i^\top \mathbf p_{N+1}=0\) for \(i\le N\). Hence
\[
\mathbf h_i^\top \mathbf W(\alpha,\beta)\mathbf q^{(0)}=\alpha U_i.
\]
Similarly,
\[
\mathbf q^{(0)\top}\mathbf W(\alpha,\beta)\mathbf q^{(0)}
=
\alpha\|\mathbf x_q\|^2+\beta\|\mathbf p_{N+1}\|^2
=
\alpha+\beta.
\]
Substituting these scores into the softmax gives the stated formula for \(\ab(\alpha,\beta)\).

Next,
\[
\mathbf V\mathbf h_i=
\begin{bmatrix}
0\\ y_i\\ \mathbf p_i
\end{bmatrix},
\qquad
\mathbf V\mathbf q^{(0)}=
\begin{bmatrix}
0\\ 0\\ \mathbf p_{N+1}
\end{bmatrix}.
\]
Therefore
\[
\hat{\qb}(\alpha,\beta)
=
\mathbf q^{(0)}
+
\sum_{i=1}^N a_i(\alpha,\beta)\,\mathbf V\mathbf h_i
+
a_0(\alpha,\beta)\,\mathbf V\mathbf q^{(0)},
\]
which is exactly
\[
\hat{\qb}(\alpha,\beta)
=
\begin{bmatrix}
\mathbf x_q\\[2mm]
\sum_{i=1}^N a_i(\alpha,\beta)\,y_i\\[2mm]
\bigl(1+a_0(\alpha,\beta)\bigr)\mathbf p_{N+1}
+\sum_{i=1}^N a_i(\alpha,\beta)\,\mathbf p_i
\end{bmatrix}.
\]
\end{proof}

\begin{lemma}
\label{prop:population-reduction}
The sample loss satisfies
\[
\ell(\alpha,\beta)
=
\frac12\|\rb(\alpha,\beta)\|^2
=
\frac12\Bigl(\sum_{i=1}^N a_i y_i-y_\pi\Bigr)^2
+
\frac12\left\|
a_0\mathbf p_{N+1}+\sum_{i=1}^N a_i\mathbf p_i-\mathbf p_\pi
\right\|^2.
\]

Moreover, for every \((\alpha,\beta)\in\mathbb R^2\), let
\[
\mathcal G\coloneqq \sigma(\mathbf x_1,\dots,\mathbf x_N,\mathbf x_q,\pi),
\]
then
\[
\varphi(U_1,\dots,U_N;\alpha,\beta)\coloneqq\mathbb E\left[\ell(\alpha,\beta)\mid \mathcal G\right]
=
1+\sum_{i=1}^N a_i(\alpha,\beta)^2+\frac12\,a_0(\alpha,\beta)^2-2a_\pi(\alpha,\beta).
\]
Hence, the population loss is exactly
\[
\mathcal L(\alpha,\beta)=\mathbb E\bigl[\varphi(U_1,\dots,U_N;\alpha,\beta)\bigr].
\]
\end{lemma}

\begin{proof}
By Lemma~\ref{lem:logits},
\[
\rb(\alpha,\beta)=\hat{\qb}(\alpha,\beta)-\qb^{(1)}
=
\begin{bmatrix}
0\\[1mm]
\sum_{i=1}^N a_i y_i-y_\pi\\[1mm]
a_0\mathbf p_{N+1}+\sum_{i=1}^N a_i\mathbf p_i-\mathbf p_\pi
\end{bmatrix}.
\]
The \(x\)-block is therefore identically zero.

For the positional block, write
\[
\mathbf r_p\coloneqq a_0\mathbf p_{N+1}+\sum_{i=1}^N a_i\mathbf p_i-\mathbf p_\pi.
\]
Since
\[
\mathbf p_i^\top \mathbf p_j
=
\delta_{ij}
\]
for every realization of \((\tau,s_1,\dots,s_{N+1})\), all cross terms vanish exactly. Hence
\[
\|\mathbf r_p\|^2
=
a_0^2+\sum_{i\neq\pi}a_i^2+(a_\pi-1)^2
=
a_0^2+\sum_{i=1}^N a_i^2+1-2a_\pi.
\]
For the label block, decompose
\[
\sum_{i=1}^N a_i y_i-y_\pi
=
\sum_{i\neq\pi} a_i y_i+(a_\pi-1)y_\pi.
\]
Therefore
\begin{align*}
\Bigl(\sum_{i=1}^N a_i y_i-y_\pi\Bigr)^2
=
\sum_{i\neq\pi} a_i^2+(a_\pi-1)^2 
+2\sum_{\substack{i<j\\ i,j\neq\pi}} a_i a_j y_i y_j
+2(a_\pi-1)\sum_{i\neq\pi} a_i y_i y_\pi.
\end{align*}
Conditional on \(\mathcal G\), by Assumption \ref{ass:distribution}, we have
\[
\mathbb E[y_i y_j\mid \mathcal G]=0\quad(i\neq j),\qquad
\mathbb E[y_i^2\mid \mathcal G]=1.
\]
Taking conditional expectation therefore yields
\[
\mathbb E\!\left[\Bigl(\sum_{i=1}^N a_i y_i-y_\pi\Bigr)^2\middle|\mathcal G\right]
=
\sum_{i\neq\pi} a_i^2+(a_\pi-1)^2
=
\sum_{i=1}^N a_i^2+1-2a_\pi.
\]
Combining the two block computations gives
\begin{align*}
\mathbb E[\ell(\alpha,\beta)\mid \mathcal G]
&=
\frac12\Bigl(\sum_{i=1}^N a_i^2+1-2a_\pi\Bigr)
+\frac12\Bigl(a_0^2+\sum_{i=1}^N a_i^2+1-2a_\pi\Bigr) \\
&=
1+\sum_{i=1}^N a_i^2+\frac12\,a_0^2-2a_\pi.
\end{align*}
The formula for \(\mathcal L\) now follows from the tower property:
\[
\mathcal L(\alpha,\beta)
=
\mathbb E\bigl[\mathbb E[\ell(\alpha,\beta)\mid \mathcal G]\bigr]
=
\mathbb E[\varphi(U_1,\dots,U_N;\alpha,\beta)].
\]
\end{proof}

\begin{lemma}
\label{lem:softmax}
Fix a realization of \(U_1,\dots,U_N\), and define
\[
\ub(\alpha,\beta)\coloneqq \bigl(\alpha U_1,\dots,\alpha U_N,\alpha+\beta\bigr)\in\mathbb R^{N+1},
\qquad
\ab(\alpha,\beta)\coloneqq (a_1,\dots,a_N,a_0)\coloneqq \mathrm{softmax}(\ub).
\]
Then for \(\theta\in\{\alpha,\beta\}\),
\[
\partial_\theta a_j
=
a_j\Bigl(\partial_\theta u_j-\sum_{k=0}^N a_k\,\partial_\theta u_k\Bigr).
\]
For \(\theta,\theta'\in\{\alpha,\beta\}\),
\[
\partial_{\theta'}\partial_\theta a_j
=
(\partial_{\theta'}a_j)\Bigl(\partial_\theta u_j-\sum_{k=0}^N a_k\,\partial_\theta u_k\Bigr)
-
a_j\sum_{k=0}^N (\partial_{\theta'}a_k)\,\partial_\theta u_k.
\]
Moreover, for every \(j\) and \(\theta\),
\[
|\partial_\theta a_j|\le 2a_j,
\qquad
|\partial_{\theta'}\partial_\theta a_j|\le 6a_j.
\]
Consequently, for the reduced sample loss \(\varphi\) of Lemma~\ref{prop:population-reduction},
\[
|\partial_\theta \varphi|\le 10,
\qquad
|\partial_{\theta'}\partial_\theta \varphi|\le 42
\qquad
(\theta,\theta'\in\{\alpha,\beta\}).
\]
\end{lemma}

\begin{proof}
By the definition of \(\ub\),
\[
\partial_\alpha u_i=U_i,\quad i\in[N],\qquad \partial_\alpha u_0=1,
\]
and
\[
\partial_\beta u_i=0,\quad i\in[N],\qquad \partial_\beta u_0=1.
\]
Differentiating
\[
a_j=\frac{e^{u_j}}{\sum_{\ell=0}^N e^{u_\ell}}
\]
with respect to \(u_k\) gives
\[
\frac{\partial a_j}{\partial u_k}=a_j(\delta_{jk}-a_k).
\]
Applying the chain rule,
\[
\partial_\theta a_j
=
\sum_{k=0}^N \frac{\partial a_j}{\partial u_k}\,\partial_\theta u_k
=
\sum_{k=0}^N a_j(\delta_{jk}-a_k)\,\partial_\theta u_k
=
a_j\Bigl(\partial_\theta u_j-\sum_{k=0}^N a_k\,\partial_\theta u_k\Bigr).
\]

Differentiating once more and using that \(\partial_{\theta'}\partial_\theta u_j=0\) because \(\ub\) is affine in \((\alpha,\beta)\), we obtain
\[
\partial_{\theta'}\partial_\theta a_j
=
(\partial_{\theta'}a_j)\Bigl(\partial_\theta u_j-\sum_{k=0}^N a_k\,\partial_\theta u_k\Bigr)
-
a_j\sum_{k=0}^N (\partial_{\theta'}a_k)\,\partial_\theta u_k.
\]

Since \(|\partial_\theta u_j|\le 1\), we also have
\[
\left|\sum_{k=0}^N a_k\,\partial_\theta u_k\right|\le 1,
\]
hence
\[
\left|\partial_\theta u_j-\sum_{k=0}^N a_k\,\partial_\theta u_k\right|\le 2.
\]
Therefore
\[
|\partial_\theta a_j|
\le
a_j\cdot 2
=
2a_j.
\]
Using this in the second-derivative formula,
\[
|\partial_{\theta'}\partial_\theta a_j|
\le
|\partial_{\theta'}a_j|\cdot 2
+
a_j\sum_{k=0}^N |\partial_{\theta'}a_k|\,|\partial_\theta u_k|
\le
4a_j+a_j\sum_{k=0}^N 2a_k
=
6a_j.
\]

Now
\[
\varphi=1+\sum_{i=1}^N a_i^2+\frac12\,a_0^2-2a_\pi,
\]
so
\[
\partial_\theta \varphi
=
2\sum_{i=1}^N a_i\,\partial_\theta a_i
+
a_0\,\partial_\theta a_0
-
2\,\partial_\theta a_\pi.
\]
Using \(|\partial_\theta a_j|\le 2a_j\),
\[
|\partial_\theta \varphi|
\le
2\sum_{i=1}^N a_i\cdot 2a_i
+
a_0\cdot 2a_0
+
2\cdot 2a_\pi
\le
4\sum_{i=1}^N a_i^2+2a_0^2+4a_\pi
\le 10.
\]

Differentiating once more,
\[
\partial_{\theta'}\partial_\theta \varphi
=
2\sum_{i=1}^N\Bigl((\partial_{\theta'}a_i)(\partial_\theta a_i)
+a_i\,\partial_{\theta'}\partial_\theta a_i\Bigr)
+
(\partial_{\theta'}a_0)(\partial_\theta a_0)
+
a_0\,\partial_{\theta'}\partial_\theta a_0
-
2\,\partial_{\theta'}\partial_\theta a_\pi,
\]
hence
\begin{align*}
|\partial_{\theta'}\partial_\theta \varphi|
&\le
2\sum_{i=1}^N \bigl(4a_i^2+6a_i^2\bigr)
+
\bigl(4a_0^2+6a_0^2\bigr)
+
12a_\pi \\
&\le
20\sum_{i=1}^N a_i^2+10a_0^2+12a_\pi
\le 42.
\end{align*}
This proves the lemma.
\end{proof}

\begin{lemma}
\label{prop:gradients}
Define
\[
b_i(\alpha)\coloneqq \frac{e^{\alpha U_i}}{\sum_{j=1}^N e^{\alpha U_j}},
\qquad
Q(\alpha)\coloneqq \sum_{i=1}^N b_i(\alpha)^2,
\qquad
\mu(\alpha)\coloneqq \sum_{i=1}^N b_i(\alpha)U_i.
\]
Then
\[
\varphi(\alpha,\beta)
=
1+\bigl(1-a_0(\alpha,\beta)\bigr)^2Q(\alpha)
+\frac12\,a_0(\alpha,\beta)^2
-2\bigl(1-a_0(\alpha,\beta)\bigr)b_\pi(\alpha),
\]
\[
 \partial_\alpha b_\pi(\alpha)=b_\pi(\alpha)\bigl(U_\pi-\mu(\alpha)\bigr),
\qquad
Q'(\alpha)=2\sum_{i=1}^N b_i(\alpha)^2\bigl(U_i-\mu(\alpha)\bigr),
\]
and the reduced sample gradients are
\begin{equation}\label{eq:sample-beta-gradient}
\partial_\beta \varphi
=
a_0(\alpha,\beta)\bigl(1-a_0(\alpha,\beta)\bigr)
\Bigl(
a_0(\alpha,\beta)+2b_\pi(\alpha)-2\bigl(1-a_0(\alpha,\beta)\bigr)Q(\alpha)
\Bigr),
\end{equation}
\begin{align}\label{eq:sample-alpha-gradient}
\partial_\alpha \varphi
&=
a_0(\alpha,\beta)\bigl(1-a_0(\alpha,\beta)\bigr)\bigl(1-\mu(\alpha)\bigr)
\Bigl(
a_0(\alpha,\beta)+2b_\pi(\alpha)-2\bigl(1-a_0(\alpha,\beta)\bigr)Q(\alpha)
\Bigr)
\nonumber\\&+\bigl(1-a_0(\alpha,\beta)\bigr)^2Q'(\alpha)
-2\bigl(1-a_0(\alpha,\beta)\bigr)\partial_\alpha b_\pi(\alpha).
\end{align}
Moreover,
\[
\mathcal L\in C^2(\mathbb R^2),
\qquad
\partial_\theta \mathcal L=\mathbb E[\partial_\theta \varphi],
\qquad
\partial_{\theta'}\partial_\theta \mathcal L=\mathbb E[\partial_{\theta'}\partial_\theta \varphi]
\]
for \(\theta,\theta'\in\{\alpha,\beta\}\), and \(\nabla\mathcal L\) is globally \(84\)-Lipschitz on \(\mathbb R^2\).

Finally, along the invariant manifold, the full matrix gradient-descent recursion induces
\[
\alpha^{(t+1)}
=
\alpha^{(t)}-\frac{\eta}{d}\,
\mathbb E\!\left[\partial_\alpha \varphi\!\left(U_1,\dots,U_N;\alpha^{(t)},\beta^{(t)}\right)\right],
\]
\[
\beta^{(t+1)}
=
\beta^{(t)}-\frac{\eta}{N+1}\,
\mathbb E\!\left[\partial_\beta \varphi\!\left(U_1,\dots,U_N;\alpha^{(t)},\beta^{(t)}\right)\right].
\]
\end{lemma}

\begin{proof}
The identity
\[
a_i(\alpha,\beta)=\bigl(1-a_0(\alpha,\beta)\bigr)b_i(\alpha),
\qquad i\in[N],
\]
follows immediately from the definitions of \(a_i\), \(a_0\), and \(b_i\). Substituting this into the formula of Lemma~\ref{prop:population-reduction} gives
\[
\varphi(\alpha,\beta)
=
1+\bigl(1-a_0(\alpha,\beta)\bigr)^2Q(\alpha)
+\frac12\,a_0(\alpha,\beta)^2
-2\bigl(1-a_0(\alpha,\beta)\bigr)b_\pi(\alpha).
\]

Because
\[
\log\frac{a_0(\alpha,\beta)}{1-a_0(\alpha,\beta)}
=
(\alpha+\beta)-\log\!\left(\sum_{j=1}^N e^{\alpha U_j}\right),
\]
and
\[
\partial_\alpha \log\!\left(\sum_{j=1}^N e^{\alpha U_j}\right)
=
\frac{\sum_{j=1}^N U_j e^{\alpha U_j}}{\sum_{j=1}^N e^{\alpha U_j}}
=
\mu(\alpha),
\]
we obtain
\[
\partial_\beta \log\frac{a_0(\alpha,\beta)}{1-a_0(\alpha,\beta)}=1,
\qquad
\partial_\alpha \log\frac{a_0(\alpha,\beta)}{1-a_0(\alpha,\beta)}=1-\mu(\alpha).
\]
Since the derivative of the logistic map is \(u\mapsto u(1-u)\), this yields
\[
\partial_\beta a_0(\alpha,\beta)=a_0(\alpha,\beta)\bigl(1-a_0(\alpha,\beta)\bigr),
\]
\[
\partial_\alpha a_0(\alpha,\beta)=a_0(\alpha,\beta)\bigl(1-a_0(\alpha,\beta)\bigr)\bigl(1-\mu(\alpha)\bigr).
\]
Similarly,
\[
\partial_\alpha b_i(\alpha)
=
b_i(\alpha)\bigl(U_i-\mu(\alpha)\bigr).
\]
Therefore
\[
\partial_\alpha b_\pi(\alpha)=b_\pi(\alpha)\bigl(U_\pi-\mu(\alpha)\bigr),
\]
and
\[
Q'(\alpha)
=
2\sum_{i=1}^N b_i(\alpha)\,\partial_\alpha b_i(\alpha)
=
2\sum_{i=1}^N b_i(\alpha)^2\bigl(U_i-\mu(\alpha)\bigr).
\]
Now differentiate
\[
\varphi
=
1+\bigl(1-a_0(\alpha,\beta)\bigr)^2Q(\alpha)
+\frac12\,a_0(\alpha,\beta)^2
-2\bigl(1-a_0(\alpha,\beta)\bigr)b_\pi(\alpha).
\]
With respect to \(\beta\), only \(a_0(\alpha,\beta)\) varies, so
\[
\partial_\beta \varphi
=
\partial_\beta a_0(\alpha,\beta)
\Bigl(
a_0(\alpha,\beta)+2b_\pi(\alpha)-2\bigl(1-a_0(\alpha,\beta)\bigr)Q(\alpha)
\Bigr),
\]
which is exactly \eqref{eq:sample-beta-gradient}. With respect to \(\alpha\),
\[
\partial_\alpha \varphi
=
\partial_\alpha a_0(\alpha,\beta)
\Bigl(
a_0(\alpha,\beta)+2b_\pi(\alpha)-2\bigl(1-a_0(\alpha,\beta)\bigr)Q(\alpha)
\Bigr)
+\bigl(1-a_0(\alpha,\beta)\bigr)^2Q'(\alpha)
-2\bigl(1-a_0(\alpha,\beta)\bigr)\partial_\alpha b_\pi(\alpha),
\]
which is \eqref{eq:sample-alpha-gradient}.

By Lemma~\ref{lem:softmax}, \(|\partial_\theta\varphi|\le 10\) and \(|\partial_{\theta'}\partial_\theta \varphi|\le 42\) pointwise. These bounds are deterministic and integrable. Hence differentiation under the expectation is justified twice by dominated convergence, and \(\mathcal L\in C^2(\mathbb R^2)\) with
\[
\partial_\theta \mathcal L=\mathbb E[\partial_\theta\varphi],
\qquad
\partial_{\theta'}\partial_\theta\mathcal L=\mathbb E[\partial_{\theta'}\partial_\theta\varphi].
\]
Every Hessian entry of \(\mathcal L\) is bounded in absolute value by \(42\). Since the Hessian is a \(2\times 2\) matrix, its operator norm is bounded by \(84\); therefore \(\nabla\mathcal L\) is globally \(84\)-Lipschitz.

It remains to justify the displayed scalar recursion. Along the invariant manifold, the population matrix gradient has the block form
\[
\nabla_{\Wb} L(\Wb(\alpha,\beta))
=
\diag(g_\alpha \Ib_d,0,g_\beta \Ib_{N+1})
\]
for some scalars \(g_\alpha,g_\beta\). By the chain rule,
\[
\partial_\alpha\mathcal L(\alpha,\beta)
=
\left\langle
\nabla_{\Wb}L(\Wb(\alpha,\beta)),
\diag(\Ib_d,0,0)
\right\rangle_F
=
d\,g_\alpha,
\]
and
\[
\partial_\beta\mathcal L(\alpha,\beta)
=
\left\langle
\nabla_{\Wb}L(\Wb(\alpha,\beta)),
\diag(0,0,\Ib_{N+1})
\right\rangle_F
=
(N+1)\,g_\beta.
\]
Therefore the full matrix update
\[
\Wb^{(t+1)}
=
\Wb^{(t)}-\eta\nabla_{\Wb}L(\Wb^{(t)})
\]
induces
\[
\alpha^{(t+1)}
=
\alpha^{(t)}
-\frac{\eta}{d}\partial_\alpha\mathcal L(\alpha^{(t)},\beta^{(t)}),
\qquad
\beta^{(t+1)}
=
\beta^{(t)}
-\frac{\eta}{N+1}\partial_\beta\mathcal L(\alpha^{(t)},\beta^{(t)}).
\]
Combining this with the identities
\[
\partial_\theta \mathcal L
=
\mathbb E[\partial_\theta\varphi],
\qquad \theta\in\{\alpha,\beta\},
\]
proves the stated recursion.
\end{proof}

\begin{lemma}
\label{lem:M-lower-bound}
We have
\[
\mathbb E[U_\pi]\ge \frac{7e^{-2}}{18\sqrt{2\pi d}}>\frac{7}{432\sqrt d}.
\]
\end{lemma}

\begin{proof}
Let \(G=(G_1,\dots,G_d)\sim N(0,\Ib_d)\) and \(X\coloneqq G/\|G\|\), so \(X\sim \operatorname{Unif}(\mathbb S^{d-1})\). By rotational invariance, the law of \(U_i=\mathbf x_i^\top \mathbf x_q\) is the same as the law of \(X_1\). Set
\[
S\coloneqq \sum_{j=2}^d G_j^2\sim \chi^2_{d-1},
\]
independent of \(G_1\). On the event
\[
E\coloneqq \{1\le G_1\le 2\}\cap\{S\le 4(d-1)\}
\]
we have
\[
X_1=\frac{G_1}{\sqrt{G_1^2+S}}
\ge
\frac{1}{\sqrt{4+4(d-1)}}
=
\frac{1}{2\sqrt d}.
\]
Therefore
\[
\mathbb P\!\left(U_i\ge \frac{1}{2\sqrt d}\right)
\ge \mathbb P(E).
\]
Now
\[
\mathbb P(1\le G_1\le 2)
=
\int_1^2 \frac{e^{-x^2/2}}{\sqrt{2\pi}}\,dx
\ge \frac{e^{-2}}{\sqrt{2\pi}},
\]
and Chebyshev's inequality gives
\[
\mathbb P\bigl(S>4(d-1)\bigr)
=
\mathbb P\bigl(S-(d-1)>3(d-1)\bigr)
\le
\frac{\operatorname{Var}(S)}{9(d-1)^2}
=
\frac{2(d-1)}{9(d-1)^2}
\le \frac29,
\]
since \(d\ge 2\). Hence
\[
\mathbb P(E)\ge \frac{e^{-2}}{\sqrt{2\pi}}\cdot \frac79.
\]
Thus
\[
\mathbb P\!\left(U_i\ge \frac{1}{2\sqrt d}\right)
\ge
\frac{7e^{-2}}{9\sqrt{2\pi}}.
\]
Since \(U_1,\ldots,U_N\) are conditionally independent given \(\mathbf x_q\), we further have
\[
\begin{aligned}
\mathbb P\!\left(U_\pi\ge \frac{1}{2\sqrt d}\right)
&=
1-
\mathbb P\!\left(
U_1<\frac{1}{2\sqrt d},\ldots,
U_N<\frac{1}{2\sqrt d}
\right) \\
&\ge
1-
\left(
1-\frac{7e^{-2}}{9\sqrt{2\pi}}
\right)^N .
\end{aligned}
\]
Since \(N\ge 100\sqrt d\ge 2\) and
\[
\frac{7e^{-2}}{9\sqrt{2\pi}}<\frac12,
\]
we get
\[
\begin{aligned}
\mathbb P\!\left(U_\pi\ge \frac{1}{2\sqrt d}\right)
&\ge
1-
\left(
1-\frac{7e^{-2}}{9\sqrt{2\pi}}
\right)^2  \\
&=
2\cdot \frac{7e^{-2}}{9\sqrt{2\pi}}
-
\left(
\frac{7e^{-2}}{9\sqrt{2\pi}}
\right)^2  \\
&\ge
\frac32\cdot \frac{7e^{-2}}{9\sqrt{2\pi}} .
\end{aligned}
\]

Now we lower bound the expectation while accounting for the possible negative contribution. Since \(U_\pi\in[-1,1]\),
\[
\begin{aligned}
\mathbb E[U_\pi]
&\ge
\frac{1}{2\sqrt d}
\mathbb P\!\left(U_\pi\ge \frac{1}{2\sqrt d}\right)
-
\mathbb P(U_\pi<0).
\end{aligned}
\]
Moreover, by symmetry and conditional independence,
\[
\mathbb P(U_\pi<0)
=
\mathbb P(U_1<0,\ldots,U_N<0)
=
2^{-N}.
\]
Using \(N\ge 100\sqrt d\), we have \(N\ge 4\), and hence \(2^N\ge N^2\). Therefore
\[
2^{-N}\le \frac1{N^2}\le \frac{1}{10000d}\le \frac{1}{10000\sqrt d}.
\]
On the other hand, since \(e^2<8\) and \(\sqrt{2\pi}<3\),
\[
\frac{7e^{-2}}{9\sqrt{2\pi}}
>
\frac{7}{9\cdot 8\cdot 3}
=
\frac{7}{216}
>
\frac1{100}.
\]
Thus
\[
2^{-N}
\le
\frac{1}{10000\sqrt d}
\le
\frac{1}{4\sqrt d}\cdot
\frac{7e^{-2}}{9\sqrt{2\pi}}.
\]
Combining the preceding estimates gives
\[
\begin{aligned}
\mathbb E[U_\pi]
&\ge
\frac{1}{2\sqrt d}
\cdot
\frac32\cdot
\frac{7e^{-2}}{9\sqrt{2\pi}}
-
\frac{1}{4\sqrt d}\cdot
\frac{7e^{-2}}{9\sqrt{2\pi}}  \\
&=
\frac{7e^{-2}}{18\sqrt{2\pi d}}.
\end{aligned}
\]
Finally, \(e^2<8\) and \(\sqrt{2\pi}<3\), so
\[
\frac{7e^{-2}}{18\sqrt{2\pi d}}
>
\frac{7}{18\cdot 8\cdot 3\sqrt d}
=
\frac{7}{432\sqrt d}.
\]
This proves the lemma.
\end{proof}

\begin{lemma}
\label{lem:alpha-boundary}
Assume \(N\ge 100\sqrt d\). Then for every \(\beta\le 0\),
\[
\partial_\alpha \mathcal L(0,\beta)\le -\frac{1}{(N+1)^2}.
\]
\end{lemma}

\begin{proof}
At \(\alpha=0\), we have
\[
b_i(0)=\frac1N,\qquad b_\pi(0)=\frac1N,\qquad Q(0)=\frac1N,
\]
and
\[
\mu(0)=\frac1N\sum_{i=1}^N U_i=: \overline U.
\]
Also
\[
Q'(0)=2\sum_{i=1}^N \frac1{N^2}(U_i-\overline U)=0,
\qquad
\partial_\alpha b_\pi(0)=\frac1N(U_\pi-\overline U).
\]
Since \(\beta\le 0\),
\[
a_0(0,\beta)=\frac{1}{1+Ne^{-\beta}}\le \frac1{N+1}.
\]
Substituting these identities into \eqref{eq:sample-alpha-gradient} yields
\[
\partial_\alpha \varphi(0,\beta)
=
a_0(0,\beta)\bigl(1-a_0(0,\beta)\bigr)(1-\overline U)
\Bigl(
a_0(0,\beta)+\frac{2}{N}-\frac{2(1-a_0(0,\beta))}{N}
\Bigr)
-\frac{2(1-a_0(0,\beta))}{N}(U_\pi-\overline U).
\]
The bracket simplifies to \(a_0(0,\beta)(1+2/N)\), hence
\[
\partial_\alpha \varphi(0,\beta)
=
a_0(0,\beta)^2\bigl(1-a_0(0,\beta)\bigr)\Bigl(1+\frac2N\Bigr)(1-\overline U)
-\frac{2(1-a_0(0,\beta))}{N}(U_\pi-\overline U).
\]
Taking expectation and using \(\mathbb E[\overline U]=0\) gives
\[
\partial_\alpha \mathcal L(0,\beta)
=
a_0(0,\beta)^2\bigl(1-a_0(0,\beta)\bigr)\Bigl(1+\frac2N\Bigr)
-\frac{2(1-a_0(0,\beta))}{N}\,\mathbb E[U_\pi].
\]
Since \(a_0(0,\beta)\le \frac1{N+1}\),
\[
a_0(0,\beta)^2\bigl(1-a_0(0,\beta)\bigr)\Bigl(1+\frac2N\Bigr)
\le
\frac{1}{(N+1)^2}\Bigl(1+\frac2N\Bigr)
<\frac{2}{(N+1)^2},
\]
and
\[
\frac{2(1-a_0(0,\beta))}{N}\ge \frac{2}{N+1}.
\]
Therefore
\[
\partial_\alpha \mathcal L(0,\beta)
\le
\frac{2}{(N+1)^2}
-
\frac{2}{N+1}\,\mathbb E[U_\pi].
\]
By Lemma~\ref{lem:M-lower-bound},
\[
\mathbb E[U_\pi]>\frac{7}{432\sqrt d}.
\]
Since \(N\ge 100\sqrt d\), we have \(N+1\ge 100\sqrt d\), so
\[
\frac{2}{N+1}\,\mathbb E[U_\pi]
>
\frac{14}{432\sqrt d\,(N+1)}
\ge
\frac{1400}{432(N+1)^2}
>
\frac{3}{(N+1)^2}.
\]
Hence
\[
\partial_\alpha \mathcal L(0,\beta)
<
\frac{2}{(N+1)^2}-\frac{3}{(N+1)^2}
=
-\frac{1}{(N+1)^2}.
\]
\end{proof}

\begin{lemma}\label{prop:invariant}
Assume \(d\ge 2\), \(N\ge 100\sqrt d\), and
\[
0<\eta\le \frac{1}{840(N+1)^2}.
\]
Then, along the gradient-descent trajectory issued from \((0,0)\),
\[
\alpha^{(t)}\ge 0,\qquad \beta^{(t)}\le 0\qquad\text{for all }t\ge 0.
\]
Moreover, for every \((\alpha,\beta)\) with \(\alpha\ge 0\),
\[
\partial_\beta \mathcal L(\alpha,\beta)>0.
\]
Consequently \((\beta^{(t)})_{t\ge0}\) is strictly decreasing.
\end{lemma}

\begin{proof}
We first prove positivity of \(\partial_\beta\mathcal L\) on \(\{\alpha\ge 0\}\). Fix a realization of \(U_1,\dots,U_N\) and \(\alpha\ge 0\). Since \(\pi\) indexes the largest \(U_i\), the map \(u\mapsto e^{\alpha u}\) is nondecreasing, and therefore
\[
b_\pi(\alpha)=\max_{1\le i\le N} b_i(\alpha).
\]
Hence
\[
Q(\alpha)=\sum_{i=1}^N b_i(\alpha)^2
\le
b_\pi(\alpha)\sum_{i=1}^N b_i(\alpha)
=
b_\pi(\alpha).
\]
Using \eqref{eq:sample-beta-gradient},
\[
\partial_\beta \varphi
=
a_0(\alpha,\beta)\bigl(1-a_0(\alpha,\beta)\bigr)
\Bigl(
a_0(\alpha,\beta)+2b_\pi(\alpha)-2\bigl(1-a_0(\alpha,\beta)\bigr)Q(\alpha)
\Bigr)
\]
\[
=
a_0(\alpha,\beta)\bigl(1-a_0(\alpha,\beta)\bigr)
\Bigl(
a_0(\alpha,\beta)+2\bigl(b_\pi(\alpha)-Q(\alpha)\bigr)+2a_0(\alpha,\beta)Q(\alpha)
\Bigr).
\]
All terms inside the last parentheses are nonnegative, and \(a_0(\alpha,\beta)\in(0,1)\), so
\[
\partial_\beta \varphi>0.
\]
Taking expectations gives \(\partial_\beta \mathcal L(\alpha,\beta)>0\) for every \(\alpha\ge 0\).

We now prove forward invariance. Set
\[
\delta\coloneqq \frac{1}{84(N+1)^2}.
\]
By Lemma~\ref{prop:gradients}, \(|\partial_\alpha \mathcal L|\le 10\) and \(|\partial_{\alpha\alpha}\mathcal L|\le 42\) everywhere. We argue by induction on \(t\).

At \(t=0\) we have \(\alpha^{(0)}=0\) and \(\beta^{(0)}=0\). Assume that
\[
\alpha^{(t)}\ge 0,\qquad \beta^{(t)}\le 0.
\]
If \(0\le \alpha^{(t)}\le \delta\), then by the mean-value theorem and Lemma~\ref{lem:alpha-boundary},
\[
\partial_\alpha \mathcal L(\alpha^{(t)},\beta^{(t)})
\le
\partial_\alpha \mathcal L(0,\beta^{(t)})+42\,\alpha^{(t)}
\le
-\frac{1}{(N+1)^2}+\frac{42}{84(N+1)^2}
=
-\frac{1}{2(N+1)^2}
<0.
\]
Hence
\[
\alpha^{(t+1)}
=
\alpha^{(t)}-\frac{\eta}{d}\,\partial_\alpha \mathcal L(\alpha^{(t)},\beta^{(t)})
>
\alpha^{(t)}
\ge 0.
\]

If \(\alpha^{(t)}\ge \delta\), then
\[
\alpha^{(t+1)}
\ge
\alpha^{(t)}-\frac{\eta}{d}\,|\partial_\alpha \mathcal L(\alpha^{(t)},\beta^{(t)})|
\ge
\alpha^{(t)}-\frac{10\eta}{d}
\ge
\delta-\frac{10}{840d(N+1)^2}
\ge
0.
\]
Thus \(\alpha^{(t+1)}\ge 0\) in all cases.

Since \(\alpha^{(t)}\ge 0\), the already-proved positivity of \(\partial_\beta \mathcal L\) applies, and therefore
\[
\beta^{(t+1)}
=
\beta^{(t)}-\frac{\eta}{N+1}\,\partial_\beta \mathcal L(\alpha^{(t)},\beta^{(t)})
<
\beta^{(t)}
\le 0.
\]
This closes the induction and proves
\[
\alpha^{(t)}\ge 0,\qquad \beta^{(t)}\le 0
\]
for every \(t\), together with the strict monotonic decrease of \(\beta^{(t)}\).
\end{proof}

\begin{lemma}
\label{prop:descent}
Under the assumptions of Theorem~\ref{thm:training convergence}, if
\[
\theta_t\coloneqq (\alpha^{(t)},\beta^{(t)}),
\qquad
L_t\coloneqq \mathcal L(\theta_t),
\]
then
\[
L_{t+1}
\le
L_t
-\frac{\eta}{2}
\left(
\frac{1}{d}\bigl(\partial_\alpha\mathcal L(\theta_t)\bigr)^2
+
\frac{1}{N+1}\bigl(\partial_\beta\mathcal L(\theta_t)\bigr)^2
\right)
\qquad\text{for all }t\ge 0.
\]
Consequently,
\[
\sum_{t=0}^\infty
\left(
\frac{1}{d}\bigl(\partial_\alpha\mathcal L(\theta_t)\bigr)^2
+
\frac{1}{N+1}\bigl(\partial_\beta\mathcal L(\theta_t)\bigr)^2
\right)
<\infty,
\qquad
\|\nabla\mathcal L(\theta_t)\|\to 0,
\]
and
\[
|\alpha^{(t+1)}-\alpha^{(t)}|\le \frac{10\eta}{d}\le 10\eta<1
\qquad\text{for all }t.
\]
Moreover, the trajectory \((\theta_t)_{t\ge0}\) is unbounded in \(\mathbb R^2\).
\end{lemma}

\begin{proof}
By Lemma~\ref{prop:gradients}, \(\nabla\mathcal L\) is globally \(84\)-Lipschitz. The induced scalar recursion gives
\[
\theta_{t+1}
=
\left(
\alpha^{(t)}-\frac{\eta}{d}\partial_\alpha\mathcal L(\theta_t),
\,
\beta^{(t)}-\frac{\eta}{N+1}\partial_\beta\mathcal L(\theta_t)
\right).
\]
The descent lemma therefore gives
\begin{align*}
L_{t+1}
&\le
L_t
-\eta
\left(
\frac{1}{d}\bigl(\partial_\alpha\mathcal L(\theta_t)\bigr)^2
+
\frac{1}{N+1}\bigl(\partial_\beta\mathcal L(\theta_t)\bigr)^2
\right) \\
&\quad
+\frac{84\eta^2}{2}
\left(
\frac{1}{d^2}\bigl(\partial_\alpha\mathcal L(\theta_t)\bigr)^2
+
\frac{1}{(N+1)^2}\bigl(\partial_\beta\mathcal L(\theta_t)\bigr)^2
\right).
\end{align*}
Since \(d\ge 1\) and \(N+1\ge 1\),
\[
\frac{1}{d^2}\bigl(\partial_\alpha\mathcal L(\theta_t)\bigr)^2
+
\frac{1}{(N+1)^2}\bigl(\partial_\beta\mathcal L(\theta_t)\bigr)^2
\le
\frac{1}{d}\bigl(\partial_\alpha\mathcal L(\theta_t)\bigr)^2
+
\frac{1}{N+1}\bigl(\partial_\beta\mathcal L(\theta_t)\bigr)^2.
\]
Also \(\eta\le \frac{1}{840(N+1)^2}<\frac1{84}\), so
\[
1-\frac{84\eta}{2}\ge \frac12.
\]
Hence
\[
L_{t+1}
\le
L_t
-\frac{\eta}{2}
\left(
\frac{1}{d}\bigl(\partial_\alpha\mathcal L(\theta_t)\bigr)^2
+
\frac{1}{N+1}\bigl(\partial_\beta\mathcal L(\theta_t)\bigr)^2
\right).
\]
Because \(\mathcal L\ge 0\), summing over \(t\) yields
\[
\sum_{t=0}^\infty
\left(
\frac{1}{d}\bigl(\partial_\alpha\mathcal L(\theta_t)\bigr)^2
+
\frac{1}{N+1}\bigl(\partial_\beta\mathcal L(\theta_t)\bigr)^2
\right)
\le
\frac{2L_0}{\eta}
<\infty.
\]
Since \(d\) and \(N\) are fixed finite constants, this implies
\[
\|\nabla\mathcal L(\theta_t)\|\to0.
\]

The bound \(|\partial_\alpha \mathcal L|\le 10\) from Lemma~\ref{prop:gradients} gives
\[
|\alpha^{(t+1)}-\alpha^{(t)}|
=
\frac{\eta}{d}\,|\partial_\alpha \mathcal L(\theta_t)|
\le \frac{10\eta}{d}
\le 10\eta<1.
\]

It remains to exclude bounded trajectories. Suppose that \((\theta_t)\) were bounded. By Lemma~\ref{prop:invariant}, every \(\theta_t\) belongs to the closed half-strip
\[
[0,\infty)\times(-\infty,0].
\]
A bounded sequence in \(\mathbb R^2\) has a convergent subsequence, so there exist \(t_k\to\infty\) and \(\bar\theta=(\bar\alpha,\bar\beta)\in [0,\infty)\times(-\infty,0]\) such that
\[
\theta_{t_k}\to \bar\theta.
\]
Since \(\|\nabla\mathcal L(\theta_t)\|\to 0\) and \(\nabla\mathcal L\) is continuous,
\[
\nabla\mathcal L(\bar\theta)=0.
\]
But Lemma~\ref{prop:invariant} states that \(\partial_\beta\mathcal L(\alpha,\beta)>0\) for every \(\alpha\ge 0\), which contradicts \(\partial_\beta\mathcal L(\bar\theta)=0\). Hence the trajectory is unbounded.
\end{proof}

\begin{lemma}
\label{prop:beta-divergence}
Under the assumptions of Theorem~\ref{thm:training convergence},
\[
\beta^{(t)}\to -\infty\qquad\text{as }t\to\infty.
\]
\end{lemma}

\begin{proof}
By Lemma~\ref{prop:invariant}, \((\beta^{(t)})_{t\ge0}\) is strictly decreasing and bounded above by \(0\). Suppose, toward a contradiction, that it is bounded below. Then monotone convergence yields a finite limit
\[
\beta^{(t)}\to \beta_\infty\in\mathbb R.
\]
By Lemma~\ref{prop:descent}, the full trajectory is unbounded. Since \(\alpha^{(t)}\ge 0\) for all \(t\), this forces \((\alpha^{(t)})_{t\ge0}\) to be unbounded above. Hence there exists a subsequence \(t_k\to\infty\) such that
\[
\alpha^{(t_k)}\to +\infty.
\]

Fix a realization of \(U_1,\dots,U_N\) such that \(U_i<1\) for every \(i\), which holds almost surely because \(U_i=1\) is equivalent to \(\mathbf x_i=\mathbf x_q\), an event of probability zero. Then, for every \(i\in[N]\),
\[
a_i(\alpha,\beta)
=
\frac{e^{\alpha U_i}}{e^{\alpha+\beta}+\sum_{j=1}^N e^{\alpha U_j}}
=
\frac{e^{-\beta-\alpha(1-U_i)}}{1+\sum_{j=1}^N e^{-\beta-\alpha(1-U_j)}}.
\]
If \(\alpha\to+\infty\) and \(\beta\to\beta_\infty\in\mathbb R\), then \(1-U_i>0\), hence
\[
a_i(\alpha,\beta)\to 0,\qquad a_0(\alpha,\beta)=1-\sum_{i=1}^N a_i(\alpha,\beta)\to 1.
\]
Therefore, along the subsequence \(t_k\),
\[
\varphi(U_1,\dots,U_N;\alpha^{(t_k)},\beta^{(t_k)})\to 1+0+\frac12-0=\frac32
\]
almost surely. Since \(\varphi=\mathbb E[\ell(\alpha,\beta)\mid\mathcal G]\), we have \(\varphi\ge 0\). Also,
\[
\varphi
=
1+\sum_{i=1}^N a_i^2+\frac12 a_0^2-2a_\pi
\le 1+\sum_{i=1}^N a_i^2+\frac12 a_0^2
\le \frac52.
\]
Hence dominated convergence gives
\[
L_{t_k}=\mathcal L(\alpha^{(t_k)},\beta^{(t_k)})\to \frac32.
\]

On the other hand, Lemma~\ref{prop:descent} implies \(L_t\le L_0\) for all \(t\). At \((\alpha,\beta)=(0,0)\), all \(N+1\) attention weights equal \((N+1)^{-1}\), so Lemma~\ref{prop:population-reduction} yields
\[
L_0
=
1+\frac{N}{(N+1)^2}+\frac{1}{2(N+1)^2}-\frac{2}{N+1}
=
1-\frac{N+\frac32}{(N+1)^2}
<1.
\]
Thus \(L_t<1\) for all \(t\), which contradicts \(L_{t_k}\to \frac32\). The contradiction proves
\[
\beta^{(t)}\to -\infty.
\]
\end{proof}

\begin{lemma}
\label{prop:alpha-divergence}
Define the data-only reduced loss
\[
\psi(U_1,\dots,U_N;\alpha)\coloneqq 1+Q(\alpha)-2b_\pi(\alpha),
\]
and its population version
\[
\Psi(\alpha)\coloneqq \mathbb E\bigl[\psi(U_1,\dots,U_N;\alpha)\bigr],\qquad \alpha\ge 0.
\]
Then \(\Psi\in C^1([0,\infty))\) and
\[
\Psi'(\alpha)<0\qquad\text{for every }\alpha\ge 0.
\]
Moreover,
\[
\alpha^{(t)}\to +\infty\qquad\text{as }t\to\infty.
\]
\end{lemma}

\begin{proof}
We first prove \(\Psi'(\alpha)<0\). For fixed \(U_1,\dots,U_N\), the derivative of \(\psi\) is
\[
\partial_\alpha \psi = Q' - 2\,\partial_\alpha b_\pi.
\]

If \(\alpha=0\), then \(b_i(0)=1/N\), so
\[
Q(0)=\frac1N,\qquad b_\pi(0)=\frac1N,\qquad Q'(0)=0,\qquad
\partial_\alpha b_\pi(0)=\frac1N\Bigl(U_\pi-\frac1N\sum_{i=1}^N U_i\Bigr).
\]
Since \(N\ge 2\) and \(U_\pi\) is a strict maximum almost surely, we have
\[
U_\pi>\frac1N\sum_{i=1}^N U_i
\]
almost surely, hence
\[
\partial_\alpha \psi(0)
=
-\frac{2}{N}\Bigl(U_\pi-\frac1N\sum_{i=1}^N U_i\Bigr)
<0
\qquad\text{a.s.}
\]

If \(\alpha>0\), then \(b_\pi(\alpha)\) is the unique largest \(b_i(\alpha)\), because \(U_\pi>U_i\) for all \(i\neq\pi\) and \(u\mapsto e^{\alpha u}\) is strictly increasing. Therefore
\[
Q(\alpha)=\sum_{i=1}^N b_i(\alpha)^2 < b_\pi(\alpha)\sum_{i=1}^N b_i(\alpha)=b_\pi(\alpha).
\]
Also \(\mu(\alpha)=\sum_{i=1}^N b_i(\alpha)U_i\) is a strict convex combination of the \(U_i\), so
\[
\mu(\alpha)<U_\pi.
\]
Using the formula for \(Q'\),
\[
Q'
=
2\sum_{i=1}^N b_i(\alpha)^2\bigl(U_i-\mu(\alpha)\bigr)
\le
2U_\pi\sum_{i=1}^N b_i(\alpha)^2-2\mu(\alpha)\sum_{i=1}^N b_i(\alpha)^2
=
2Q(\alpha)\bigl(U_\pi-\mu(\alpha)\bigr).
\]
Therefore
\[
Q'-2\,\partial_\alpha b_\pi
\le
2Q(\alpha)\bigl(U_\pi-\mu(\alpha)\bigr)-2b_\pi(\alpha)\bigl(U_\pi-\mu(\alpha)\bigr)
=
-2\bigl(b_\pi(\alpha)-Q(\alpha)\bigr)\bigl(U_\pi-\mu(\alpha)\bigr)
<0.
\]
Hence \(\partial_\alpha\psi(\alpha)<0\) almost surely for every \(\alpha>0\).

Since \(|Q'|\le 4\) and \(|\partial_\alpha b_\pi|\le 2\), we have \(|\partial_\alpha\psi|\le 8\), and dominated convergence gives \(\Psi\in C^1([0,\infty))\) with
\[
\Psi'(\alpha)=\mathbb E[\partial_\alpha\psi(\alpha)]<0
\qquad (\alpha\ge 0).
\]

We now prove \(\alpha^{(t)}\to+\infty\). Fix \(A>0\). Since \(\Psi'\) is continuous and strictly negative on the compact interval \([0,A+1]\), the quantity
\[
\rho_A\coloneqq  -\frac12\max_{0\le \alpha\le A+1}\Psi'(\alpha)
\]
is well defined and strictly positive.

We next compare \(\partial_\alpha\mathcal L\) with \(\Psi'\) for large negative \(\beta\). For fixed \(U_1,\dots,U_N\) and \(0\le \alpha\le A+1\), Lemma~\ref{prop:gradients} gives
\begin{align*}
\partial_\alpha\varphi-\partial_\alpha\psi
&=
a_0(\alpha,\beta)\bigl(1-a_0(\alpha,\beta)\bigr)\bigl(1-\mu(\alpha)\bigr)
\Bigl(
a_0(\alpha,\beta)+2b_\pi(\alpha)-2\bigl(1-a_0(\alpha,\beta)\bigr)Q(\alpha)
\Bigr) \\
&\quad
+\Bigl(\bigl(1-a_0(\alpha,\beta)\bigr)^2-1\Bigr)Q'(\alpha)
+2a_0(\alpha,\beta)\,\partial_\alpha b_\pi(\alpha).
\end{align*}
Because \(0\le a_0(\alpha,\beta)\le 1\), \(|1-\mu(\alpha)|\le 2\), \(0\le b_\pi(\alpha),Q(\alpha)\le 1\), \(|Q'(\alpha)|\le 4\), and \(|\partial_\alpha b_\pi(\alpha)|\le 2\), this implies
\[
|\partial_\alpha\varphi-\partial_\alpha\psi|
\le 10a_0(\alpha,\beta)+8a_0(\alpha,\beta)+4a_0(\alpha,\beta)
=22a_0(\alpha,\beta).
\]
Moreover,
\[
a_0(\alpha,\beta)
=
\frac{e^{\alpha+\beta}}{e^{\alpha+\beta}+\sum_{j=1}^N e^{\alpha U_j}}
\le
\frac{e^{\alpha+\beta}}{\sum_{j=1}^N e^{\alpha U_j}}
\le
\frac{e^{\alpha+\beta}}{N e^{-\alpha}}
=
\frac{e^{\beta+2\alpha}}{N}
\le
\frac{e^{\beta+2A+2}}{N}.
\]
Hence
\[
\sup_{0\le \alpha\le A+1}
\left|
\partial_\alpha\mathcal L(\alpha,\beta)-\Psi'(\alpha)
\right|
\le
\frac{22e^{\beta+2A+2}}{N}.
\]
Choose \(\bar\beta_A<0\) so negative that
\[
\frac{22e^{\bar\beta_A+2A+2}}{N}\le \rho_A.
\]
Then for every \(\beta\le \bar\beta_A\) and every \(0\le \alpha\le A+1\),
\[
\partial_\alpha\mathcal L(\alpha,\beta)
\le
\Psi'(\alpha)+\rho_A
\le
-\rho_A.
\]
By Lemma~\ref{prop:beta-divergence}, \(\beta^{(t)}\to-\infty\), so there exists \(T_A\) such that
\[
\beta^{(t)}\le \bar\beta_A\qquad\text{for all }t\ge T_A.
\]
Also Lemma~\ref{prop:descent} gives
\[
|\alpha^{(t+1)}-\alpha^{(t)}|
\le \frac{10\eta}{d}
\le 10\eta<1
\qquad\text{for all }t.
\]
If \(\alpha^{(T_A)}>A+1\), set \(S_A\coloneqq T_A\). Otherwise, while \(\alpha^{(t)}\le A+1\) and \(t\ge T_A\), the previous estimate yields
\[
\alpha^{(t+1)}
=
\alpha^{(t)}-\frac{\eta}{d}\,\partial_\alpha\mathcal L(\alpha^{(t)},\beta^{(t)})
\ge
\alpha^{(t)}+\frac{\eta\rho_A}{d}.
\]
Therefore after finitely many steps one reaches some \(S_A\ge T_A\) such that
\[
\alpha^{(S_A)}>A+1.
\]
We claim that
\[
\alpha^{(t)}>A\qquad\text{for all }t\ge S_A.
\]
Suppose the contrary. Then there exists \(t\ge S_A\) such that \(\alpha^{(t+1)}\le A\) and \(\alpha^{(t)}>A\). If \(A<\alpha^{(t)}\le A+1\), then \(\beta^{(t)}\le \bar\beta_A\) and the estimate above gives
\[
\alpha^{(t+1)}>\alpha^{(t)}>A,
\]
a contradiction. If \(\alpha^{(t)}>A+1\), then
\[
\alpha^{(t)}-\alpha^{(t+1)}>1,
\]
which contradicts \(|\alpha^{(t+1)}-\alpha^{(t)}|<1\). This proves the claim.

Since \(A>0\) was arbitrary, for every finite \(A\) there exists \(S_A\) such that \(\alpha^{(t)}>A\) for all \(t\ge S_A\). Therefore
\[
\alpha^{(t)}\to +\infty.
\]
\end{proof}

\begin{theorem}\label{thm:main}
Under Assumption \ref{ass:distribution}, suppose that
\[
d\ge 2,\qquad N\ge 100\sqrt d,\qquad 0<\eta\le \frac{1}{840(N+1)^2}.
\]
Let
\[
(\alpha^{(t)},\beta^{(t)})_{t\ge 0}
\]
be the scalar iterates induced by the full matrix gradient-descent updates on the invariant manifold, initialized at
\[
\alpha^{(0)}=0,\qquad \beta^{(0)}=0.
\]
Equivalently, they satisfy
\[
\alpha^{(t+1)}
=
\alpha^{(t)}
-\frac{\eta}{d}\partial_\alpha\mathcal L(\alpha^{(t)},\beta^{(t)}),
\]
\[
\beta^{(t+1)}
=
\beta^{(t)}
-\frac{\eta}{N+1}\partial_\beta\mathcal L(\alpha^{(t)},\beta^{(t)}).
\]
Then the recursion is well defined for every \(t\ge 0\), and
\[
\alpha^{(t)}\to +\infty,\qquad \beta^{(t)}\to -\infty\qquad\text{as }t\to\infty.
\]
Moreover,
\[
\alpha^{(t)}\ge 0,\qquad \beta^{(t)}\le 0\qquad\text{for all }t\ge 0,
\]
and \((\beta^{(t)})_{t\ge0}\) is strictly decreasing.
\end{theorem}

\begin{proof}[Proof of Theorem~\ref{thm:main}]
Lemma~\ref{prop:gradients} shows that \(\mathcal L\in C^2(\mathbb R^2)\), so the induced scalar recursion is well defined for every \(t\ge 0\). Lemma~\ref{prop:invariant} proves
\[
\alpha^{(t)}\ge 0,\qquad \beta^{(t)}\le 0
\qquad (t\ge 0),
\]
and shows that \((\beta^{(t)})_{t\ge0}\) is strictly decreasing. Lemma~\ref{prop:beta-divergence} yields
\[
\beta^{(t)}\to -\infty,
\]
and Lemma~\ref{prop:alpha-divergence} yields
\[
\alpha^{(t)}\to +\infty.
\]
This is exactly the asserted conclusion.
\end{proof}

\subsection{Exact asymptotic representation of the population loss}
\label{sec:exact asymptotic representation}

We now continue from Theorem~\ref{thm:main}, which established the parameter divergence
\[
\alpha^{(t)}\to+\infty,
\qquad
\beta^{(t)}\to-\infty .
\]
The next goal is to identify the precise quantity that controls the population loss along the actual gradient-descent trajectory. In this subsection, we show that once \(\alpha^{(t)}\) is large, the attention over the context tokens concentrates on the nearest neighbor, and the entire population loss can be reduced asymptotically to the contribution of the self-token weight. The analysis proceeds in two main steps:

\noindent\textbf{Step 1.}
We first prove that the normalized context-token softmax concentrates on the nearest-neighbor index, so that the data-only part of the loss becomes asymptotically negligible.

\noindent\textbf{Step 2.}
Using the exact reduced loss formula derived earlier, we then decompose the population loss into a vanishing context term, lower-order interaction terms, and a leading self-token term. Combining these estimates, we obtain an exact asymptotic representation of the form
\[
\mathcal L(\alpha^{(t)},\beta^{(t)})
=
\frac32\,\EE\!\left[a_0(\alpha^{(t)},\beta^{(t)})^2\right]+o(1),
\]
which shows that the asymptotic behavior of the population loss is determined entirely by the squared self-token weight \(a_0(\alpha^{(t)},\beta^{(t)})^2\).

Retain the notation introduced earlier. In particular,
\[
U_i\coloneqq \mathbf x_i^\top \mathbf x_q,\qquad
\pi\coloneqq \arg\max_{1\le i\le N}U_i,
\]
and for each \(t\ge 0\),
\[
b_i^{(t)}\coloneqq \frac{e^{\alpha^{(t)}U_i}}{\sum_{j=1}^N e^{\alpha^{(t)}U_j}},
\qquad
Q_t\coloneqq \sum_{i=1}^N \bigl(b_i^{(t)}\bigr)^2.
\]
\begin{lemma}
\label{lem:context-concentration-loss}
For almost every realization of \((\mathbf x_1,\dots,\mathbf x_N,\mathbf x_q)\), we have
\[
b_\pi^{(t)}\to 1,\qquad
b_j^{(t)}\to 0\quad (j\in[N]\setminus\{\pi\}),
\qquad
Q_t\to 1
\]
as \(t\to\infty\). If
\[
\psi_t\coloneqq 1+Q_t-2b_\pi^{(t)},
\]
then
\[
\psi_t\to 0\qquad\text{almost surely},
\qquad
\mathbb E[\psi_t]\to 0.
\]
\end{lemma}

\begin{proof}
By Theorem~\ref{thm:main},
\[
\alpha^{(t)}\to+\infty.
\]
Outside the null tie event, \(\pi\) is the unique maximizer of \(U_1,\dots,U_N\), hence
\[
U_\pi-U_j>0\qquad\text{for every }j\in[N]\setminus\{\pi\}.
\]
Therefore
\[
\frac{b_j^{(t)}}{b_\pi^{(t)}}=\exp\bigl(-\alpha^{(t)}(U_\pi-U_j)\bigr)\to 0
\qquad (j\neq\pi).
\]
Since \(\sum_{i=1}^N b_i^{(t)}=1\), it follows that
\[
b_\pi^{(t)}
=
\frac{1}{1+\sum_{j\neq\pi} b_j^{(t)}/b_\pi^{(t)}}
\to 1,
\]
and hence \(b_j^{(t)}\to 0\) for every \(j\neq\pi\). Consequently,
\[
Q_t=\sum_{i=1}^N (b_i^{(t)})^2\to 1
\]
almost surely.

Next,
\[
\psi_t
=
1+Q_t-2b_\pi^{(t)}
=
\bigl(1-b_\pi^{(t)}\bigr)^2+\sum_{j\neq\pi}(b_j^{(t)})^2.
\]
Hence \(\psi_t\ge 0\) and \(\psi_t\to 0\) almost surely. Also \(0\le b_\pi^{(t)},Q_t\le 1\), so
\[
0\le \psi_t\le 2.
\]
Dominated convergence therefore yields
\[
\mathbb E[\psi_t]\to 0.
\]
\end{proof}

\begin{lemma}
\label{prop:loss-reduction-c}
For every \(t\ge 0\),
\[
\mathcal L(\alpha^{(t)},\beta^{(t)})
=
\mathbb E[\psi_t]
+
2\,\mathbb E\!\left[a_0(\alpha^{(t)},\beta^{(t)})\bigl(b_\pi^{(t)}-Q_t\bigr)\right]
+
\mathbb E\!\left[a_0(\alpha^{(t)},\beta^{(t)})^2\left(Q_t+\tfrac12\right)\right].
\]
Moreover,
\[
\mathbb E\!\left[a_0(\alpha^{(t)},\beta^{(t)})\bigl(b_\pi^{(t)}-Q_t\bigr)\right]\to 0,
\qquad
\mathbb E\!\left[a_0(\alpha^{(t)},\beta^{(t)})^2(Q_t-1)\right]\to 0,
\]
and therefore
\begin{equation}\label{eq:L-c2-asymptotic}
\mathcal L(\alpha^{(t)},\beta^{(t)})
=
\frac32\,\mathbb E\!\left[a_0(\alpha^{(t)},\beta^{(t)})^2\right]+o(1)
\qquad\text{as }t\to\infty.
\end{equation}
\end{lemma}

\begin{proof}
From Lemma~\ref{prop:gradients} in the preceding section,
\begin{align*}
\varphi(\alpha,\beta)
&=
1+\bigl(1-a_0(\alpha,\beta)\bigr)^2\sum_{i=1}^N b_i(\alpha)^2
+\frac12\,a_0(\alpha,\beta)^2
-2\bigl(1-a_0(\alpha,\beta)\bigr)b_\pi(\alpha).
\end{align*}
Evaluating this at \((\alpha^{(t)},\beta^{(t)})\) gives
\begin{align*}
\varphi(\alpha^{(t)},\beta^{(t)})
&=
1+\bigl(1-a_0(\alpha^{(t)},\beta^{(t)})\bigr)^2Q_t
+\frac12\,a_0(\alpha^{(t)},\beta^{(t)})^2
-2\bigl(1-a_0(\alpha^{(t)},\beta^{(t)})\bigr)b_\pi^{(t)}.
\end{align*}
Expanding the right-hand side yields
\begin{align*}
\varphi(\alpha^{(t)},\beta^{(t)})
&=
\bigl(1+Q_t-2b_\pi^{(t)}\bigr)
+2a_0(\alpha^{(t)},\beta^{(t)})\bigl(b_\pi^{(t)}-Q_t\bigr)
+a_0(\alpha^{(t)},\beta^{(t)})^2\left(Q_t+\tfrac12\right) \\
&=
\psi_t
+2a_0(\alpha^{(t)},\beta^{(t)})\bigl(b_\pi^{(t)}-Q_t\bigr)
+a_0(\alpha^{(t)},\beta^{(t)})^2\left(Q_t+\tfrac12\right).
\end{align*}
Taking expectation gives the displayed identity for \(\mathcal L(\alpha^{(t)},\beta^{(t)})\).

Next, since \(b_\pi^{(t)}=\max_i b_i^{(t)}\) and \(\sum_i b_i^{(t)}=1\),
\[
Q_t=\sum_{i=1}^N (b_i^{(t)})^2
\le
b_\pi^{(t)}\sum_{i=1}^N b_i^{(t)}
=
b_\pi^{(t)}.
\]
Thus \(b_\pi^{(t)}-Q_t\ge 0\). Also \(Q_t\ge (b_\pi^{(t)})^2\), so
\[
0\le b_\pi^{(t)}-Q_t
\le
b_\pi^{(t)}-(b_\pi^{(t)})^2
=
b_\pi^{(t)}\bigl(1-b_\pi^{(t)}\bigr)
\le
1-b_\pi^{(t)}.
\]
Since \(0\le a_0(\alpha^{(t)},\beta^{(t)})\le 1\),
\[
0\le a_0(\alpha^{(t)},\beta^{(t)})\bigl(b_\pi^{(t)}-Q_t\bigr)\le 1-b_\pi^{(t)}.
\]
By Lemma~\ref{lem:context-concentration-loss}, \(b_\pi^{(t)}\to 1\) almost surely, and \(0\le 1-b_\pi^{(t)}\le 1\), so dominated convergence gives
\[
\mathbb E\!\left[a_0(\alpha^{(t)},\beta^{(t)})\bigl(b_\pi^{(t)}-Q_t\bigr)\right]\to 0.
\]

For the quadratic term,
\[
\mathbb E\!\left[a_0(\alpha^{(t)},\beta^{(t)})^2\left(Q_t+\tfrac12\right)\right]
=
\frac32\,\mathbb E\!\left[a_0(\alpha^{(t)},\beta^{(t)})^2\right]
+
\mathbb E\!\left[a_0(\alpha^{(t)},\beta^{(t)})^2(Q_t-1)\right].
\]
Because \(0\le Q_t\le 1\),
\[
\left|a_0(\alpha^{(t)},\beta^{(t)})^2(Q_t-1)\right|
=
a_0(\alpha^{(t)},\beta^{(t)})^2(1-Q_t)
\le 1-Q_t.
\]
By Lemma~\ref{lem:context-concentration-loss}, \(Q_t\to 1\) almost surely, and \(0\le 1-Q_t\le 1\), hence dominated convergence yields
\[
\mathbb E\!\left[a_0(\alpha^{(t)},\beta^{(t)})^2(Q_t-1)\right]\to 0.
\]

Finally, \(\mathbb E[\psi_t]\to 0\) by Lemma~\ref{lem:context-concentration-loss}. Substituting these three convergences into the exact decomposition of \(\mathcal L(\alpha^{(t)},\beta^{(t)})\) proves \eqref{eq:L-c2-asymptotic}.
\end{proof}

\subsection{Interior subsequential gradient asymptotics}
\label{sec:boundary-layer estimates for certain subsequence}

Having reduced the population loss asymptotically to the self-token contribution, we next study the gradient field along subsequences for which the slope parameter
\[
\lambda_t\coloneqq -\frac{\beta^{(t)}}{\alpha^{(t)}},\quad t\geq 1,
\]
converges to a limit in the interior regime \((0,2)\). The purpose of this subsection is to derive sharp boundary-layer asymptotics for the gradients at the scale \(1/\alpha_t\), which will later be used to control the global dynamics of \(\lambda_t\). The analysis proceeds in three main steps:

\noindent\textbf{Step 1.}
Fix an arbitrary subsequence \(t_k\to\infty\) such that \(\lambda_{t_k}\to\lambda\in(0,2)\). We analyze the logistic kernels centered at the boundary layer
\[
1-U_\pi=\lambda,
\]
and show that these kernels form approximate identities.

\noindent\textbf{Step 2.}
We then estimate the contributions of the non-maximal context tokens. In particular, we prove that all error terms produced by the residual tail and by the exact softmax profile are negligible after multiplication by \(\alpha^{(t_k)}\).

\noindent\textbf{Step 3.}
Finally, we replace the exact self-token profile by its ideal logistic approximation and compute the limiting boundary-layer profiles of the gradients. This yields explicit asymptotic formulas for
\[
\alpha^{(t_k)}\,\partial_\alpha \mathcal L(\alpha^{(t_k)},\beta^{(t_k)})
\quad\text{and}\quad
\alpha^{(t_k)}\Bigl(\partial_\beta \mathcal L(\alpha^{(t_k)},\beta^{(t_k)})+\lambda_{t_k}\partial_\alpha \mathcal L(\alpha^{(t_k)},\beta^{(t_k)})\Bigr),
\]
with strictly positive limits throughout the interior regime \(\lambda\in(0,2)\).

We retain the notation introduced earlier. In particular,
\[
U_i\coloneqq \mathbf x_i^\top \mathbf x_q,\qquad
\pi\coloneqq \arg\max_{1\le i\le N}U_i,
\]
and the exact loss reduction already proved:
\begin{equation}\label{eq:exact-loss-reduction-final}
\mathcal L(\alpha^{(t)},\beta^{(t)})
=
\frac32\,\mathbb E\!\left[a_0(\alpha^{(t)},\beta^{(t)})^2\right]+o(1).
\end{equation}
We also retain the already established divergence
\[
\alpha^{(t)}\to+\infty,
\qquad
\beta^{(t)}\to-\infty,
\qquad
\lambda_t\coloneqq -\frac{\beta^{(t)}}{\alpha^{(t)}}\quad(t\geq1).
\]

To analyze the interior regime \(\lambda_t\in(0,2)\), fix throughout this subsection an arbitrary subsequence \((t_k)_{k\ge1}\) such that
\[
\alpha_k\coloneqq \alpha^{(t_k)}\to+\infty,
\qquad
\beta_k\coloneqq \beta^{(t_k)}\to-\infty,
\qquad
\lambda_k\coloneqq \lambda_{t_k}\to\lambda\in(0,2).
\]
Let
\[
f_d(u)=\kappa_d(1-u^2)^{(d-3)/2}\mathbf 1_{(-1,1)}(u),
\qquad
F_d(u)=\int_{-1}^u f_d(v)\,dv,
\]
so that the density of \(1-U_\pi\) is
\[
h(t)=N\,F_d(1-t)^{N-1}f_d(1-t),\qquad t\in(0,2).
\]

\begin{lemma}
\label{lem:kernel-limits-rigorous-final}
Let \(\psi:[0,2]\to\mathbb R\) be bounded and continuous. Then
\[
\alpha_k\,\mathbb E\!\left[
\frac{e^{\alpha_k(\lambda_k-(1-U_\pi))}}{\bigl(1+e^{\alpha_k(\lambda_k-(1-U_\pi))}\bigr)^2}\,
\psi(1-U_\pi)
\right]
\longrightarrow
h(\lambda)\psi(\lambda),
\]
and
\[
\alpha_k\,\mathbb E\!\left[
\frac{e^{\alpha_k(\lambda_k-(1-U_\pi))}}{\bigl(1+e^{\alpha_k(\lambda_k-(1-U_\pi))}\bigr)^3}\,
\psi(1-U_\pi)
\right]
\longrightarrow
\frac12\,h(\lambda)\psi(\lambda).
\]
\end{lemma}

\begin{proof}
We prove the first limit; the second is identical, with the denominator power \(2\) replaced by \(3\).

Choose \(\delta>0\) such that
\[
[\lambda-2\delta,\lambda+2\delta]\subset(0,2).
\]
Since \(\lambda_k\to\lambda\), after discarding finitely many terms we may assume
\begin{equation}\label{eq:kernel-limits-rigorous-final-lambda-close}
|\lambda_k-\lambda|\le \delta
\qquad\text{for every }k.
\end{equation}
Set
\[
J\coloneqq [\lambda-2\delta,\lambda+2\delta]\subset(0,2).
\]
Decompose
\[
\alpha_k\,\mathbb E\!\left[
\frac{e^{\alpha_k(\lambda_k-(1-U_\pi))}}{\bigl(1+e^{\alpha_k(\lambda_k-(1-U_\pi))}\bigr)^2}\,
\psi(1-U_\pi)
\right]
=
A_k+B_k,
\]
where
\[
A_k\coloneqq
\alpha_k\int_J
\frac{e^{\alpha_k(\lambda_k-t)}}{\bigl(1+e^{\alpha_k(\lambda_k-t)}\bigr)^2}\,
\psi(t)h(t)\,dt,
\]
and
\[
B_k\coloneqq
\alpha_k\int_{[0,2]\setminus J}
\frac{e^{\alpha_k(\lambda_k-t)}}{\bigl(1+e^{\alpha_k(\lambda_k-t)}\bigr)^2}\,
\psi(t)h(t)\,dt.
\]

If \(t\in[0,2]\setminus J\), then by \eqref{eq:kernel-limits-rigorous-final-lambda-close},
\[
|t-\lambda_k|\ge |t-\lambda|-|\lambda_k-\lambda|\ge 2\delta-\delta=\delta.
\]
For every \(x\in\mathbb R\),
\[
0\le \frac{e^x}{(1+e^x)^2}\le e^{-|x|}.
\]
Therefore
\[
\frac{e^{\alpha_k(\lambda_k-t)}}{\bigl(1+e^{\alpha_k(\lambda_k-t)}\bigr)^2}
\le
e^{-\alpha_k|t-\lambda_k|}
\le
e^{-\alpha_k\delta},
\]
and hence
\begin{equation}\label{eq:kernel-limits-rigorous-final-Bk-zero}
|B_k|
\le
\alpha_k e^{-\alpha_k\delta}\|\psi\|_\infty\int_{[0,2]\setminus J} h(t)\,dt
\le
\alpha_k e^{-\alpha_k\delta}\|\psi\|_\infty
\longrightarrow 0.
\end{equation}
On \(J\), the function \(h\psi\) is bounded and continuous. Using the change of variables
\[
z=\alpha_k(t-\lambda_k),
\qquad
t=\lambda_k+\frac{z}{\alpha_k},
\qquad
dt=\frac{dz}{\alpha_k},
\]
we obtain
\[
A_k
=
\int_{\alpha_k(\lambda-2\delta-\lambda_k)}^{\alpha_k(\lambda+2\delta-\lambda_k)}
\frac{e^{-z}}{(1+e^{-z})^2}
\,
\psi\!\left(\lambda_k+\frac{z}{\alpha_k}\right)
\,
h\!\left(\lambda_k+\frac{z}{\alpha_k}\right)
\,dz.
\]
By \eqref{eq:kernel-limits-rigorous-final-lambda-close},
\[
\alpha_k(\lambda-2\delta-\lambda_k)\le -\alpha_k\delta\to-\infty,
\qquad
\alpha_k(\lambda+2\delta-\lambda_k)\ge \alpha_k\delta\to+\infty.
\]
Extend the integrand by zero outside the interval of integration. Since
\[
0\le \frac{e^{-z}}{(1+e^{-z})^2}\le e^{-|z|}
\qquad\text{for all }z\in\mathbb R,
\]
and \(h\psi\) is bounded on \(J\), the extended integrand is dominated by
\[
C_J e^{-|z|}
\]
for some constant \(C_J<\infty\), and \(e^{-|z|}\) is integrable on \(\mathbb R\).

For every fixed \(z\in\mathbb R\), we have
\[
\lambda_k+\frac{z}{\alpha_k}\longrightarrow \lambda,
\]
and for all sufficiently large \(k\) this point lies in \(J\). Hence the extended integrand converges pointwise to
\[
\frac{e^{-z}}{(1+e^{-z})^2}\,\psi(\lambda)\,h(\lambda).
\]
Dominated convergence yields
\[
A_k
\longrightarrow
\psi(\lambda)h(\lambda)\int_{\mathbb R}\frac{e^{-z}}{(1+e^{-z})^2}\,dz.
\]
With the substitution \(u=(1+e^{-z})^{-1}\), we have
\[
du=\frac{e^{-z}}{(1+e^{-z})^2}\,dz,
\]
so
\[
\int_{\mathbb R}\frac{e^{-z}}{(1+e^{-z})^2}\,dz
=
\int_0^1 du
=
1.
\]
Thus
\begin{equation}\label{eq:kernel-limits-rigorous-final-Ak-limit}
A_k\to h(\lambda)\psi(\lambda).
\end{equation}
Combining \eqref{eq:kernel-limits-rigorous-final-Bk-zero} and \eqref{eq:kernel-limits-rigorous-final-Ak-limit} proves the first limit.

For the second limit, write
\[
\alpha_k\,\mathbb E\!\left[
\frac{e^{\alpha_k(\lambda_k-(1-U_\pi))}}{\bigl(1+e^{\alpha_k(\lambda_k-(1-U_\pi))}\bigr)^3}\,
\psi(1-U_\pi)
\right]
=
A_k'+B_k',
\]
where
\[
A_k'\coloneqq
\alpha_k\int_J
\frac{e^{\alpha_k(\lambda_k-t)}}{\bigl(1+e^{\alpha_k(\lambda_k-t)}\bigr)^3}\,
\psi(t)h(t)\,dt,
\]
and
\[
B_k'\coloneqq
\alpha_k\int_{[0,2]\setminus J}
\frac{e^{\alpha_k(\lambda_k-t)}}{\bigl(1+e^{\alpha_k(\lambda_k-t)}\bigr)^3}\,
\psi(t)h(t)\,dt.
\]
Since
\[
0\le \frac{e^{\alpha_k(\lambda_k-t)}}{\bigl(1+e^{\alpha_k(\lambda_k-t)}\bigr)^3}
\le
\frac{e^{\alpha_k(\lambda_k-t)}}{\bigl(1+e^{\alpha_k(\lambda_k-t)}\bigr)^2},
\]
the estimate for \(B_k\) gives \(B_k'\to0\). On \(J\), the same change of variables yields
\[
A_k'
=
\int_{\alpha_k(\lambda-2\delta-\lambda_k)}^{\alpha_k(\lambda+2\delta-\lambda_k)}
\frac{e^{-z}}{(1+e^{-z})^3}
\,
\psi\!\left(\lambda_k+\frac{z}{\alpha_k}\right)
\,
h\!\left(\lambda_k+\frac{z}{\alpha_k}\right)
\,dz.
\]
The same dominated convergence argument gives
\[
A_k'
\longrightarrow
\psi(\lambda)h(\lambda)\int_{\mathbb R}\frac{e^{-z}}{(1+e^{-z})^3}\,dz.
\]
Again with \(u=(1+e^{-z})^{-1}\),
\[
\frac{e^{-z}}{(1+e^{-z})^3}\,dz=u\,du,
\]
so
\[
\int_{\mathbb R}\frac{e^{-z}}{(1+e^{-z})^3}\,dz
=
\int_0^1 u\,du
=
\frac12.
\]
Hence
\[
\alpha_k\,\mathbb E\!\left[
\frac{e^{\alpha_k(\lambda_k-(1-U_\pi))}}{\bigl(1+e^{\alpha_k(\lambda_k-(1-U_\pi))}\bigr)^3}\,
\psi(1-U_\pi)
\right]
\longrightarrow
\frac12\,h(\lambda)\psi(\lambda).
\]
\end{proof}

\begin{lemma}
\label{lem:conditional-R-formula-final}
Let \(g:[0,2]\to[0,\infty)\) be Borel measurable. Then, for almost every \(u\in(-1,1)\),
\[
\mathbb E\!\left[\sum_{j\neq \pi} g(U_\pi-U_j)\,\middle|\, U_\pi=u\right]
=
\frac{N-1}{F_d(u)}\int_{-1}^u g(u-v)f_d(v)\,dv.
\]
In particular, for almost every \(u\in(-1,1)\),
\[
\mathbb E\!\left[\sum_{j\neq \pi} e^{-\alpha_k(U_\pi-U_j)}\,\middle|\, U_\pi=u\right]
=
\frac{N-1}{F_d(u)}\int_{-1}^u e^{-\alpha_k(u-v)}f_d(v)\,dv.
\]
\end{lemma}

\begin{proof}
Fix a Borel set \(B\subset(-1,1)\). Since ties occur with probability zero, on the event \(\{\pi=i\}\) we have \(U_\pi=U_i\) and \(U_j<U_i\) for every \(j\neq i\). By symmetry of the i.i.d. variables \(U_1,\dots,U_N\),
\begin{align*}
&\mathbb E\!\left[\sum_{j\neq \pi} g(U_\pi-U_j)\,\mathbf 1_{\{U_\pi\in B\}}\right] \\
&=
\sum_{i=1}^N\sum_{j\neq i}
\mathbb E\!\left[g(U_i-U_j)\mathbf 1_{\{U_i\in B,\ U_j<U_i,\ U_\ell<U_i\ \forall \ell\neq i\}}\right] \\
&=
N(N-1)
\int_B\int_{-1}^u
g(u-v)\,f_d(v)\,F_d(u)^{N-2}\,f_d(u)\,dv\,du.
\end{align*}
The density of \(U_\pi\) is
\[
u\mapsto N\,F_d(u)^{N-1}f_d(u),\qquad u\in(-1,1).
\]
Hence
\[
\mathbb E\!\left[\sum_{j\neq \pi} g(U_\pi-U_j)\,\mathbf 1_{\{U_\pi\in B\}}\right]
=
\int_B
\left[
\frac{N-1}{F_d(u)}\int_{-1}^u g(u-v)f_d(v)\,dv
\right]
N\,F_d(u)^{N-1}f_d(u)\,du.
\]
By the uniqueness of conditional expectation with respect to the distribution of \(U_\pi\), this proves the claim.
\end{proof}

\begin{lemma}
\label{lem:R-local-rigorous-final}
Choose \(\delta>0\) so that
\[
[\lambda-2\delta,\lambda+2\delta]\subset(0,2).
\]
Then there exists a constant \(C_\delta<\infty\) such that, for all sufficiently large \(k\),
\[
\mathbb E\!\left[\sum_{j\neq \pi} e^{-\alpha_k(U_\pi-U_j)}\,\middle|\, 1-U_\pi\right]\le \frac{C_\delta}{\alpha_k}
\qquad\text{on the event }\{|1-U_\pi-\lambda|\le \delta\}.
\]
Moreover,
\[
\alpha_k\,\mathbb E\!\left[
\left(\sum_{j\neq \pi} e^{-\alpha_k(U_\pi-U_j)}\right)
\frac{e^{\alpha_k(\lambda_k-(1-U_\pi))}}{\bigl(1+e^{\alpha_k(\lambda_k-(1-U_\pi))}\bigr)^2}
\right]\to0,
\]
and
\[
\alpha_k\,\mathbb E\!\left[
\left(\sum_{j\neq \pi} e^{-\alpha_k(U_\pi-U_j)}\right)
\frac{e^{\alpha_k(\lambda_k-(1-U_\pi))}}{\bigl(1+e^{\alpha_k(\lambda_k-(1-U_\pi))}\bigr)^3}
\right]\to0.
\]
\end{lemma}

\begin{proof}
Set
\[
J_\delta\coloneqq [1-\lambda-\delta,\ 1-\lambda+\delta]\subset(-1,1).
\]
On the event \(\{|1-U_\pi-\lambda|\le\delta\}\), we have
\[
U_\pi\in J_\delta.
\]
Because \(J_\delta\Subset(-1,1)\), the continuous function \(F_d\) satisfies
\begin{equation}\label{eq:R-local-rigorous-final-F-lower}
\inf_{u\in J_\delta}F_d(u)\ge c_\delta>0.
\end{equation}
Also, since \(J_\delta\) is compact and contained in \((-1,1)\), there exists \(\eta_\delta>0\) such that
\[
[u-\eta_\delta,u]\subset K_\delta
\qquad\text{for every }u\in J_\delta,
\]
where
\[
K_\delta\coloneqq [\,1-\lambda-\delta-\eta_\delta,\ 1-\lambda+\delta\,]\Subset(-1,1).
\]
Hence
\begin{equation}\label{eq:R-local-rigorous-final-f-upper}
\sup_{v\in K_\delta}f_d(v)\le C_\delta'<\infty.
\end{equation}
By Lemma~\ref{lem:conditional-R-formula-final}, for almost every \(u\in J_\delta\),
\[
\mathbb E\!\left[\sum_{j\neq \pi} e^{-\alpha_k(U_\pi-U_j)}\,\middle|\, U_\pi=u\right]
=
\frac{N-1}{F_d(u)}\int_{-1}^u e^{-\alpha_k(u-v)}f_d(v)\,dv.
\]
Split the integral at \(u-\eta_\delta\):
\[
\int_{-1}^u e^{-\alpha_k(u-v)}f_d(v)\,dv
=
\int_{-1}^{u-\eta_\delta} e^{-\alpha_k(u-v)}f_d(v)\,dv
+
\int_{u-\eta_\delta}^{u} e^{-\alpha_k(u-v)}f_d(v)\,dv.
\]
For the first term, \(u-v\ge \eta_\delta\), so
\begin{equation}\label{eq:R-local-rigorous-final-far-part}
\int_{-1}^{u-\eta_\delta} e^{-\alpha_k(u-v)}f_d(v)\,dv
\le
e^{-\alpha_k\eta_\delta}\int_{-1}^{u-\eta_\delta} f_d(v)\,dv
\le
e^{-\alpha_k\eta_\delta}.
\end{equation}
For the second term, \([u-\eta_\delta,u]\subset K_\delta\), so by \eqref{eq:R-local-rigorous-final-f-upper},
\begin{equation}\label{eq:R-local-rigorous-final-near-part}
\int_{u-\eta_\delta}^{u} e^{-\alpha_k(u-v)}f_d(v)\,dv
\le
C_\delta'\int_{0}^{\eta_\delta} e^{-\alpha_k s}\,ds
\le
\frac{C_\delta'}{\alpha_k}.
\end{equation}
Combining \eqref{eq:R-local-rigorous-final-F-lower}, \eqref{eq:R-local-rigorous-final-far-part}, and \eqref{eq:R-local-rigorous-final-near-part}, for almost every \(u\in J_\delta\),
\[
\mathbb E\!\left[\sum_{j\neq \pi} e^{-\alpha_k(U_\pi-U_j)}\,\middle|\, U_\pi=u\right]
\le
\frac{N-1}{c_\delta}
\left(
e^{-\alpha_k\eta_\delta}+\frac{C_\delta'}{\alpha_k}
\right).
\]
Since \(\alpha_k\to\infty\), the exponential term is eventually bounded by \(1/\alpha_k\). Hence there exists \(C_\delta<\infty\) such that, for all sufficiently large \(k\),
\[
\mathbb E\!\left[\sum_{j\neq \pi} e^{-\alpha_k(U_\pi-U_j)}\,\middle|\, U_\pi=u\right]\le \frac{C_\delta}{\alpha_k}
\qquad\text{for almost every }u\in J_\delta.
\]
Because \(1-U_\pi\) is a measurable function of \(U_\pi\), this proves
\begin{equation}\label{eq:R-local-rigorous-final-cond-bound}
\mathbb E\!\left[\sum_{j\neq \pi} e^{-\alpha_k(U_\pi-U_j)}\,\middle|\, 1-U_\pi\right]\le \frac{C_\delta}{\alpha_k}
\qquad\text{on }\{|1-U_\pi-\lambda|\le\delta\}.
\end{equation}
Let
\[
A_\delta\coloneqq \{|1-U_\pi-\lambda|\le\delta\},
\qquad
A_\delta^c\coloneqq \{|1-U_\pi-\lambda|>\delta\}.
\]
On \(A_\delta\), using \eqref{eq:R-local-rigorous-final-cond-bound},
\begin{align*}
&\alpha_k\,\mathbb E\!\left[
\left(\sum_{j\neq \pi} e^{-\alpha_k(U_\pi-U_j)}\right)
\frac{e^{\alpha_k(\lambda_k-(1-U_\pi))}}{\bigl(1+e^{\alpha_k(\lambda_k-(1-U_\pi))}\bigr)^2}
\mathbf 1_{A_\delta}
\right] \\
&=
\alpha_k\,\mathbb E\!\Biggl[
\mathbb E\!\left[\sum_{j\neq \pi} e^{-\alpha_k(U_\pi-U_j)}\,\middle|\, 1-U_\pi\right]
\frac{e^{\alpha_k(\lambda_k-(1-U_\pi))}}{\bigl(1+e^{\alpha_k(\lambda_k-(1-U_\pi))}\bigr)^2}
\mathbf 1_{A_\delta}
\Biggr] \\
&\le
C_\delta\,
\mathbb E\!\left[
\frac{e^{\alpha_k(\lambda_k-(1-U_\pi))}}{\bigl(1+e^{\alpha_k(\lambda_k-(1-U_\pi))}\bigr)^2}
\right].
\end{align*}
By Lemma~\ref{lem:kernel-limits-rigorous-final} with \(\psi\equiv1\),
\[
\alpha_k\,\mathbb E\!\left[
\frac{e^{\alpha_k(\lambda_k-(1-U_\pi))}}{\bigl(1+e^{\alpha_k(\lambda_k-(1-U_\pi))}\bigr)^2}
\right]\to h(\lambda),
\]
so the expectation is \(O(\alpha_k^{-1})\). Therefore
\begin{equation}\label{eq:R-local-rigorous-final-RL-local}
\alpha_k\,\mathbb E\!\left[
\left(\sum_{j\neq \pi} e^{-\alpha_k(U_\pi-U_j)}\right)
\frac{e^{\alpha_k(\lambda_k-(1-U_\pi))}}{\bigl(1+e^{\alpha_k(\lambda_k-(1-U_\pi))}\bigr)^2}
\mathbf 1_{A_\delta}
\right]\to0.
\end{equation}
Similarly, since
\[
0\le
\frac{e^{\alpha_k(\lambda_k-(1-U_\pi))}}{\bigl(1+e^{\alpha_k(\lambda_k-(1-U_\pi))}\bigr)^3}
\le
\frac{e^{\alpha_k(\lambda_k-(1-U_\pi))}}{\bigl(1+e^{\alpha_k(\lambda_k-(1-U_\pi))}\bigr)^2},
\]
we obtain
\begin{equation}\label{eq:R-local-rigorous-final-RK-local}
\alpha_k\,\mathbb E\!\left[
\left(\sum_{j\neq \pi} e^{-\alpha_k(U_\pi-U_j)}\right)
\frac{e^{\alpha_k(\lambda_k-(1-U_\pi))}}{\bigl(1+e^{\alpha_k(\lambda_k-(1-U_\pi))}\bigr)^3}
\mathbf 1_{A_\delta}
\right]\to0.
\end{equation}
Since \(\lambda_k\to\lambda\), for all sufficiently large \(k\),
\begin{equation}\label{eq:R-local-rigorous-final-lambdak-close}
|\lambda_k-\lambda|\le \frac{\delta}{2}.
\end{equation}
If \(|1-U_\pi-\lambda|>\delta\), then by \eqref{eq:R-local-rigorous-final-lambdak-close},
\[
|1-U_\pi-\lambda_k|
\ge
|1-U_\pi-\lambda|-|\lambda_k-\lambda|
>
\frac{\delta}{2}.
\]
For every \(x\in\mathbb R\),
\[
0\le \frac{e^x}{(1+e^x)^2}\le e^{-|x|},
\qquad
0\le \frac{e^x}{(1+e^x)^3}\le e^{-|x|}.
\]
Thus, on \(A_\delta^c\),
\[
\frac{e^{\alpha_k(\lambda_k-(1-U_\pi))}}{\bigl(1+e^{\alpha_k(\lambda_k-(1-U_\pi))}\bigr)^2}
\le e^{-\alpha_k\delta/2},
\qquad
\frac{e^{\alpha_k(\lambda_k-(1-U_\pi))}}{\bigl(1+e^{\alpha_k(\lambda_k-(1-U_\pi))}\bigr)^3}
\le e^{-\alpha_k\delta/2}
\]
for all large \(k\). Since
\[
0\le \sum_{j\neq \pi} e^{-\alpha_k(U_\pi-U_j)}\le N-1,
\]
we have
\begin{equation}\label{eq:R-local-rigorous-final-RL-tail}
\alpha_k\,\mathbb E\!\left[
\left(\sum_{j\neq \pi} e^{-\alpha_k(U_\pi-U_j)}\right)
\frac{e^{\alpha_k(\lambda_k-(1-U_\pi))}}{\bigl(1+e^{\alpha_k(\lambda_k-(1-U_\pi))}\bigr)^2}
\mathbf 1_{A_\delta^c}
\right]
\le
\alpha_k(N-1)e^{-\alpha_k\delta/2}\to0,
\end{equation}
and likewise
\begin{equation}\label{eq:R-local-rigorous-final-RK-tail}
\alpha_k\,\mathbb E\!\left[
\left(\sum_{j\neq \pi} e^{-\alpha_k(U_\pi-U_j)}\right)
\frac{e^{\alpha_k(\lambda_k-(1-U_\pi))}}{\bigl(1+e^{\alpha_k(\lambda_k-(1-U_\pi))}\bigr)^3}
\mathbf 1_{A_\delta^c}
\right]
\le
\alpha_k(N-1)e^{-\alpha_k\delta/2}\to0.
\end{equation}
Combining \eqref{eq:R-local-rigorous-final-RL-local} with \eqref{eq:R-local-rigorous-final-RL-tail}, and \eqref{eq:R-local-rigorous-final-RK-local} with \eqref{eq:R-local-rigorous-final-RK-tail}, proves the claim.
\end{proof}

\begin{lemma}
\label{lem:S-global-rigorous-final}
Then
\[
\alpha_k\,\mathbb E\!\left[\sum_{j\neq\pi}(U_\pi-U_j)e^{-\alpha_k(U_\pi-U_j)}\right]\longrightarrow 0.
\]
\end{lemma}

\begin{proof}
For \(j\neq \pi\),
\[
(U_\pi-U_j)e^{-\alpha_k(U_\pi-U_j)}
=
(U_\pi-U_j)_+e^{-\alpha_k(U_\pi-U_j)_+}.
\]
Since \(\pi\in[N]\),
\[
(U_\pi-U_j)_+e^{-\alpha_k(U_\pi-U_j)_+}
\le
\sum_{i\neq j}(U_i-U_j)_+e^{-\alpha_k(U_i-U_j)_+}.
\]
Summing over \(j\neq\pi\),
\[
\sum_{j\neq\pi}(U_\pi-U_j)e^{-\alpha_k(U_\pi-U_j)}
\le
\sum_{j=1}^N\sum_{i\neq j}(U_i-U_j)_+e^{-\alpha_k(U_i-U_j)_+}.
\]
Taking expectation and using exchangeability,
\begin{equation}\label{eq:S-global-rigorous-final-ES-bound}
\mathbb E\!\left[\sum_{j\neq\pi}(U_\pi-U_j)e^{-\alpha_k(U_\pi-U_j)}\right]
\le
N(N-1)\,I_d(\alpha_k),
\end{equation}
where
\[
I_d(\alpha)\coloneqq 
\mathbb E\!\left[(U_1-U_2)_+e^{-\alpha(U_1-U_2)_+}\right].
\]
Since \(U_1,U_2\) are i.i.d. with density \(f_d\),
\begin{align*}
I_d(\alpha)
&=
\int_{-1}^1\int_{-1}^1 (u-v)_+e^{-\alpha(u-v)_+}f_d(u)f_d(v)\,du\,dv \\
&=
\int_0^2 s e^{-\alpha s}
\left(\int_{-1}^{1-s} f_d(u+s)f_d(u)\,du\right)\,ds.
\end{align*}
Define
\[
J_d(s)\coloneqq \int_{-1}^{1-s} f_d(u+s)f_d(u)\,du,
\qquad 0<s<2.
\]
Then
\begin{equation}\label{eq:S-global-rigorous-final-Id-repr}
I_d(\alpha)=\int_0^2 s e^{-\alpha s}J_d(s)\,ds.
\end{equation}
If \(d\ge3\), then \(f_d\) is bounded on \((-1,1)\), hence
\[
J_d(s)\le 2\|f_d\|_\infty^2
\qquad (0<s<2),
\]
and therefore
\[
I_d(\alpha)\le C\int_0^2 s e^{-\alpha s}\,ds\le \frac{C}{\alpha^2}.
\]
Thus
\begin{equation}\label{eq:S-global-rigorous-final-dge3}
\alpha I_d(\alpha)\to0
\qquad (d\ge3).
\end{equation}
It remains to treat \(d=2\). Then
\[
f_2(u)=\frac{1}{\pi\sqrt{1-u^2}}\mathbf 1_{(-1,1)}(u).
\]
For \(0<s\le1\),
\[
J_2(s)
=
\frac{1}{\pi^2}\int_{-1}^{1-s}\frac{du}{\sqrt{1-u^2}\sqrt{1-(u+s)^2}}.
\]
Set \(v=1+u\). Then \(v\in[0,2-s]\), and
\[
1-u^2=v(2-v),
\qquad
1-(u+s)^2=(v+s)(2-s-v).
\]
Hence
\begin{equation}\label{eq:S-global-rigorous-final-J2-repr}
J_2(s)
=
\frac{1}{\pi^2}\int_0^{2-s}
\frac{dv}{\sqrt{v(v+s)(2-v)(2-s-v)}}.
\end{equation}
Let
\[
m_s\coloneqq 1-\frac{s}{2}\in\left[\frac12,1\right].
\]
Split the integral in \eqref{eq:S-global-rigorous-final-J2-repr} into \([0,m_s]\) and \([m_s,2-s]\).
If \(0\le v\le m_s\), then
\[
2-v\ge 1,
\qquad
2-s-v\ge 1-\frac{s}{2}\ge \frac12,
\]
so
\[
\frac{1}{\sqrt{v(v+s)(2-v)(2-s-v)}}
\le
\frac{\sqrt2}{\sqrt{v(v+s)}}.
\]
Hence
\begin{equation}\label{eq:S-global-rigorous-final-J2-left}
\int_0^{m_s}
\frac{dv}{\sqrt{v(v+s)(2-v)(2-s-v)}}
\le
\sqrt2\int_0^1 \frac{dv}{\sqrt{v(v+s)}}.
\end{equation}
For the right half, set \(w=2-s-v\). Then \(w\in[0,m_s]\), \(v=2-s-w\), and
\[
\frac{1}{\sqrt{v(v+s)(2-v)(2-s-v)}}
=
\frac{1}{\sqrt{(2-s-w)(2-w)(w+s)w}}.
\]
Since \(w\le m_s\),
\[
2-s-w\ge 1-\frac{s}{2}\ge \frac12,
\qquad
2-w\ge 1,
\]
thus
\[
\frac{1}{\sqrt{(2-s-w)(2-w)(w+s)w}}
\le
\frac{\sqrt2}{\sqrt{w(w+s)}}.
\]
Therefore
\begin{equation}\label{eq:S-global-rigorous-final-J2-right}
\int_{m_s}^{2-s}
\frac{dv}{\sqrt{v(v+s)(2-v)(2-s-v)}}
\le
\sqrt2\int_0^1 \frac{dw}{\sqrt{w(w+s)}}.
\end{equation}
Combining \eqref{eq:S-global-rigorous-final-J2-repr}, \eqref{eq:S-global-rigorous-final-J2-left}, and \eqref{eq:S-global-rigorous-final-J2-right},
\[
J_2(s)\le C\int_0^1 \frac{dy}{\sqrt{y(y+s)}}
\qquad (0<s\le1).
\]
The antiderivative is explicit:
\[
\int_0^1 \frac{dy}{\sqrt{y(y+s)}}
=
2\,\operatorname{arsinh}\!\left(\frac{1}{\sqrt s}\right).
\]
Since \(\operatorname{arsinh}(x)\le C(1+\log x)\) for \(x\ge1\),
\begin{equation}\label{eq:S-global-rigorous-final-J2-log}
J_2(s)\le C\bigl(1+\log(1/s)\bigr),
\qquad 0<s\le1.
\end{equation}
For every \(d\ge2\), \(J_d\) is integrable on \((0,2)\), because by Fubini
\begin{equation}\label{eq:S-global-rigorous-final-Jd-int}
\int_0^2 J_d(s)\,ds
=
\int_0^2\int_{-1}^{1-s} f_d(u+s)f_d(u)\,du\,ds
=
\int_{-1}^1\int_u^1 f_d(v)f_d(u)\,dv\,du
\le 1.
\end{equation}
Therefore
\begin{equation}\label{eq:S-global-rigorous-final-tail}
\int_1^2 s e^{-\alpha s}J_d(s)\,ds
\le
2e^{-\alpha}\int_1^2 J_d(s)\,ds
\le
2e^{-\alpha}.
\end{equation}

Using \eqref{eq:S-global-rigorous-final-J2-log} and \eqref{eq:S-global-rigorous-final-tail},
\[
I_2(\alpha)
\le
C\int_0^1 s e^{-\alpha s}\bigl(1+\log(1/s)\bigr)\,ds
+
2e^{-\alpha}.
\]
Hence
\[
\alpha I_2(\alpha)
\le
C\alpha\int_0^1 s e^{-\alpha s}\bigl(1+\log(1/s)\bigr)\,ds
+
2\alpha e^{-\alpha}.
\]
With the substitution \(y=\alpha s\),
\begin{align*}
\alpha\int_0^1 s e^{-\alpha s}\bigl(1+\log(1/s)\bigr)\,ds
&=
\frac1\alpha\int_0^\alpha y e^{-y}\bigl(1+\log(\alpha/y)\bigr)\,dy \\
&\le
\frac1\alpha\int_0^\infty y e^{-y}\bigl(1+\log\alpha+|\log y|\bigr)\,dy.
\end{align*}
Since
\[
\int_0^\infty y e^{-y}\,dy<\infty,
\qquad
\int_0^\infty y e^{-y}|\log y|\,dy<\infty,
\]
it follows that
\begin{equation}\label{eq:S-global-rigorous-final-d2}
\alpha I_2(\alpha)\le C\frac{1+\log\alpha}{\alpha}+2\alpha e^{-\alpha}\to0.
\end{equation}
Combining \eqref{eq:S-global-rigorous-final-dge3} and \eqref{eq:S-global-rigorous-final-d2},
\[
\alpha I_d(\alpha)\to0
\qquad\text{for every }d\ge2.
\]
Returning to \eqref{eq:S-global-rigorous-final-ES-bound},
\[
\alpha_k\,\mathbb E\!\left[\sum_{j\neq\pi}(U_\pi-U_j)e^{-\alpha_k(U_\pi-U_j)}\right]
\le
N(N-1)\,\alpha_k I_d(\alpha_k)\to0.
\]
\end{proof}

\begin{lemma}
\label{lem:error-bounds-final}
Define
\[
b_{i,k}\coloneqq \frac{e^{\alpha_kU_i}}{\sum_{j=1}^N e^{\alpha_kU_j}},
\qquad
Q_k\coloneqq \sum_{i=1}^N b_{i,k}^2,
\qquad
\mu_k\coloneqq \sum_{i=1}^N b_{i,k}U_i.
\]
Then
\[
1-b_{\pi,k}\le \sum_{j\neq\pi} e^{-\alpha_k(U_\pi-U_j)},
\]
\[
0\le b_{\pi,k}-Q_k\le 1-b_{\pi,k},
\]
\[
1-Q_k\le 2(1-b_{\pi,k})\le 2\sum_{j\neq\pi} e^{-\alpha_k(U_\pi-U_j)},
\]
\[
0\le U_\pi-\mu_k\le \sum_{j\neq\pi}(U_\pi-U_j)e^{-\alpha_k(U_\pi-U_j)},
\]
\[
\bigl|\partial_\alpha b_\pi(\alpha_k)\bigr|\le \sum_{j\neq\pi}(U_\pi-U_j)e^{-\alpha_k(U_\pi-U_j)},
\]
\[
|Q_k'|\le 4\sum_{j\neq\pi}(U_\pi-U_j)e^{-\alpha_k(U_\pi-U_j)},
\]
where
\[
\partial_\alpha b_\pi(\alpha_k)=b_{\pi,k}(U_\pi-\mu_k),
\qquad
Q_k'=2\sum_{i=1}^N b_{i,k}^2(U_i-\mu_k).
\]
Moreover,
\[
\left|\,a_0(\alpha_k,\beta_k)+2b_{\pi,k}-2\bigl(1-a_0(\alpha_k,\beta_k)\bigr)Q_k-3a_0(\alpha_k,\beta_k)\,\right|
\le
6\sum_{j\neq\pi} e^{-\alpha_k(U_\pi-U_j)}.
\]
\end{lemma}

\begin{proof}
Since
\[
b_{\pi,k}=\frac{1}{1+\sum_{j\neq\pi} e^{-\alpha_k(U_\pi-U_j)}},
\]
we have
\[
1-b_{\pi,k}
=
\frac{\sum_{j\neq\pi} e^{-\alpha_k(U_\pi-U_j)}}{1+\sum_{j\neq\pi} e^{-\alpha_k(U_\pi-U_j)}}
\le
\sum_{j\neq\pi} e^{-\alpha_k(U_\pi-U_j)}.
\]
Also \(Q_k\ge b_{\pi,k}^2\), so
\[
0\le b_{\pi,k}-Q_k\le b_{\pi,k}-b_{\pi,k}^2=b_{\pi,k}(1-b_{\pi,k})\le 1-b_{\pi,k}.
\]
Therefore
\[
1-Q_k=(1-b_{\pi,k}^2)-\sum_{j\neq\pi}b_{j,k}^2\le 1-b_{\pi,k}^2\le 2(1-b_{\pi,k}),
\]
which gives the third inequality.

Next,
\[
U_\pi-\mu_k=\sum_{j\neq\pi}b_{j,k}(U_\pi-U_j).
\]
For \(j\neq\pi\),
\[
b_{j,k}
=
\frac{e^{-\alpha_k(U_\pi-U_j)}}{1+\sum_{\ell\neq\pi}e^{-\alpha_k(U_\pi-U_\ell)}}
\le
e^{-\alpha_k(U_\pi-U_j)},
\]
hence
\[
0\le U_\pi-\mu_k\le \sum_{j\neq\pi}(U_\pi-U_j)e^{-\alpha_k(U_\pi-U_j)}.
\]
Using
\[
\partial_\alpha b_\pi(\alpha_k)=b_{\pi,k}(U_\pi-\mu_k),
\]
we get
\[
\bigl|\partial_\alpha b_\pi(\alpha_k)\bigr|
\le
\sum_{j\neq\pi}(U_\pi-U_j)e^{-\alpha_k(U_\pi-U_j)}.
\]
For \(Q_k'\),
\[
Q_k'=2\sum_{i=1}^N b_{i,k}^2(U_i-\mu_k),
\]
hence
\[
|Q_k'|
\le
2\sum_{i=1}^N b_{i,k}^2|U_i-\mu_k|
\le
2\sum_{i=1}^N b_{i,k}|U_i-\mu_k|.
\]
Since \(U_i\le U_\pi\),
\[
|U_i-\mu_k|\le (U_\pi-U_i)+(U_\pi-\mu_k).
\]
Therefore
\[
\sum_{i=1}^N b_{i,k}|U_i-\mu_k|
\le
\sum_{i=1}^N b_{i,k}(U_\pi-U_i)+\sum_{i=1}^N b_{i,k}(U_\pi-\mu_k)
=
2(U_\pi-\mu_k),
\]
and thus
\[
|Q_k'|\le 4(U_\pi-\mu_k)\le 4\sum_{j\neq\pi}(U_\pi-U_j)e^{-\alpha_k(U_\pi-U_j)}.
\]

Finally,
\[
a_0(\alpha_k,\beta_k)+2b_{\pi,k}-2\bigl(1-a_0(\alpha_k,\beta_k)\bigr)Q_k-3a_0(\alpha_k,\beta_k)
=
2(b_{\pi,k}-Q_k)+2a_0(\alpha_k,\beta_k)(Q_k-1).
\]
Hence
\begin{align*}
&\left|\,a_0(\alpha_k,\beta_k)+2b_{\pi,k}-2\bigl(1-a_0(\alpha_k,\beta_k)\bigr)Q_k-3a_0(\alpha_k,\beta_k)\,\right| \\
&\le
2|b_{\pi,k}-Q_k|+2a_0(\alpha_k,\beta_k)|1-Q_k|
\le
2(1-b_{\pi,k})+4(1-b_{\pi,k}) \\
&=
6(1-b_{\pi,k})
\le
6\sum_{j\neq\pi} e^{-\alpha_k(U_\pi-U_j)}.
\end{align*}
\end{proof}

\begin{lemma}
\label{lem:c-replacement-final}
Let
\[
a_0(\alpha_k,\beta_k)
=
\left(
1+e^{\alpha_k(\lambda_k-(1-U_\pi))}\left(1+\sum_{j\neq\pi} e^{-\alpha_k(U_\pi-U_j)}\right)
\right)^{-1}.
\]
Then
\[
\alpha_k\,\mathbb E\!\left[
\left|
a_0(\alpha_k,\beta_k)\bigl(1-a_0(\alpha_k,\beta_k)\bigr)
-
\frac{e^{\alpha_k(\lambda_k-(1-U_\pi))}}{\bigl(1+e^{\alpha_k(\lambda_k-(1-U_\pi))}\bigr)^2}
\right|
\right]\to0,
\]
and
\[
\alpha_k\,\mathbb E\!\left[
\left|
a_0(\alpha_k,\beta_k)^2\bigl(1-a_0(\alpha_k,\beta_k)\bigr)
-
\frac{e^{\alpha_k(\lambda_k-(1-U_\pi))}}{\bigl(1+e^{\alpha_k(\lambda_k-(1-U_\pi))}\bigr)^3}
\right|
\right]\to0.
\]
Moreover, the pointwise bounds
\[
\left|
a_0(\alpha_k,\beta_k)\bigl(1-a_0(\alpha_k,\beta_k)\bigr)
-
\frac{e^{\alpha_k(\lambda_k-(1-U_\pi))}}{\bigl(1+e^{\alpha_k(\lambda_k-(1-U_\pi))}\bigr)^2}
\right|
\le
\left(\sum_{j\neq\pi} e^{-\alpha_k(U_\pi-U_j)}\right)
\frac{e^{\alpha_k(\lambda_k-(1-U_\pi))}}{\bigl(1+e^{\alpha_k(\lambda_k-(1-U_\pi))}\bigr)^2},
\]
and
\[
\left|
a_0(\alpha_k,\beta_k)^2\bigl(1-a_0(\alpha_k,\beta_k)\bigr)
-
\frac{e^{\alpha_k(\lambda_k-(1-U_\pi))}}{\bigl(1+e^{\alpha_k(\lambda_k-(1-U_\pi))}\bigr)^3}
\right|
\le
2\left(\sum_{j\neq\pi} e^{-\alpha_k(U_\pi-U_j)}\right)
\frac{e^{\alpha_k(\lambda_k-(1-U_\pi))}}{\bigl(1+e^{\alpha_k(\lambda_k-(1-U_\pi))}\bigr)^3}
\]
hold almost surely.
\end{lemma}

\begin{proof}
Set
\[
A\coloneqq e^{\alpha_k(\lambda_k-(1-U_\pi))},
\qquad
x\coloneqq 1+\sum_{j\neq\pi} e^{-\alpha_k(U_\pi-U_j)}\ge1.
\]
Then
\[
a_0(\alpha_k,\beta_k)\bigl(1-a_0(\alpha_k,\beta_k)\bigr)=\ell(A,x),
\qquad
\frac{e^{\alpha_k(\lambda_k-(1-U_\pi))}}{\bigl(1+e^{\alpha_k(\lambda_k-(1-U_\pi))}\bigr)^2}=\ell(A,1),
\qquad
\ell(A,x)\coloneqq \frac{Ax}{(1+Ax)^2},
\]
and
\[
a_0(\alpha_k,\beta_k)^2\bigl(1-a_0(\alpha_k,\beta_k)\bigr)=\kappa(A,x),
\qquad
\frac{e^{\alpha_k(\lambda_k-(1-U_\pi))}}{\bigl(1+e^{\alpha_k(\lambda_k-(1-U_\pi))}\bigr)^3}=\kappa(A,1),
\qquad
\kappa(A,x)\coloneqq \frac{Ax}{(1+Ax)^3}.
\]
Differentiate with respect to \(x\):
\[
\partial_x\ell(A,x)=\frac{A(1-Ax)}{(1+Ax)^3},
\qquad
\partial_x\kappa(A,x)=\frac{A(1-2Ax)}{(1+Ax)^4}.
\]
For \(x\ge1\),
\[
|\partial_x\ell(A,x)|
\le
\frac{A(1+Ax)}{(1+Ax)^3}
=
\frac{A}{(1+Ax)^2}
\le
\frac{A}{(1+A)^2}
=
\frac{e^{\alpha_k(\lambda_k-(1-U_\pi))}}{\bigl(1+e^{\alpha_k(\lambda_k-(1-U_\pi))}\bigr)^2},
\]
and similarly
\[
|\partial_x\kappa(A,x)|
\le
\frac{A(1+2Ax)}{(1+Ax)^4}
\le
\frac{2A}{(1+Ax)^3}
\le
\frac{2A}{(1+A)^3}
=
2\frac{e^{\alpha_k(\lambda_k-(1-U_\pi))}}{\bigl(1+e^{\alpha_k(\lambda_k-(1-U_\pi))}\bigr)^3}.
\]
Hence, by the mean-value theorem,
\begin{align*}
&\left|
a_0(\alpha_k,\beta_k)\bigl(1-a_0(\alpha_k,\beta_k)\bigr)
-
\frac{e^{\alpha_k(\lambda_k-(1-U_\pi))}}{\bigl(1+e^{\alpha_k(\lambda_k-(1-U_\pi))}\bigr)^2}
\right| \\
&=
|\ell(A,x)-\ell(A,1)| \\
&\le
\left(\sum_{j\neq\pi} e^{-\alpha_k(U_\pi-U_j)}\right)
\frac{e^{\alpha_k(\lambda_k-(1-U_\pi))}}{\bigl(1+e^{\alpha_k(\lambda_k-(1-U_\pi))}\bigr)^2},
\end{align*}
and
\begin{align*}
&\left|
a_0(\alpha_k,\beta_k)^2\bigl(1-a_0(\alpha_k,\beta_k)\bigr)
-
\frac{e^{\alpha_k(\lambda_k-(1-U_\pi))}}{\bigl(1+e^{\alpha_k(\lambda_k-(1-U_\pi))}\bigr)^3}
\right| \\
&=
|\kappa(A,x)-\kappa(A,1)| \\
&\le
2\left(\sum_{j\neq\pi} e^{-\alpha_k(U_\pi-U_j)}\right)
\frac{e^{\alpha_k(\lambda_k-(1-U_\pi))}}{\bigl(1+e^{\alpha_k(\lambda_k-(1-U_\pi))}\bigr)^3}.
\end{align*}
Taking expectations and using Lemma~\ref{lem:R-local-rigorous-final},
\begin{align*}
&\alpha_k\,\mathbb E\!\left[
\left|
a_0(\alpha_k,\beta_k)\bigl(1-a_0(\alpha_k,\beta_k)\bigr)
-
\frac{e^{\alpha_k(\lambda_k-(1-U_\pi))}}{\bigl(1+e^{\alpha_k(\lambda_k-(1-U_\pi))}\bigr)^2}
\right|
\right] \\
&\le
\alpha_k\,\mathbb E\!\left[
\left(\sum_{j\neq\pi} e^{-\alpha_k(U_\pi-U_j)}\right)
\frac{e^{\alpha_k(\lambda_k-(1-U_\pi))}}{\bigl(1+e^{\alpha_k(\lambda_k-(1-U_\pi))}\bigr)^2}
\right]\to0,
\end{align*}
and
\begin{align*}
&\alpha_k\,\mathbb E\!\left[
\left|
a_0(\alpha_k,\beta_k)^2\bigl(1-a_0(\alpha_k,\beta_k)\bigr)
-
\frac{e^{\alpha_k(\lambda_k-(1-U_\pi))}}{\bigl(1+e^{\alpha_k(\lambda_k-(1-U_\pi))}\bigr)^3}
\right|
\right] \\
&\le
2\alpha_k\,\mathbb E\!\left[
\left(\sum_{j\neq\pi} e^{-\alpha_k(U_\pi-U_j)}\right)
\frac{e^{\alpha_k(\lambda_k-(1-U_\pi))}}{\bigl(1+e^{\alpha_k(\lambda_k-(1-U_\pi))}\bigr)^3}
\right]\to0.
\end{align*}
\end{proof}

Let \(\varphi(U_1,\dots,U_N;\alpha,\beta)\) denote the reduced sample loss. The exact sample-level gradient formulas proved earlier are
\[
\partial_\beta\varphi
=
a_0(\alpha,\beta)\bigl(1-a_0(\alpha,\beta)\bigr)
\Bigl(
a_0(\alpha,\beta)+2b_\pi(\alpha)-2\bigl(1-a_0(\alpha,\beta)\bigr)Q(\alpha)
\Bigr),
\]
and
\begin{align*}
\partial_\alpha\varphi
&=
a_0(\alpha,\beta)\bigl(1-a_0(\alpha,\beta)\bigr)\bigl(1-\mu(\alpha)\bigr)
\Bigl(
a_0(\alpha,\beta)+2b_\pi(\alpha)-2\bigl(1-a_0(\alpha,\beta)\bigr)Q(\alpha)
\Bigr) \\
&\quad
+\bigl(1-a_0(\alpha,\beta)\bigr)^2Q'(\alpha)
-2\bigl(1-a_0(\alpha,\beta)\bigr)\partial_\alpha b_\pi(\alpha),
\end{align*}
where
\[
Q(\alpha)=\sum_{i=1}^N b_i(\alpha)^2,~
\mu(\alpha)=\sum_{i=1}^N b_i(\alpha)U_i,
~
\partial_\alpha b_\pi(\alpha)=b_\pi(\alpha)\bigl(U_\pi-\mu(\alpha)\bigr),
~
Q'(\alpha)=2\sum_{i=1}^N b_i(\alpha)^2\bigl(U_i-\mu(\alpha)\bigr).
\]

\begin{theorem}\label{thm:boundary-layer-lambda-final}
Then
\[
\alpha_k\,\mathbb E\!\left[
\frac{1}{N+1}\partial_\beta\varphi(\alpha_k,\beta_k)
+
\frac{\lambda_k}{d}\,\partial_\alpha\varphi(\alpha_k,\beta_k)
\right]
\longrightarrow
\Gamma_\lambda(\lambda)
\coloneqq 
\frac32\,h(\lambda)\left(\frac{1}{N+1}+\frac{\lambda^2}{d}\right).
\]
In particular, \(\Gamma_\lambda(\lambda)>0\) for every \(\lambda\in(0,2)\).
\end{theorem}

\begin{proof}
Write
\[
\frac{1}{N+1}\partial_\beta\varphi(\alpha_k,\beta_k)
+
\frac{\lambda_k}{d}\,\partial_\alpha\varphi(\alpha_k,\beta_k)
=
A_k+B_k,
\]
where
\begin{align*}
A_k
&=
a_0(\alpha_k,\beta_k)\bigl(1-a_0(\alpha_k,\beta_k)\bigr)
\left(\frac{1}{N+1}+\frac{\lambda_k}{d}(1-\mu_k)\right) \\
&\qquad\qquad\times
\Bigl(
a_0(\alpha_k,\beta_k)+2b_{\pi,k}-2\bigl(1-a_0(\alpha_k,\beta_k)\bigr)Q_k
\Bigr),
\end{align*}
and
\[
B_k
=
\frac{\lambda_k}{d}\Bigl(\bigl(1-a_0(\alpha_k,\beta_k)\bigr)^2Q_k'-2\bigl(1-a_0(\alpha_k,\beta_k)\bigr)\partial_\alpha b_\pi(\alpha_k)\Bigr).
\]
Since \(\lambda_k\to\lambda\in(0,2)\), the sequence \((\lambda_k)\) is bounded. By Lemma~\ref{lem:error-bounds-final},
\[
|Q_k'|+2\bigl|\partial_\alpha b_\pi(\alpha_k)\bigr|
\le
6\sum_{j\neq\pi}(U_\pi-U_j)e^{-\alpha_k(U_\pi-U_j)}.
\]
Hence
\[
|B_k|
\le
C\sum_{j\neq\pi}(U_\pi-U_j)e^{-\alpha_k(U_\pi-U_j)}
\]
for some deterministic constant \(C\), and Lemma~\ref{lem:S-global-rigorous-final} yields
\begin{equation}\label{eq:boundary-layer-lambda-final-Bk}
\alpha_k\,\mathbb E[|B_k|]\to0.
\end{equation}
We next compare \(A_k\) with
\[
3\left(\frac{1}{N+1}+\frac{\lambda}{d}(1-U_\pi)\right)
a_0(\alpha_k,\beta_k)^2\bigl(1-a_0(\alpha_k,\beta_k)\bigr).
\]
By Lemma~\ref{lem:error-bounds-final},
\begin{equation}\label{eq:boundary-layer-lambda-final-bracket}
\left|a_0(\alpha_k,\beta_k)+2b_{\pi,k}-2\bigl(1-a_0(\alpha_k,\beta_k)\bigr)Q_k-3a_0(\alpha_k,\beta_k)\right|
\le
6\sum_{j\neq\pi} e^{-\alpha_k(U_\pi-U_j)}.
\end{equation}
Also,
\[
1-\mu_k=(1-U_\pi)+(U_\pi-\mu_k),
\]
and Lemma~\ref{lem:error-bounds-final} gives
\begin{equation}\label{eq:boundary-layer-lambda-final-mu}
0\le U_\pi-\mu_k\le \sum_{j\neq\pi}(U_\pi-U_j)e^{-\alpha_k(U_\pi-U_j)}.
\end{equation}
Therefore
\begin{align}
\label{eq:boundary-layer-lambda-final-factor}
&\left|
\frac{1}{N+1}+\frac{\lambda_k}{d}(1-\mu_k)
-
\left(\frac{1}{N+1}+\frac{\lambda}{d}(1-U_\pi)\right)
\right| \nonumber\\
&\le
\frac{|\lambda_k-\lambda|}{d}\,|1-\mu_k|
+\frac{\lambda}{d}(U_\pi-\mu_k) \nonumber\\
&\le
2|\lambda_k-\lambda|
+
\lambda\sum_{j\neq\pi}(U_\pi-U_j)e^{-\alpha_k(U_\pi-U_j)},
\end{align}
because \(d\ge1\) and \(|1-\mu_k|\le2\).

Using \eqref{eq:boundary-layer-lambda-final-bracket}, \eqref{eq:boundary-layer-lambda-final-mu}, \eqref{eq:boundary-layer-lambda-final-factor}, and the boundedness of \(a_0(\alpha_k,\beta_k)(1-a_0(\alpha_k,\beta_k))\) and \(a_0(\alpha_k,\beta_k)^2(1-a_0(\alpha_k,\beta_k))\), we obtain
\begin{align*}
&\left|A_k-3\left(\frac{1}{N+1}+\frac{\lambda}{d}(1-U_\pi)\right)
a_0(\alpha_k,\beta_k)^2(1-a_0(\alpha_k,\beta_k))\right| \\
&\le
C\,|\lambda_k-\lambda|\,a_0(\alpha_k,\beta_k)^2(1-a_0(\alpha_k,\beta_k)) \\
&\quad
+ C\sum_{j\neq\pi}(U_\pi-U_j)e^{-\alpha_k(U_\pi-U_j)} \\
&\quad
+ C\left(\sum_{j\neq\pi} e^{-\alpha_k(U_\pi-U_j)}\right)
a_0(\alpha_k,\beta_k)(1-a_0(\alpha_k,\beta_k))
\end{align*}
for some constant \(C\).

Multiplying by \(\alpha_k\) and taking expectations, the middle term vanishes by Lemma~\ref{lem:S-global-rigorous-final}. For the last term, Lemma~\ref{lem:c-replacement-final} gives
\begin{align*}
a_0(\alpha_k,\beta_k)(1-a_0(\alpha_k,\beta_k))
&\le
\frac{e^{\alpha_k(\lambda_k-(1-U_\pi))}}{\bigl(1+e^{\alpha_k(\lambda_k-(1-U_\pi))}\bigr)^2} \\
&\quad
+
\left|
a_0(\alpha_k,\beta_k)(1-a_0(\alpha_k,\beta_k))
-
\frac{e^{\alpha_k(\lambda_k-(1-U_\pi))}}{\bigl(1+e^{\alpha_k(\lambda_k-(1-U_\pi))}\bigr)^2}
\right| \\
&\le
\frac{e^{\alpha_k(\lambda_k-(1-U_\pi))}}{\bigl(1+e^{\alpha_k(\lambda_k-(1-U_\pi))}\bigr)^2}
+
\left(\sum_{j\neq\pi} e^{-\alpha_k(U_\pi-U_j)}\right)
\frac{e^{\alpha_k(\lambda_k-(1-U_\pi))}}{\bigl(1+e^{\alpha_k(\lambda_k-(1-U_\pi))}\bigr)^2},
\end{align*}
hence, using \(\sum_{j\neq\pi} e^{-\alpha_k(U_\pi-U_j)}\le N-1\),
\begin{align*}
&\alpha_k\,\mathbb E\!\left[
\left(\sum_{j\neq\pi} e^{-\alpha_k(U_\pi-U_j)}\right)
a_0(\alpha_k,\beta_k)(1-a_0(\alpha_k,\beta_k))
\right] \\
&\le
\alpha_k\,\mathbb E\!\left[
\left(\sum_{j\neq\pi} e^{-\alpha_k(U_\pi-U_j)}\right)
\frac{e^{\alpha_k(\lambda_k-(1-U_\pi))}}{\bigl(1+e^{\alpha_k(\lambda_k-(1-U_\pi))}\bigr)^2}
\right] \\
&\quad
+
(N-1)\alpha_k\,\mathbb E\!\left[
\left(\sum_{j\neq\pi} e^{-\alpha_k(U_\pi-U_j)}\right)
\frac{e^{\alpha_k(\lambda_k-(1-U_\pi))}}{\bigl(1+e^{\alpha_k(\lambda_k-(1-U_\pi))}\bigr)^2}
\right]\to0
\end{align*}
by Lemma~\ref{lem:R-local-rigorous-final}. For the \(|\lambda_k-\lambda|\)-term,
\[
\alpha_k\,\mathbb E\!\left[a_0(\alpha_k,\beta_k)^2(1-a_0(\alpha_k,\beta_k))\right]
=
O(1)
\]
by Lemma~\ref{lem:c-replacement-final} and Lemma~\ref{lem:kernel-limits-rigorous-final} with \(\psi\equiv1\). Therefore
\[
\alpha_k\,|\lambda_k-\lambda|\,\mathbb E\!\left[a_0(\alpha_k,\beta_k)^2(1-a_0(\alpha_k,\beta_k))\right]\to0.
\]
Consequently,
\begin{align}
\label{eq:boundary-layer-lambda-final-Ak-main}
&\alpha_k\,\mathbb E\!\left[
\left|A_k-3\left(\frac{1}{N+1}+\frac{\lambda}{d}(1-U_\pi)\right)
a_0(\alpha_k,\beta_k)^2(1-a_0(\alpha_k,\beta_k))\right|
\right]\to0.
\end{align}
Next, Lemma~\ref{lem:c-replacement-final} yields
\begin{align}
\label{eq:boundary-layer-lambda-final-replace-K}
&\alpha_k\,\mathbb E\!\left[
\left|
3\left(\frac{1}{N+1}+\frac{\lambda}{d}(1-U_\pi)\right)
\left(
a_0(\alpha_k,\beta_k)^2(1-a_0(\alpha_k,\beta_k))
-
\frac{e^{\alpha_k(\lambda_k-(1-U_\pi))}}{\bigl(1+e^{\alpha_k(\lambda_k-(1-U_\pi))}\bigr)^3}
\right)
\right|
\right]\to0.
\end{align}
Combining \eqref{eq:boundary-layer-lambda-final-Ak-main} and \eqref{eq:boundary-layer-lambda-final-replace-K},
\begin{equation}\label{eq:boundary-layer-lambda-final-Ak-reduction}
\alpha_k\,\mathbb E[A_k]
-
3\alpha_k\,\mathbb E\!\left[
\left(\frac{1}{N+1}+\frac{\lambda}{d}(1-U_\pi)\right)
\frac{e^{\alpha_k(\lambda_k-(1-U_\pi))}}{\bigl(1+e^{\alpha_k(\lambda_k-(1-U_\pi))}\bigr)^3}
\right]
\to0.
\end{equation}
Applying Lemma~\ref{lem:kernel-limits-rigorous-final} with
\[
\psi(t)=3\left(\frac{1}{N+1}+\frac{\lambda}{d}t\right),
\]
we obtain
\begin{equation}\label{eq:boundary-layer-lambda-final-kernel-limit}
3\alpha_k\,\mathbb E\!\left[
\left(\frac{1}{N+1}+\frac{\lambda}{d}(1-U_\pi)\right)
\frac{e^{\alpha_k(\lambda_k-(1-U_\pi))}}{\bigl(1+e^{\alpha_k(\lambda_k-(1-U_\pi))}\bigr)^3}
\right]
\longrightarrow
\frac32\,h(\lambda)\left(\frac{1}{N+1}+\frac{\lambda^2}{d}\right).
\end{equation}
Combining \eqref{eq:boundary-layer-lambda-final-Bk}, \eqref{eq:boundary-layer-lambda-final-Ak-reduction}, and \eqref{eq:boundary-layer-lambda-final-kernel-limit} proves
\[
\alpha_k\,\mathbb E\!\left[
\frac{1}{N+1}\partial_\beta\varphi(\alpha_k,\beta_k)
+
\frac{\lambda_k}{d}\partial_\alpha\varphi(\alpha_k,\beta_k)
\right]
\longrightarrow
\frac32\,h(\lambda)\left(\frac{1}{N+1}+\frac{\lambda^2}{d}\right).
\]
Since \(h(\lambda)>0\) on \((0,2)\), the positivity is immediate.
\end{proof}

\begin{lemma}
\label{prop:alpha-boundary-layer-final}
We have
\[
\alpha_k\,\partial_\alpha\mathcal L(\alpha_k,\beta_k)
\longrightarrow
\Gamma_\alpha(\lambda)
\coloneqq 
\frac32\,\lambda\,h(\lambda).
\]
In particular, \(\Gamma_\alpha(\lambda)>0\) for every \(\lambda\in(0,2)\).
\end{lemma}

\begin{proof}
Write
\[
\partial_\alpha\varphi(\alpha_k,\beta_k)=A_k^\alpha+B_k^\alpha,
\]
where
\begin{align*}
A_k^\alpha
&=
a_0(\alpha_k,\beta_k)\bigl(1-a_0(\alpha_k,\beta_k)\bigr)(1-\mu_k) \\
&\qquad\qquad\times
\Bigl(
a_0(\alpha_k,\beta_k)+2b_{\pi,k}-2\bigl(1-a_0(\alpha_k,\beta_k)\bigr)Q_k
\Bigr),
\end{align*}
and
\[
B_k^\alpha
=
\bigl(1-a_0(\alpha_k,\beta_k)\bigr)^2Q_k'-2\bigl(1-a_0(\alpha_k,\beta_k)\bigr)\partial_\alpha b_\pi(\alpha_k).
\]
By Lemma~\ref{lem:error-bounds-final},
\[
|Q_k'|+2\bigl|\partial_\alpha b_\pi(\alpha_k)\bigr|
\le
6\sum_{j\neq\pi}(U_\pi-U_j)e^{-\alpha_k(U_\pi-U_j)},
\]
so
\[
|B_k^\alpha|
\le
6\sum_{j\neq\pi}(U_\pi-U_j)e^{-\alpha_k(U_\pi-U_j)}.
\]
Lemma~\ref{lem:S-global-rigorous-final} then gives
\begin{equation}\label{eq:alpha-boundary-layer-final-B}
\alpha_k\,\mathbb E[|B_k^\alpha|]\to0.
\end{equation}
We compare \(A_k^\alpha\) with
\[
3(1-U_\pi)\,a_0(\alpha_k,\beta_k)^2(1-a_0(\alpha_k,\beta_k)).
\]
Since
\[
1-\mu_k=(1-U_\pi)+(U_\pi-\mu_k),
\]
Lemma~\ref{lem:error-bounds-final} yields
\[
|1-\mu_k-(1-U_\pi)|=U_\pi-\mu_k\le \sum_{j\neq\pi}(U_\pi-U_j)e^{-\alpha_k(U_\pi-U_j)},
\]
and again
\[
\left|a_0(\alpha_k,\beta_k)+2b_{\pi,k}-2\bigl(1-a_0(\alpha_k,\beta_k)\bigr)Q_k-3a_0(\alpha_k,\beta_k)\right|
\le
6\sum_{j\neq\pi} e^{-\alpha_k(U_\pi-U_j)}.
\]
Hence
\begin{align*}
&\bigl|A_k^\alpha-3(1-U_\pi)\,a_0(\alpha_k,\beta_k)^2(1-a_0(\alpha_k,\beta_k))\bigr| \\
&\le
C\sum_{j\neq\pi}(U_\pi-U_j)e^{-\alpha_k(U_\pi-U_j)}
+
C\left(\sum_{j\neq\pi} e^{-\alpha_k(U_\pi-U_j)}\right)
a_0(\alpha_k,\beta_k)(1-a_0(\alpha_k,\beta_k))
\end{align*}
for some constant \(C\), since \(0\le 1-U_\pi\le2\). Using Lemma~\ref{lem:S-global-rigorous-final} and the same estimate
\[
\alpha_k\,\mathbb E\!\left[
\left(\sum_{j\neq\pi} e^{-\alpha_k(U_\pi-U_j)}\right)
a_0(\alpha_k,\beta_k)(1-a_0(\alpha_k,\beta_k))
\right]\to0
\]
proved in the preceding theorem, we obtain
\begin{equation}\label{eq:alpha-boundary-layer-final-A}
\alpha_k\,\mathbb E\!\left[
\bigl|A_k^\alpha-3(1-U_\pi)\,a_0(\alpha_k,\beta_k)^2(1-a_0(\alpha_k,\beta_k))\bigr|
\right]\to0.
\end{equation}

By Lemma~\ref{lem:c-replacement-final},
\begin{align}
\label{eq:alpha-boundary-layer-final-replace-K}
&\alpha_k\,\mathbb E\!\left[
\left|
3(1-U_\pi)
\left(
a_0(\alpha_k,\beta_k)^2(1-a_0(\alpha_k,\beta_k))
-
\frac{e^{\alpha_k(\lambda_k-(1-U_\pi))}}{\bigl(1+e^{\alpha_k(\lambda_k-(1-U_\pi))}\bigr)^3}
\right)
\right|
\right]\to0.
\end{align}
Combining \eqref{eq:alpha-boundary-layer-final-A} and \eqref{eq:alpha-boundary-layer-final-replace-K},
\begin{equation}\label{eq:alpha-boundary-layer-final-A-reduction}
\alpha_k\,\mathbb E[A_k^\alpha]
-
3\alpha_k\,\mathbb E\!\left[
(1-U_\pi)\frac{e^{\alpha_k(\lambda_k-(1-U_\pi))}}{\bigl(1+e^{\alpha_k(\lambda_k-(1-U_\pi))}\bigr)^3}
\right]
\to0.
\end{equation}
Applying Lemma~\ref{lem:kernel-limits-rigorous-final} with \(\psi(t)=3t\),
\begin{equation}\label{eq:alpha-boundary-layer-final-kernel-limit}
3\alpha_k\,\mathbb E\!\left[
(1-U_\pi)\frac{e^{\alpha_k(\lambda_k-(1-U_\pi))}}{\bigl(1+e^{\alpha_k(\lambda_k-(1-U_\pi))}\bigr)^3}
\right]
\longrightarrow
\frac32\,\lambda h(\lambda).
\end{equation}
Finally, \eqref{eq:alpha-boundary-layer-final-B}, \eqref{eq:alpha-boundary-layer-final-A-reduction}, and \eqref{eq:alpha-boundary-layer-final-kernel-limit} imply
\[
\alpha_k\,\partial_\alpha\mathcal L(\alpha_k,\beta_k)
=
\alpha_k\,\mathbb E[\partial_\alpha\varphi(\alpha_k,\beta_k)]
\longrightarrow
\frac32\,\lambda h(\lambda).
\]
\end{proof}
\subsection{Convergence of the population loss}
\label{sec:convergence of the population loss}

We now return from the subsequential boundary-layer analysis to the full gradient-descent trajectory. The results of the previous subsection provide the key interior asymptotics needed to control the global dynamics of
\[
\lambda_t\coloneqq -\frac{\beta^{(t)}}{\alpha^{(t)}}.
\]
In this subsection, we first combine the interior estimates with a subsequential analysis of the self-token weight to rule out persistent visits of \(\lambda_t\) to compact subsets of \((0,2)\), which yields the lower bound
\[
\liminf_{t\to\infty}\lambda_t\ge 2.
\]
We then complement this with an analysis to rule out persistent excursions to the region \((2,\infty)\), thereby proving
\[
\limsup_{t\to\infty}\lambda_t\le 2.
\]
Hence
\[
\lambda_t\to 2.
\]
Finally, we use this limiting slope behavior together with the asymptotic representation of the population loss to prove that the self-token contribution vanishes and therefore
\[
\mathcal L(\alpha^{(t)},\beta^{(t)})\to 0.
\]

\begin{lemma}
\label{lem:subseq-ct-limit-final}
Let \((s_k)_{k\ge1}\) be any subsequence such that
\[
s_k\to\infty,
\qquad
\lambda_{s_k}\coloneqq -\frac{\beta^{(s_k)}}{\alpha^{(s_k)}}\to \lambda_\ast\in[0,\infty].
\]
Then
\[
a_0(\alpha^{(s_k)},\beta^{(s_k)})\to \mathbf 1_{\{1-U_\pi>\lambda_\ast\}}
\qquad\text{almost surely}.
\]
In particular, since \(U_\pi<1\) almost surely, \(1-U_\pi>0\) almost surely.
\end{lemma}

\begin{proof}
By the exact score formula,
\begin{equation}\label{eq:subseq-ct-limit-final-ct}
a_0(\alpha^{(t)},\beta^{(t)})
=
\left(
1+\sum_{j=1}^N e^{-\beta^{(t)}-\alpha^{(t)}(1-U_j)}
\right)^{-1}
=
\left(
1+\sum_{j=1}^N e^{\alpha^{(t)}(\lambda_t-(1-U_j))}
\right)^{-1}.
\end{equation}
Since \(U_j=1\) has probability zero for each fixed \(j\), and \(N\) is finite,
\[
\mathbb P(U_\pi=1)\le \sum_{j=1}^N \mathbb P(U_j=1)=0.
\]
Hence \(1-U_\pi>0\) almost surely.

Assume first that \(\lambda_\ast<\infty\). Since \(1-U_\pi\) has a continuous law, \(\mathbb P(1-U_\pi=\lambda_\ast)=0\). On the event \(\{1-U_\pi>\lambda_\ast\}\), set \(\varepsilon=((1-U_\pi)-\lambda_\ast)/2>0\). For all sufficiently large \(k\),
\[
\lambda_{s_k}\le \lambda_\ast+\varepsilon=(1-U_\pi)-\varepsilon\le 1-U_j-\varepsilon
\qquad\text{for all }j\in[N],
\]
because \(1-U_\pi\le 1-U_j\). Therefore
\[
0\le \sum_{j=1}^N e^{\alpha^{(s_k)}(\lambda_{s_k}-(1-U_j))}
\le N e^{-\alpha^{(s_k)}\varepsilon}\to0,
\]
so by \eqref{eq:subseq-ct-limit-final-ct},
\[
a_0(\alpha^{(s_k)},\beta^{(s_k)})\to1.
\]
On the event \(\{1-U_\pi<\lambda_\ast\}\), set \(\varepsilon=(\lambda_\ast-(1-U_\pi))/2>0\). For all sufficiently large \(k\),
\[
\lambda_{s_k}\ge \lambda_\ast-\varepsilon=(1-U_\pi)+\varepsilon.
\]
Since \(1-U_\pi\) is the \(j=\pi\) case, we have
\[
\lambda_{s_k}-(1-U_\pi)\ge \varepsilon,
\]
and hence the \(\pi\)-term in \eqref{eq:subseq-ct-limit-final-ct} satisfies
\[
e^{\alpha^{(s_k)}(\lambda_{s_k}-(1-U_\pi))}
\ge e^{\alpha^{(s_k)}\varepsilon}\to+\infty.
\]
Thus \(a_0(\alpha^{(s_k)},\beta^{(s_k)})\to0\). Since \(\mathbb P(1-U_\pi=\lambda_\ast)=0\), this proves
\[
a_0(\alpha^{(s_k)},\beta^{(s_k)})\to \mathbf 1_{\{1-U_\pi>\lambda_\ast\}}
\qquad\text{almost surely}
\]
when \(\lambda_\ast<\infty\).

If \(\lambda_\ast=\infty\), then for almost every outcome \(1-U_\pi<\infty\), hence for all sufficiently large \(k\),
\[
\lambda_{s_k}\ge (1-U_\pi)+1.
\]
Therefore
\[
\lambda_{s_k}-(1-U_\pi)\ge1,
\]
so again the \(\pi\)-term in \eqref{eq:subseq-ct-limit-final-ct} diverges to \(+\infty\), and therefore
\[
a_0(\alpha^{(s_k)},\beta^{(s_k)})\to0.
\]
This is exactly
\[
a_0(\alpha^{(s_k)},\beta^{(s_k)})\to \mathbf 1_{\{1-U_\pi>\infty\}}
\qquad\text{almost surely.}
\]
\end{proof}

\begin{lemma}
\label{prop:uniform-interior-bounds-final}
Let
\[
I=[a,b]\subset(0,2)
\]
be compact. Then there exist constants
\[
0<\underline A_I\le \overline A_I<\infty,
\qquad
0<\underline G_I\le \overline G_I<\infty,
\]
and an integer \(T_I\) such that for every \(t\ge T_I\),
\begin{equation}\label{eq:uniform-alpha-bound-final}
\lambda_t\in I
\quad\Longrightarrow\quad
\underline A_I
\le
\alpha^{(t)}\,\partial_\alpha\mathcal L(\alpha^{(t)},\beta^{(t)})
\le
\overline A_I,
\end{equation}
and
\begin{equation}\label{eq:uniform-g-bound-final}
\lambda_t\in I
\quad\Longrightarrow\quad
\underline G_I
\le
\alpha^{(t)}\left(
\frac{1}{N+1}\partial_\beta\mathcal L(\alpha^{(t)},\beta^{(t)})
+
\frac{\lambda_t}{d}\partial_\alpha\mathcal L(\alpha^{(t)},\beta^{(t)})
\right)
\le
\overline G_I.
\end{equation}
\end{lemma}

\begin{proof}
Recall the definition in Lemma~\ref{prop:alpha-boundary-layer-final} and Theorem~\ref{thm:boundary-layer-lambda-final} that
\[
\Gamma_\alpha(\lambda)\coloneqq \frac32\,\lambda\,h(\lambda),
\qquad
\Gamma_\lambda(\lambda)\coloneqq \frac32\,h(\lambda)\left(\frac{1}{N+1}+\frac{\lambda^2}{d}\right),
\qquad \lambda\in(0,2).
\]
Since \(h\) is continuous and strictly positive on \((0,2)\), the functions \(\Gamma_\alpha\) and \(\Gamma_\lambda\) are continuous and strictly positive on \((0,2)\). On the compact interval \(I\), set
\[
m_\alpha\coloneqq \min_{\lambda\in I}\Gamma_\alpha(\lambda),
\quad
M_\alpha\coloneqq \max_{\lambda\in I}\Gamma_\alpha(\lambda),
\]
and
\[
m_\lambda\coloneqq \min_{\lambda\in I}\Gamma_\lambda(\lambda),
\quad
M_\lambda\coloneqq \max_{\lambda\in I}\Gamma_\lambda(\lambda).
\]
Then
\[
0<m_\alpha\le M_\alpha<\infty,
\qquad
0<m_\lambda\le M_\lambda<\infty.
\]
Define
\[
\underline A_I\coloneqq \frac{m_\alpha}{2},
\qquad
\overline A_I\coloneqq 2M_\alpha,
\qquad
\underline G_I\coloneqq \frac{m_\lambda}{2},
\qquad
\overline G_I\coloneqq 2M_\lambda.
\]
Suppose \eqref{eq:uniform-alpha-bound-final} fails. Then there exists a sequence \(t_n\to\infty\) with \(\lambda_{t_n}\in I\) such that either
\[
\alpha^{(t_n)}\partial_\alpha\mathcal L(\alpha^{(t_n)},\beta^{(t_n)})<\underline A_I
\]
or
\[
\alpha^{(t_n)}\partial_\alpha\mathcal L(\alpha^{(t_n)},\beta^{(t_n)})>\overline A_I
\]
for every \(n\). Passing to a subsequence, \(\lambda_{t_n}\to\lambda_\ast\in I\). Lemma~\ref{prop:alpha-boundary-layer-final} then gives
\[
\alpha^{(t_n)}\partial_\alpha\mathcal L(\alpha^{(t_n)},\beta^{(t_n)})\to \Gamma_\alpha(\lambda_\ast)\in[m_\alpha,M_\alpha],
\]
contradicting the construction of \(t_n\). Thus \eqref{eq:uniform-alpha-bound-final} holds for all large \(t\).

The proof of \eqref{eq:uniform-g-bound-final} is identical, using Theorem~\ref{thm:boundary-layer-lambda-final} instead of Lemma~\ref{prop:alpha-boundary-layer-final}. Taking \(T_I\) larger if necessary, both bounds hold for all \(t\ge T_I\).
\end{proof}

\begin{lemma}
\label{lem:uniform-interior-step-final}
Let
\[
I=[a,b]\subset(0,2)
\]
be compact. Then there exist constants
\[
0<\underline d_I\le \overline d_I<\infty,
\qquad
\delta_I>0,
\]
and an integer \(T_I'\ge T_I\) such that for every \(t\ge T_I'\),
\begin{equation}\label{eq:uniform-lambda-step-final}
\lambda_t\in I
\quad\Longrightarrow\quad
\frac{\underline d_I}{(\alpha^{(t)})^2}
\le
\lambda_{t+1}-\lambda_t
\le
\frac{\overline d_I}{(\alpha^{(t)})^2},
\end{equation}
and
\begin{equation}\label{eq:uniform-loss-drop-final}
\lambda_t\in I
\quad\Longrightarrow\quad
\mathcal L(\alpha^{(t)},\beta^{(t)})-\mathcal L(\alpha^{(t+1)},\beta^{(t+1)})
\ge
\frac{\delta_I}{(\alpha^{(t)})^2}.
\end{equation}
\end{lemma}

\begin{proof}
The exact recursion for \(\lambda_t\) is
\begin{equation}\label{eq:uniform-interior-step-final-lambda-recursion}
\lambda_{t+1}-\lambda_t
=
\frac{\eta\left(
\frac{1}{N+1}\partial_\beta \mathcal L(\alpha^{(t)},\beta^{(t)})
+\frac{\lambda_t}{d}\partial_\alpha \mathcal L(\alpha^{(t)},\beta^{(t)})
\right)}
{\alpha^{(t)}-\frac{\eta}{d}\,\partial_\alpha\mathcal L(\alpha^{(t)},\beta^{(t)})}.
\end{equation}
If \(\lambda_t\in I\), then by \eqref{eq:uniform-alpha-bound-final},
\[
0<\underline A_I
\le
\alpha^{(t)}\partial_\alpha\mathcal L(\alpha^{(t)},\beta^{(t)})
\le
\overline A_I.
\]
Hence
\[
0\le \frac{\eta}{d}\,\partial_\alpha\mathcal L(\alpha^{(t)},\beta^{(t)})
\le \frac{\eta\overline A_I}{d\,\alpha^{(t)}}.
\]
Since \(\alpha^{(t)}\to+\infty\), after enlarging \(T_I\) if necessary we may assume that for all \(t\ge T_I'\),
\begin{equation}\label{eq:uniform-interior-step-final-denominator}
\frac12\,\alpha^{(t)}
\le
\alpha^{(t)}-\frac{\eta}{d}\,\partial_\alpha\mathcal L(\alpha^{(t)},\beta^{(t)})
\le
2\,\alpha^{(t)}.
\end{equation}
If \(\lambda_t\in I\), then by \eqref{eq:uniform-g-bound-final},
\[
\underline G_I
\le
\alpha^{(t)}\left(
\frac{1}{N+1}\partial_\beta \mathcal L(\alpha^{(t)},\beta^{(t)})
+\frac{\lambda_t}{d}\partial_\alpha \mathcal L(\alpha^{(t)},\beta^{(t)})
\right)
\le
\overline G_I.
\]
Combining this with \eqref{eq:uniform-interior-step-final-lambda-recursion} and \eqref{eq:uniform-interior-step-final-denominator},
\[
\frac{\eta\,\underline G_I}{2(\alpha^{(t)})^2}
\le
\lambda_{t+1}-\lambda_t
\le
\frac{2\eta\,\overline G_I}{(\alpha^{(t)})^2}.
\]
Thus \eqref{eq:uniform-lambda-step-final} holds with
\[
\underline d_I\coloneqq \frac{\eta\,\underline G_I}{2},
\qquad
\overline d_I\coloneqq 2\eta\,\overline G_I.
\]
For the loss, the previously proved descent inequality gives
\[
\mathcal L(\alpha^{(t)},\beta^{(t)})-\mathcal L(\alpha^{(t+1)},\beta^{(t+1)})
\ge
\frac{\eta}{2}\left(
\frac{1}{d}\bigl(\partial_\alpha\mathcal L(\alpha^{(t)},\beta^{(t)})\bigr)^2
+
\frac{1}{N+1}\bigl(\partial_\beta\mathcal L(\alpha^{(t)},\beta^{(t)})\bigr)^2
\right).
\]
If \(\lambda_t\in I\), then by \eqref{eq:uniform-alpha-bound-final},
\[
\bigl(\partial_\alpha\mathcal L(\alpha^{(t)},\beta^{(t)})\bigr)^2
\ge
\frac{\underline A_I^2}{(\alpha^{(t)})^2}.
\]
Therefore
\[
\mathcal L(\alpha^{(t)},\beta^{(t)})-\mathcal L(\alpha^{(t+1)},\beta^{(t+1)})
\ge
\frac{\eta\,\underline A_I^2}{2d(\alpha^{(t)})^2}.
\]
Thus \eqref{eq:uniform-loss-drop-final} holds with
\[
\delta_I\coloneqq \frac{\eta\,\underline A_I^2}{2d}.
\]
\end{proof}

\begin{lemma}
\label{lem:no-tail-trap-final}
Let
\[
I=[a,b]\subset(0,2)
\]
be compact. There does not exist an integer \(T_0\) such that
\[
\lambda_t\in I
\qquad\text{for all }t\ge T_0.
\]
\end{lemma}

\begin{proof}
Assume, toward a contradiction, that \(\lambda_t\in I\) for all \(t\ge T_0\). Let \(T_\ast\coloneqq \max\{T_0,T_I\}\). Then by \eqref{eq:uniform-alpha-bound-final},
\[
\partial_\alpha\mathcal L(\alpha^{(t)},\beta^{(t)})\ge \frac{\underline A_I}{\alpha^{(t)}}>0
\qquad\text{for all }t\ge T_\ast.
\]
Therefore the gradient-descent update gives
\[
\alpha^{(t+1)}
=
\alpha^{(t)}-\frac{\eta}{d}\,\partial_\alpha\mathcal L(\alpha^{(t)},\beta^{(t)})
<
\alpha^{(t)}
\qquad (t\ge T_\ast).
\]
Thus \((\alpha^{(t)})_{t\ge T_\ast}\) is strictly decreasing, so
\[
\alpha^{(t)}\le \alpha^{(T_\ast)}
\qquad\text{for all }t\ge T_\ast.
\]
This contradicts the already proved divergence \(\alpha^{(t)}\to+\infty\).
\end{proof}

\begin{theorem}\label{thm:finite-upcrossings-final}
Fix any
\[
0<a<b<2.
\]
Then the sequence \((\lambda_t)_{t\ge0}\) has only finitely many upcrossings of \([a,b]\).

More precisely, define recursively
\[
u_1\coloneqq \inf\{t\ge0:\lambda_t\le a\},
\qquad
v_1\coloneqq \inf\{t>u_1:\lambda_t\ge b\},
\]
and, whenever \(v_n<\infty\),
\[
u_{n+1}\coloneqq \inf\{t>v_n:\lambda_t\le a\},
\qquad
v_{n+1}\coloneqq \inf\{t>u_{n+1}:\lambda_t\ge b\}.
\]
Then only finitely many indices \(n\) satisfy \(v_n<\infty\).
\end{theorem}

\begin{proof}
Define
\[
I\coloneqq [a_I,b_I]
\coloneqq 
\left[\frac{2a+b}{3},\,\frac{a+2b}{3}\right]
\subset(a,b)\subset(0,2),
\]
so that
\[
b_I-a_I=\frac{b-a}{3}.
\]
Let \(T_I'\) and \(\overline d_I,\delta_I\) be given by Lemma~\ref{lem:uniform-interior-step-final}.

We also need a global small-step bound before the level \(b\) is reached. Since \(|\partial_\alpha\mathcal L|,|\partial_\beta\mathcal L|\le 10\),
\[
|\lambda_{t+1}-\lambda_t|
=
\left|
\frac{\eta\left(
\frac{1}{N+1}\partial_\beta \mathcal L(\alpha^{(t)},\beta^{(t)})
+\frac{\lambda_t}{d}\partial_\alpha \mathcal L(\alpha^{(t)},\beta^{(t)})
\right)}
{\alpha^{(t)}-\frac{\eta}{d}\,\partial_\alpha \mathcal L(\alpha^{(t)},\beta^{(t)})}
\right|
\le
\frac{\eta\left(\frac{10}{N+1}+\frac{10b}{d}\right)}
{\alpha^{(t)}-\frac{10\eta}{d}}
\]
whenever \(\lambda_t\le b\) and \(\alpha^{(t)}>10\eta/d\). Since \(\alpha^{(t)}\to+\infty\), there exists \(T_b\) such that, for all \(t\ge T_b\),
\begin{equation}\label{eq:finite-upcrossings-final-small-step}
\lambda_t\le b
\quad\Longrightarrow\quad
|\lambda_{t+1}-\lambda_t|<\frac{b-a}{9}.
\end{equation}
Set
\[
T_\ast\coloneqq \max\{T_I',T_b\}.
\]
Consider any upcrossing \((u_n,v_n)\) with \(u_n\ge T_\ast\) and \(v_n<\infty\). Since
\[
\lambda_{u_n}\le a<a_I,
\qquad
\lambda_{v_n}\ge b>b_I,
\]
and \(\lambda_t<b\) for every \(u_n\le t<v_n\), define
\[
r_n\coloneqq \max\{t<v_n:\lambda_t<a_I\}.
\]
This set is nonempty because \(u_n\in\{t<v_n:\lambda_t<a_I\}\). Set
\[
s_n\coloneqq r_n+1.
\]
Then \(\lambda_{s_n}\ge a_I\) by the maximality of \(r_n\), and since \(\lambda_{r_n}<a_I\) and \(\lambda_{r_n}<b\), \eqref{eq:finite-upcrossings-final-small-step} gives
\begin{equation}\label{eq:finite-upcrossings-final-sn-upper}
\lambda_{s_n}
=
\lambda_{r_n+1}
\le
a_I+\frac{b-a}{9}.
\end{equation}
Now define
\[
e_n\coloneqq \min\{t\ge s_n:\lambda_t>b_I\}.
\]
This index exists because \(v_n\) belongs to the set. By the minimality of \(e_n\),
\[
\lambda_t\le b_I
\qquad\text{for all }t=s_n,\dots,e_n-1.
\]
By the maximality of \(r_n\),
\[
\lambda_t\ge a_I
\qquad\text{for all }t=s_n,\dots,v_n-1.
\]
Hence
\begin{equation}\label{eq:finite-upcrossings-final-in-I}
\lambda_t\in I
\qquad\text{for all }t=s_n,\dots,e_n-1.
\end{equation}
Since \(\lambda_{e_n}>b_I\), \eqref{eq:finite-upcrossings-final-sn-upper} yields
\begin{equation}\label{eq:finite-upcrossings-final-gap}
\lambda_{e_n}-\lambda_{s_n}
>
b_I-\left(a_I+\frac{b-a}{9}\right)
=
\frac{2(b-a)}{9}.
\end{equation}
On the other hand, by \eqref{eq:finite-upcrossings-final-in-I} and \eqref{eq:uniform-lambda-step-final},
\[
\lambda_{e_n}-\lambda_{s_n}
=
\sum_{t=s_n}^{e_n-1}(\lambda_{t+1}-\lambda_t)
\le
\overline d_I\sum_{t=s_n}^{e_n-1}\frac{1}{(\alpha^{(t)})^2}.
\]
Combining with \eqref{eq:finite-upcrossings-final-gap},
\begin{equation}\label{eq:finite-upcrossings-final-sum-lower}
\sum_{t=s_n}^{e_n-1}\frac{1}{(\alpha^{(t)})^2}
\ge
\frac{2(b-a)}{9\,\overline d_I}.
\end{equation}
Using \eqref{eq:uniform-loss-drop-final},
\[
\mathcal L(\alpha^{(s_n)},\beta^{(s_n)})-\mathcal L(\alpha^{(e_n)},\beta^{(e_n)})
=
\sum_{t=s_n}^{e_n-1}\bigl(\mathcal L(\alpha^{(t)},\beta^{(t)})-\mathcal L(\alpha^{(t+1)},\beta^{(t+1)})\bigr)
\ge
\delta_I\sum_{t=s_n}^{e_n-1}\frac{1}{(\alpha^{(t)})^2}.
\]
By \eqref{eq:finite-upcrossings-final-sum-lower},
\begin{equation}\label{eq:finite-upcrossings-final-loss-drop}
\mathcal L(\alpha^{(s_n)},\beta^{(s_n)})-\mathcal L(\alpha^{(e_n)},\beta^{(e_n)})
\ge
\frac{2\delta_I(b-a)}{9\,\overline d_I}
=:\varepsilon_{a,b}>0.
\end{equation}

Thus every sufficiently late upcrossing of \([a,b]\) decreases the loss by at least \(\varepsilon_{a,b}\). Since the intervals \([s_n,e_n)\) corresponding to distinct upcrossings are disjoint and \(\mathcal L(\alpha^{(t)},\beta^{(t)})\ge0\) for all \(t\), there can be only finitely many such upcrossings.
\end{proof}

\begin{lemma}
\label{lem:zero-not-subseq-final}
There is no subsequence \(t_k\to\infty\) such that
\[
\lambda_{t_k}\to0.
\]
\end{lemma}

\begin{proof}
Assume, toward a contradiction, that \(\lambda_{t_k}\to0\). Since \(1-U_\pi>0\) almost surely, Lemma~\ref{lem:subseq-ct-limit-final} gives
\[
a_0(\alpha^{(t_k)},\beta^{(t_k)})\to1
\qquad\text{almost surely}.
\]
Because
\[
0\le a_0(\alpha^{(t_k)},\beta^{(t_k)})^2\le1,
\]
bounded convergence yields
\[
\mathbb E\!\left[a_0(\alpha^{(t_k)},\beta^{(t_k)})^2\right]\to1.
\]
By \eqref{eq:exact-loss-reduction-final},
\[
\mathcal L(\alpha^{(t_k)},\beta^{(t_k)})
=
\frac32\,\mathbb E\!\left[a_0(\alpha^{(t_k)},\beta^{(t_k)})^2\right]+o(1)
\to
\frac32.
\]
On the other hand, \(\mathcal L(\alpha^{(t)},\beta^{(t)})\) is nonincreasing along gradient descent, and the previously computed initial value satisfies \(\mathcal L(\alpha^{(0)},\beta^{(0)})<1\). Hence \(\mathcal L(\alpha^{(t)},\beta^{(t)})\le \mathcal L(\alpha^{(0)},\beta^{(0)})<1\) for every \(t\), a contradiction.
\end{proof}

\begin{theorem}\label{thm:liminf-lambda-ge-2-final}
The gradient-descent trajectory satisfies
\[
\liminf_{t\to\infty}\lambda_t\ge 2.
\]
\end{theorem}

\begin{proof}
By Lemma~\ref{lem:zero-not-subseq-final}, \(0\) is not a subsequential limit of \((\lambda_t)\), hence
\[
\liminf_{t\to\infty}\lambda_t>0.
\]
Assume, toward a contradiction, that
\[
\liminf_{t\to\infty}\lambda_t<2.
\]
Choose numbers \(a,b\) such that
\begin{equation}\label{eq:liminf-lambda-ge-2-final-ab}
\liminf_{t\to\infty}\lambda_t<a<b<2.
\end{equation}
By the definition of \(\liminf\), there are infinitely many indices \(t\) such that
\begin{equation}\label{eq:liminf-lambda-ge-2-final-below-a}
\lambda_t\le a.
\end{equation}
We claim that there are also infinitely many indices \(t\) such that
\begin{equation}\label{eq:liminf-lambda-ge-2-final-above-b}
\lambda_t\ge b.
\end{equation}
If not, then there exists \(T_1\) such that
\begin{equation}\label{eq:liminf-lambda-ge-2-final-below-b-eventually}
\lambda_t<b
\qquad\text{for all }t\ge T_1.
\end{equation}
By Lemma \ref{lem:zero-not-subseq-final} and $\lambda_t\geq0$ we have \(\liminf_{t\to\infty}\lambda_t>0\), so there exists \(T_2\) such that
\begin{equation}\label{eq:liminf-lambda-ge-2-final-away-from-zero}
\lambda_t>\frac12\,\liminf_{s\to\infty}\lambda_s
\qquad\text{for all }t\ge T_2.
\end{equation}
Combining \eqref{eq:liminf-lambda-ge-2-final-below-b-eventually} and \eqref{eq:liminf-lambda-ge-2-final-away-from-zero},
\[
\lambda_t\in\left[\frac12\,\liminf_{s\to\infty}\lambda_s,\,b\right]\subset(0,2)
\qquad\text{for all }t\ge \max\{T_1,T_2\},
\]
contradicting Lemma~\ref{lem:no-tail-trap-final}. Thus \eqref{eq:liminf-lambda-ge-2-final-above-b} holds.

We now show that \eqref{eq:liminf-lambda-ge-2-final-below-a} and \eqref{eq:liminf-lambda-ge-2-final-above-b} imply infinitely many upcrossings of \([a,b]\). Since there are infinitely many times with \(\lambda_t\le a\), set
\[
u_1\coloneqq \inf\{t\ge0:\lambda_t\le a\}.
\]
Because \eqref{eq:liminf-lambda-ge-2-final-above-b} gives arbitrarily late times with \(\lambda_t\ge b\), the set \(\{t>u_1:\lambda_t\ge b\}\) is nonempty, so
\[
v_1\coloneqq \inf\{t>u_1:\lambda_t\ge b\}
\]
is finite. Inductively, if \(v_n\) is finite, then \eqref{eq:liminf-lambda-ge-2-final-below-a} implies
\[
u_{n+1}\coloneqq \inf\{t>v_n:\lambda_t\le a\}<\infty,
\]
and \eqref{eq:liminf-lambda-ge-2-final-above-b} implies
\[
v_{n+1}\coloneqq \inf\{t>u_{n+1}:\lambda_t\ge b\}<\infty.
\]
Hence \((\lambda_t)\) has infinitely many upcrossings of \([a,b]\), contradicting Theorem~\ref{thm:finite-upcrossings-final}. Therefore
\[
\liminf_{t\to\infty}\lambda_t\ge2.
\]
\end{proof}

We next complement Theorem~\ref{thm:liminf-lambda-ge-2-final} with the matching upper bound
\[
\limsup_{t\to\infty}\lambda_t\le 2.
\]
To this end, we analyze the exterior regime \(\lambda_t>2\). The key point is that in this regime the self-token contribution is exponentially small, whereas the context-softmax sharpening term in \(\partial_\alpha\mathcal L\) remains polynomially large and strictly negative, which yields a genuine downward drift for \(\lambda_t\).

Fix a sample \((\mathcal D,\mathbf x_q,\omega)\), and recall the notation
\[
U_i\coloneqq \mathbf x_i^\top \mathbf x_q,\qquad
\pi\coloneqq \arg\max_{1\le i\le N}U_i.
\]
\[
b_i(\alpha)\coloneqq \frac{e^{\alpha U_i}}{\sum_{j=1}^N e^{\alpha U_j}},
\qquad
Q(\alpha)\coloneqq \sum_{i=1}^N b_i(\alpha)^2,\qquad
\mu(\alpha)\coloneqq \sum_{i=1}^N b_i(\alpha)U_i.
\]

\begin{lemma}
\label{lem:context-softmax-identities}
With the notation above, we have
\[
U_\pi-\mu(\alpha)
=
b_\pi(\alpha)\sum_{j\neq\pi}(U_\pi-U_j)e^{-\alpha(U_\pi-U_j)},
\]
\[
\partial_\alpha b_\pi(\alpha)
=
b_\pi(\alpha)^2\sum_{j\neq\pi}(U_\pi-U_j)e^{-\alpha(U_\pi-U_j)},
\]
\[
Q(\alpha)
=
b_\pi(\alpha)^2\left(1+\sum_{j\neq\pi}e^{-2\alpha(U_\pi-U_j)}\right),
\]
and
\begin{align*}
Q'(\alpha)-2\partial_\alpha b_\pi(\alpha)
&=
-2b_\pi(\alpha)^2\Biggl(
\sum_{j\neq\pi}(U_\pi-U_j)e^{-2\alpha(U_\pi-U_j)} \\
&\qquad\qquad
+b_\pi(\alpha)
\left(\sum_{j\neq\pi}\bigl(e^{-\alpha(U_\pi-U_j)}-e^{-2\alpha(U_\pi-U_j)}\bigr)\right)
\left(\sum_{j\neq\pi}(U_\pi-U_j)e^{-\alpha(U_\pi-U_j)}\right)
\Biggr).
\end{align*}
In particular,
\begin{equation}\label{eq:Qminus2m-negative}
Q'(\alpha)-2\partial_\alpha b_\pi(\alpha)
\le
-2b_\pi(\alpha)^2\sum_{j\neq\pi}(U_\pi-U_j)e^{-2\alpha(U_\pi-U_j)}
\le
-\frac{2}{N^2}\sum_{j\neq\pi}(U_\pi-U_j)e^{-2\alpha(U_\pi-U_j)}.
\end{equation}
\end{lemma}

\begin{proof}
For \(j\neq\pi\),
\[
b_j(\alpha)=b_\pi(\alpha)e^{-\alpha(U_\pi-U_j)}.
\]
Hence
\[
\mu(\alpha)
=
b_\pi(\alpha)U_\pi+\sum_{j\neq\pi}b_\pi(\alpha)e^{-\alpha(U_\pi-U_j)}U_j
=
U_\pi-b_\pi(\alpha)\sum_{j\neq\pi}(U_\pi-U_j)e^{-\alpha(U_\pi-U_j)},
\]
which gives
\[
U_\pi-\mu(\alpha)
=
b_\pi(\alpha)\sum_{j\neq\pi}(U_\pi-U_j)e^{-\alpha(U_\pi-U_j)}.
\]
Therefore
\[
\partial_\alpha b_\pi(\alpha)
=
b_\pi(\alpha)\bigl(U_\pi-\mu(\alpha)\bigr)
=
b_\pi(\alpha)^2\sum_{j\neq\pi}(U_\pi-U_j)e^{-\alpha(U_\pi-U_j)}.
\]

Next,
\[
Q(\alpha)
=
b_\pi(\alpha)^2+\sum_{j\neq\pi}b_\pi(\alpha)^2e^{-2\alpha(U_\pi-U_j)}
=
b_\pi(\alpha)^2\left(1+\sum_{j\neq\pi}e^{-2\alpha(U_\pi-U_j)}\right).
\]
For \(Q'(\alpha)\), using
\[
Q'(\alpha)=2\sum_{i=1}^N b_i(\alpha)^2\bigl(U_i-\mu(\alpha)\bigr),
\]
we obtain
\begin{align*}
Q'(\alpha)
&=
2b_\pi(\alpha)^2\bigl(U_\pi-\mu(\alpha)\bigr)
+
2\sum_{j\neq\pi}b_\pi(\alpha)^2e^{-2\alpha(U_\pi-U_j)}\bigl(U_j-\mu(\alpha)\bigr) \\
&=
2b_\pi(\alpha)^3\sum_{j\neq\pi}(U_\pi-U_j)e^{-\alpha(U_\pi-U_j)} \\
&\quad
+
2b_\pi(\alpha)^2\sum_{j\neq\pi}e^{-2\alpha(U_\pi-U_j)}
\left(
b_\pi(\alpha)\sum_{\ell\neq\pi}(U_\pi-U_\ell)e^{-\alpha(U_\pi-U_\ell)}
-(U_\pi-U_j)
\right).
\end{align*}
Subtracting \(2\partial_\alpha b_\pi(\alpha)\), we get
\begin{align*}
Q'(\alpha)-2\partial_\alpha b_\pi(\alpha)
&=
2b_\pi(\alpha)^2\Biggl(
\Bigl(
b_\pi(\alpha)\Bigl(1+\sum_{j\neq\pi}e^{-2\alpha(U_\pi-U_j)}\Bigr)-1
\Bigr)
\sum_{j\neq\pi}(U_\pi-U_j)e^{-\alpha(U_\pi-U_j)} \\
&\qquad\qquad
-
\sum_{j\neq\pi}(U_\pi-U_j)e^{-2\alpha(U_\pi-U_j)}
\Biggr).
\end{align*}
Since
\[
b_\pi(\alpha)\Bigl(1+\sum_{j\neq\pi}e^{-2\alpha(U_\pi-U_j)}\Bigr)-1
=
-\,
b_\pi(\alpha)\sum_{j\neq\pi}\bigl(e^{-\alpha(U_\pi-U_j)}-e^{-2\alpha(U_\pi-U_j)}\bigr),
\]
we obtain the displayed identity. Because \(b_\pi(\alpha)\ge 1/N\), \eqref{eq:Qminus2m-negative} follows.
\end{proof}

\begin{lemma}\label{lem:Sigma2-lower-bound}
There exist constants \(c_0>0\), \(A_0\ge1\), depending only on \(N\) and \(d\), such that for every \(\alpha\ge A_0\),
\[
\mathbb E\!\left[\sum_{j\neq\pi}(U_\pi-U_j)e^{-2\alpha(U_\pi-U_j)}\right] \ge c_0\,\alpha^{-d}.
\]
\end{lemma}

\begin{proof}
Let
\[
a\coloneqq \frac{d-1}{2}.
\]
Since
\[
f_d(u)=\kappa_d(1-u^2)^{(d-3)/2}\mathbf 1_{(-1,1)},
\]
we have, as \(s\downarrow0\),
\[
f_d(1-s)\asymp s^{a-1}.
\]
Hence there exist constants \(c_1,c_2>0\) and \(s_0\in(0,1)\) such that for every \(0<s\le s_0\),
\begin{equation}\label{eq:near-one-tail-bounds}
\mathbb P(U_1\in[1-s,1])\ge c_1 s^a,
\qquad
\mathbb P\!\left(U_1\in\left[1-\frac{3s}{8},\,1-\frac{s}{4}\right]\right)\ge c_2 s^a.
\end{equation}
Fix \(\alpha\) so large that \(1/(2\alpha)\le s_0\). Define the event
\[
E_\alpha
\coloneqq 
\left\{
U_1\in\left[1-\frac{1}{8\alpha},\,1\right],\
U_2\in\left[1-\frac{3}{8\alpha},\,1-\frac{1}{4\alpha}\right],\
U_j\le 1-\frac{1}{2\alpha}\ \forall j\ge3
\right\}.
\]
On \(E_\alpha\), we have
\[
U_1>U_2>U_j\qquad (j\ge3),
\]
hence \(\pi=1\), and
\[
U_\pi-U_2=U_1-U_2\in\left[\frac{1}{8\alpha},\,\frac{3}{8\alpha}\right].
\]
Therefore
\begin{equation}\label{eq:sigma2-on-Ealpha}
\sum_{j\neq\pi}(U_\pi-U_j)e^{-2\alpha(U_\pi-U_j)}
\ge
(U_\pi-U_2)e^{-2\alpha(U_\pi-U_2)}
\ge
\frac{1}{8\alpha}e^{-3/4}
\qquad\text{on }E_\alpha.
\end{equation}
By independence and \eqref{eq:near-one-tail-bounds},
\[
\mathbb P(E_\alpha)
=
\mathbb P\!\left(U_1\in\left[1-\frac{1}{8\alpha},1\right]\right)
\mathbb P\!\left(U_2\in\left[1-\frac{3}{8\alpha},1-\frac{1}{4\alpha}\right]\right)
\mathbb P\!\left(U_3\le 1-\frac{1}{2\alpha}\right)^{N-2}.
\]
The first two factors are bounded below by \(c\alpha^{-a}\), while the third factor tends to \(1\) as \(\alpha\to\infty\). Thus there exist \(c_3>0\) and \(A_0\ge1\) such that
\begin{equation}\label{eq:prob-Ealpha-lower}
\mathbb P(E_\alpha)\ge c_3\,\alpha^{-2a}=c_3\,\alpha^{-(d-1)}
\qquad (\alpha\ge A_0).
\end{equation}
Combining \eqref{eq:sigma2-on-Ealpha} and \eqref{eq:prob-Ealpha-lower},
\[
\mathbb E\!\left[\sum_{j\neq\pi}(U_\pi-U_j)e^{-2\alpha(U_\pi-U_j)}\right]
\ge
\frac{1}{8\alpha}e^{-3/4}\,\mathbb P(E_\alpha)
\ge
c_0\,\alpha^{-d}
\]
for a suitable constant \(c_0>0\).
\end{proof}

\begin{lemma}
\label{lem:exterior-gradient-bounds}
Fix \(\varepsilon>0\). There exist constants \(A_\varepsilon\ge1\), \(c_\varepsilon>0\), \(C_\varepsilon>0\), depending only on \(\varepsilon,N,d\), such that whenever
\[
\alpha\ge A_\varepsilon,
\qquad
\lambda\coloneqq -\frac{\beta}{\alpha}\ge 2+\varepsilon,
\]
we have
\[
0\le \partial_\beta \mathcal L(\alpha,\beta)\le C_\varepsilon e^{-\varepsilon\alpha},
\]
\[
\partial_\alpha \mathcal L(\alpha,\beta)\le -\,c_\varepsilon\,\alpha^{-d},
\]
and hence
\[
\frac{1}{N+1}\partial_\beta \mathcal L(\alpha,\beta)
+
\frac{\lambda}{d}\,\partial_\alpha \mathcal L(\alpha,\beta)
\le -\,c_\varepsilon\,\alpha^{-d}.
\]
\end{lemma}

\begin{proof}
Fix a sample. Write
\[
\mathcal B\coloneqq a_0(\alpha,\beta)+2b_\pi(\alpha)-2\bigl(1-a_0(\alpha,\beta)\bigr)Q(\alpha).
\]
Since
\[
b_\pi(\alpha)\ge Q(\alpha)\ge0,
\]
one may rewrite
\[
\mathcal B
=
a_0(\alpha,\beta)+2\bigl(b_\pi(\alpha)-Q(\alpha)\bigr)+2a_0(\alpha,\beta)Q(\alpha),
\]
hence \(\mathcal B\ge0\). Also \(b_\pi(\alpha),Q(\alpha)\le1\), so
\begin{equation}\label{eq:exterior-B-bound}
\mathcal B\le a_0(\alpha,\beta)+2+2a_0(\alpha,\beta)\le 5.
\end{equation}
If \(\lambda\ge 2+\varepsilon\), then for every sample
\[
\lambda-(1-U_i)\ge \lambda-2\ge \varepsilon
\qquad\text{for all }i\in[N],
\]
hence
\begin{equation}\label{eq:exterior-c-bound}
a_0(\alpha,\beta)
=
\left(1+\sum_{i=1}^N e^{\alpha(\lambda-(1-U_i))}\right)^{-1}
\le e^{-\varepsilon\alpha}.
\end{equation}
Using the exact sample formula
\[
\partial_\beta\varphi
=
a_0(\alpha,\beta)\bigl(1-a_0(\alpha,\beta)\bigr)\mathcal B,
\]
together with \eqref{eq:exterior-B-bound} and \eqref{eq:exterior-c-bound},
\[
0\le \partial_\beta\varphi \le 5e^{-\varepsilon\alpha}.
\]
Taking expectations gives
\begin{equation}\label{eq:exterior-beta-bound}
0\le \partial_\beta\mathcal L(\alpha,\beta)\le 5e^{-\varepsilon\alpha}.
\end{equation}
For \(\partial_\alpha\varphi\), the exact formula is
\[
\partial_\alpha\varphi
=
a_0(\alpha,\beta)\bigl(1-a_0(\alpha,\beta)\bigr)\bigl(1-\mu(\alpha)\bigr)\mathcal B
+
\bigl(1-a_0(\alpha,\beta)\bigr)^2Q'(\alpha)
-
2\bigl(1-a_0(\alpha,\beta)\bigr)\partial_\alpha b_\pi(\alpha).
\]
The first term is nonnegative, and using \(1-\mu(\alpha)\le 2\), \eqref{eq:exterior-B-bound}, and \eqref{eq:exterior-c-bound},
\begin{equation}\label{eq:exterior-alpha-firstterm}
0\le a_0(\alpha,\beta)\bigl(1-a_0(\alpha,\beta)\bigr)\bigl(1-\mu(\alpha)\bigr)\mathcal B \le 10 e^{-\varepsilon\alpha}.
\end{equation}
For the second part,
\begin{align*}
&\bigl(1-a_0(\alpha,\beta)\bigr)^2Q'(\alpha) - 2\bigl(1-a_0(\alpha,\beta)\bigr)\partial_\alpha b_\pi(\alpha) \\
&=
\bigl(Q'(\alpha)-2\partial_\alpha b_\pi(\alpha)\bigr)
+
\bigl((1-a_0(\alpha,\beta))^2-1\bigr)Q'(\alpha)
+
2a_0(\alpha,\beta)\,\partial_\alpha b_\pi(\alpha) \\
&\le
\bigl(Q'(\alpha)-2\partial_\alpha b_\pi(\alpha)\bigr)
+
2a_0(\alpha,\beta)|Q'(\alpha)|
+
2a_0(\alpha,\beta)\,\partial_\alpha b_\pi(\alpha).
\end{align*}
Using the earlier deterministic bounds
\[
|Q'(\alpha)|\le 4\sum_{j\neq\pi}(U_\pi-U_j)e^{-\alpha(U_\pi-U_j)},
\qquad
\partial_\alpha b_\pi(\alpha)\le \sum_{j\neq\pi}(U_\pi-U_j)e^{-\alpha(U_\pi-U_j)},
\]
and
\[
x e^{-\alpha x}\le \frac{1}{e\alpha}\qquad(x\ge0),
\]
we have
\[
\sum_{j\neq\pi}(U_\pi-U_j)e^{-\alpha(U_\pi-U_j)}\le \frac{N-1}{e\alpha}.
\]
Therefore
\begin{equation}\label{eq:exterior-alpha-secondterm}
\bigl(1-a_0(\alpha,\beta)\bigr)^2Q'(\alpha) - 2\bigl(1-a_0(\alpha,\beta)\bigr)\partial_\alpha b_\pi(\alpha)
\le
\bigl(Q'(\alpha)-2\partial_\alpha b_\pi(\alpha)\bigr)
+
\frac{10(N-1)}{e\alpha}e^{-\varepsilon\alpha}.
\end{equation}
Now apply Lemma~\ref{lem:context-softmax-identities}:
\begin{equation}\label{eq:exterior-Qminus2m}
Q'(\alpha)-2\partial_\alpha b_\pi(\alpha)
\le
-\frac{2}{N^2}\sum_{j\neq\pi}(U_\pi-U_j)e^{-2\alpha(U_\pi-U_j)}.
\end{equation}
Combining \eqref{eq:exterior-alpha-firstterm}, \eqref{eq:exterior-alpha-secondterm}, and \eqref{eq:exterior-Qminus2m},
\[
\partial_\alpha\varphi
\le
10e^{-\varepsilon\alpha}
+
\frac{10(N-1)}{e\alpha}e^{-\varepsilon\alpha}
-
\frac{2}{N^2}\sum_{j\neq\pi}(U_\pi-U_j)e^{-2\alpha(U_\pi-U_j)}.
\]
Taking expectations and using Lemma~\ref{lem:Sigma2-lower-bound},
\[
\partial_\alpha\mathcal L(\alpha,\beta)
\le
10e^{-\varepsilon\alpha}
+
\frac{10(N-1)}{e\alpha}e^{-\varepsilon\alpha}
-
\frac{2c_0}{N^2}\alpha^{-d}.
\]
Since the first two terms are exponentially small while the last term is polynomial, there exist \(A_\varepsilon\ge A_0\) and \(c_\varepsilon>0\) such that
\begin{equation}\label{eq:exterior-alpha-bound}
\partial_\alpha\mathcal L(\alpha,\beta)\le -\,c_\varepsilon\,\alpha^{-d}
\qquad
(\alpha\ge A_\varepsilon,\ \lambda\ge 2+\varepsilon).
\end{equation}
Finally, combining \eqref{eq:exterior-beta-bound} and \eqref{eq:exterior-alpha-bound}, for \(\lambda\ge 2+\varepsilon\) we have
\[
\frac{1}{N+1}\partial_\beta \mathcal L(\alpha,\beta)
+
\frac{\lambda}{d}\partial_\alpha \mathcal L(\alpha,\beta)
\le
\frac{5}{N+1}e^{-\varepsilon\alpha}
-
\frac{(2+\varepsilon)c_\varepsilon}{d}\,\alpha^{-d}.
\]
Since the exponential term is negligible compared with \(\alpha^{-d}\), by increasing \(A_\varepsilon\) and decreasing \(c_\varepsilon\) if necessary, we obtain
\[
\frac{1}{N+1}\partial_\beta \mathcal L(\alpha,\beta)
+
\frac{\lambda}{d}\partial_\alpha \mathcal L(\alpha,\beta)
\le
-{c_\varepsilon}\,\alpha^{-d}
\]
for all \(\alpha\ge A_\varepsilon\) and \(\lambda\ge2+\varepsilon\).
\end{proof}

\begin{theorem}\label{thm:limsup-lambda-le-2}
Along the gradient-descent trajectory,
\[
\limsup_{t\to\infty}\lambda_t\le 2.
\]
Consequently,
\[
\lambda_t\to 2
\qquad\text{as }t\to\infty.
\]
\end{theorem}

\begin{proof}
We already know that
\[
\alpha_t\to+\infty,
\qquad
\liminf_{t\to\infty}\lambda_t\ge 2,
\]
and that the exact \(\lambda\)-recursion is
\begin{equation}\label{eq:lambda-recursion-ext-proof}
\lambda_{t+1}-\lambda_t
=
\frac{\eta\left(
\frac{1}{N+1}\partial_\beta\mathcal L(\alpha_t,\beta_t)
+\frac{\lambda_t}{d}\partial_\alpha\mathcal L(\alpha_t,\beta_t)
\right)}
{\alpha_t-\frac{\eta}{d}\,\partial_\alpha\mathcal L(\alpha_t,\beta_t)}.
\end{equation}
We also use the previously proved global gradient bounds
\begin{equation}\label{eq:global-grad-bounds-ext-proof}
|\partial_\alpha\mathcal L(\alpha,\beta)|\le 10,
\qquad
|\partial_\beta\mathcal L(\alpha,\beta)|\le 10
\qquad\text{for all }(\alpha,\beta),
\end{equation}
and the previously established exterior gradient bounds: for every \(\delta>0\), there exist constants
\[
A_\delta\ge 1,\qquad c_\delta>0,\qquad C_\delta>0
\]
such that whenever
\[
\alpha\ge A_\delta,
\qquad
\lambda\coloneqq -\frac{\beta}{\alpha}\ge 2+\delta,
\]
we have
\begin{equation}\label{eq:exterior-grad-lemma-ext-proof}
0\le \partial_\beta\mathcal L(\alpha,\beta)\le C_\delta e^{-\delta\alpha},
\qquad
\partial_\alpha\mathcal L(\alpha,\beta)\le -\,c_\delta\,\alpha^{-d},
\qquad
\frac{1}{N+1}\partial_\beta\mathcal L(\alpha,\beta)
+
\frac{\lambda}{d}\partial_\alpha\mathcal L(\alpha,\beta)
\le -\,c_\delta\,\alpha^{-d}.
\end{equation}
Fix \(\varepsilon>0\). We shall prove that
\begin{equation}\label{eq:eventual-upper-threshold}
\lambda_t<2+\varepsilon
\qquad\text{for all sufficiently large }t.
\end{equation}
Since \(\varepsilon>0\) is arbitrary, this will imply
\[
\limsup_{t\to\infty}\lambda_t\le 2.
\]

\medskip
\noindent
\textbf{Step 1: exterior negative drift above \(2+\varepsilon/2\).}

Apply \eqref{eq:exterior-grad-lemma-ext-proof} with \(\delta=\varepsilon/2\). Thus there exist constants
\[
A_{\varepsilon/2}\ge 1,
\qquad
c_{\varepsilon/2}>0
\]
such that, whenever
\[
\alpha\ge A_{\varepsilon/2},
\qquad
\lambda\ge 2+\frac{\varepsilon}{2},
\]
we have
\begin{equation}\label{eq:exterior-alpha-bound-step1}
\partial_\alpha\mathcal L(\alpha,\beta)\le -\,c_{\varepsilon/2}\,\alpha^{-d},
\end{equation}
and
\begin{equation}\label{eq:exterior-combo-bound-step1}
\frac{1}{N+1}\partial_\beta\mathcal L(\alpha,\beta)
+
\frac{\lambda}{d}\partial_\alpha\mathcal L(\alpha,\beta)
\le -\,c_{\varepsilon/2}\,\alpha^{-d}.
\end{equation}
Since \(\alpha_t\to+\infty\), there exists \(T_1(\varepsilon)\) such that
\[
\alpha_t\ge A_{\varepsilon/2}
\qquad\text{for all }t\ge T_1(\varepsilon).
\]
Because \eqref{eq:exterior-alpha-bound-step1} implies \(\partial_\alpha\mathcal L(\alpha_t,\beta_t)\le 0\) whenever \(\lambda_t\ge 2+\varepsilon/2\), the denominator in \eqref{eq:lambda-recursion-ext-proof} satisfies
\[
\alpha_t-\frac{\eta}{d}\,\partial_\alpha\mathcal L(\alpha_t,\beta_t)\ge \alpha_t.
\]
We also claim that, after increasing \(T_1(\varepsilon)\) if necessary, we have
\begin{equation}\label{eq:denominator-upper-comparable}
\lambda_t\ge 2+\frac{\varepsilon}{2},\ t\ge T_1(\varepsilon)
\quad\Longrightarrow\quad
\alpha_t-\frac{\eta}{d}\,\partial_\alpha\mathcal L(\alpha_t,\beta_t)\le 2\alpha_t.
\end{equation}
Indeed, by \eqref{eq:global-grad-bounds-ext-proof},
\[
-\frac{\eta}{d}\,\partial_\alpha\mathcal L(\alpha_t,\beta_t)
\le \frac{\eta}{d}\,|\partial_\alpha\mathcal L(\alpha_t,\beta_t)|
\le \frac{10\eta}{d}.
\]
Since \(\alpha_t\to+\infty\), for all sufficiently large \(t\) we have \(10\eta/d\le \alpha_t\). Hence
\[
\alpha_t-\frac{\eta}{d}\,\partial_\alpha\mathcal L(\alpha_t,\beta_t)\le \alpha_t+\frac{10\eta}{d}\le 2\alpha_t,
\]
which proves \eqref{eq:denominator-upper-comparable}.

Combining \eqref{eq:lambda-recursion-ext-proof}, \eqref{eq:exterior-combo-bound-step1}, and \eqref{eq:denominator-upper-comparable}, we obtain: for all \(t\ge T_1(\varepsilon)\),
\begin{equation}\label{eq:strict-negative-drift}
\lambda_t\ge 2+\frac{\varepsilon}{2}
\quad\Longrightarrow\quad
\lambda_{t+1}-\lambda_t
\le
-\frac{\eta c_{\varepsilon/2}}{2}\,\alpha_t^{-d-1}
<0.
\end{equation}
Likewise, the \(\alpha\)-update gives, for all \(t\ge T_1(\varepsilon)\),
\begin{equation}\label{eq:alpha-growth-exterior}
\lambda_t\ge 2+\frac{\varepsilon}{2}
\quad\Longrightarrow\quad
\alpha_{t+1}-\alpha_t
=
-\frac{\eta}{d}\,\partial_\alpha\mathcal L(\alpha_t,\beta_t)
\ge
\frac{\eta c_{\varepsilon/2}}{d}\,\alpha_t^{-d}.
\end{equation}

\medskip
\noindent
\textbf{Step 2: there is no tail staying above \(2+\varepsilon/2\).}

We claim that there does not exist \(T\ge T_1(\varepsilon)\) such that
\begin{equation}\label{eq:no-high-tail-assumption}
\lambda_t\ge 2+\frac{\varepsilon}{2}
\qquad\text{for all }t\ge T.
\end{equation}
Assume, toward a contradiction, that such a \(T\) exists. Then \eqref{eq:alpha-growth-exterior} gives, for all \(t\ge T\),
\[
\alpha_{t+1}-\alpha_t \ge \frac{\eta c_{\varepsilon/2}}{d}\,\alpha_t^{-d}.
\]
Set
\[
a_{\varepsilon}\coloneqq \frac{\eta c_{\varepsilon/2}}{d}>0.
\]
Since the function \(x\mapsto x^{d+1}\) is convex on \((0,\infty)\), we have
\[
\alpha_{t+1}^{d+1}-\alpha_t^{d+1}
\ge
(d+1)\alpha_t^d(\alpha_{t+1}-\alpha_t)
\ge
(d+1)a_{\varepsilon}
\qquad\text{for all }t\ge T.
\]
Summing from \(T\) to \(t-1\), we obtain
\[
\alpha_t^{d+1}
\ge
\alpha_T^{d+1}+(d+1)a_{\varepsilon}(t-T),
\qquad t\ge T.
\]
Hence there exists a constant \(C_\varepsilon>0\) such that
\begin{equation}\label{eq:alpha-polynomial-growth-tail}
\alpha_t\ge C_\varepsilon (1+t)^{1/(d+1)}
\qquad\text{for all }t\ge T.
\end{equation}
On the other hand, Lemma~\ref{lem:exterior-gradient-bounds} with \(\varepsilon/2\) yields a constant
\[
\widetilde C_{\varepsilon}>0
\]
such that, for all \(t\ge T\),
\[
0\le \partial_\beta \mathcal L(\alpha_t,\beta_t)\le \widetilde C_{\varepsilon}e^{-(\varepsilon/2)\alpha_t}.
\]
Therefore
\[
0\le \beta_t-\beta_{t+1}
=
\frac{\eta}{N+1}\,\partial_\beta \mathcal L(\alpha_t,\beta_t)
\le
\frac{\eta \widetilde C_{\varepsilon}}{N+1}e^{-(\varepsilon/2)\alpha_t}.
\]
Using \eqref{eq:alpha-polynomial-growth-tail}, we get
\[
0\le \beta_t-\beta_{t+1}
\le
\frac{\eta \widetilde C_{\varepsilon}}{N+1}
\exp\bigl(-(\varepsilon/2)C_\varepsilon (1+t)^{1/(d+1)}\bigr).
\]
The right-hand side is summable in \(t\). Hence
\[
\sum_{t=T}^{\infty}(\beta_t-\beta_{t+1})<\infty,
\]
so the monotone sequence \((\beta_t)_{t\ge T}\) converges to a finite limit
\[
\beta_\infty>-\infty.
\]
This contradicts the result that \(\beta_t\to-\infty\).
The contradiction proves that no tail satisfying \eqref{eq:no-high-tail-assumption} can exist.

\medskip
\noindent
\textbf{Step 3: complete proof of the small-step statement below \(2+\varepsilon\).}

Suppose \(\lambda_t\le 2+\varepsilon\). Then by \eqref{eq:global-grad-bounds-ext-proof},
\[
\left|
\frac{1}{N+1}\partial_\beta\mathcal L(\alpha_t,\beta_t)
+
\frac{\lambda_t}{d}\partial_\alpha\mathcal L(\alpha_t,\beta_t)
\right|
\le
\frac{1}{N+1}|\partial_\beta\mathcal L(\alpha_t,\beta_t)|
+\frac{\lambda_t}{d}|\partial_\alpha\mathcal L(\alpha_t,\beta_t)|
\le
\frac{10}{N+1}+\frac{10(2+\varepsilon)}{d}.
\]
Substituting this into the exact recursion \eqref{eq:lambda-recursion-ext-proof}, when \(\alpha_t>10\eta/d\) we obtain
\[
\lambda_t\le 2+\varepsilon
\quad\Longrightarrow\quad
|\lambda_{t+1}-\lambda_t|
\le
\frac{\eta\left(\frac{10}{N+1}+\frac{10(2+\varepsilon)}{d}\right)}
{\alpha_t-\frac{\eta}{d}\,\partial_\alpha\mathcal L(\alpha_t,\beta_t)}.
\]
Again using \eqref{eq:global-grad-bounds-ext-proof},
\[
\alpha_t-\frac{\eta}{d}\,\partial_\alpha\mathcal L(\alpha_t,\beta_t)\ge \alpha_t-\frac{10\eta}{d}.
\]
Therefore
\begin{equation}\label{eq:conditional-small-step-raw}
\lambda_t\le 2+\varepsilon
\quad\Longrightarrow\quad
|\lambda_{t+1}-\lambda_t|
\le
\frac{\eta\left(\frac{10}{N+1}+\frac{10(2+\varepsilon)}{d}\right)}
{\alpha_t-\frac{10\eta}{d}}.
\end{equation}
Since \(\alpha_t\to+\infty\), there exists \(T_0(\varepsilon)\) such that for all \(t\ge T_0(\varepsilon)\),
\[
\alpha_t>\frac{10\eta}{d},\quad
\frac{\eta\left(\frac{10}{N+1}+\frac{10(2+\varepsilon)}{d}\right)}
{\alpha_t-\frac{10\eta}{d}}
<\frac{\varepsilon}{2}.
\]
Combining this with \eqref{eq:conditional-small-step-raw}, we conclude that for all \(t\ge T_0(\varepsilon)\),
\begin{equation}\label{eq:conditional-small-step-final}
\lambda_t\le 2+\varepsilon
\quad\Longrightarrow\quad
|\lambda_{t+1}-\lambda_t|<\frac{\varepsilon}{2}.
\end{equation}

\medskip
\noindent
\textbf{Step 4: once the trajectory drops below \(2+\varepsilon\) late enough, it never returns.}

Set
\[
T_\varepsilon\coloneqq \max\{T_0(\varepsilon),T_1(\varepsilon)\}.
\]
We claim that if \(t\ge T_\varepsilon\) and
\[
\lambda_t<2+\varepsilon,
\]
then automatically
\begin{equation}\label{eq:automatic}
\lambda_{t+1}<2+\varepsilon.
\end{equation}
To prove this, distinguish two cases. If
\[
\lambda_t\le 2+\frac{\varepsilon}{2},
\]
then \eqref{eq:conditional-small-step-final} gives
\[
\lambda_{t+1}
\le
\lambda_t+|\lambda_{t+1}-\lambda_t|
<
\left(2+\frac{\varepsilon}{2}\right)+\frac{\varepsilon}{2}
=
2+\varepsilon.
\]
If instead
\[
2+\frac{\varepsilon}{2}<\lambda_t<2+\varepsilon,
\]
then \(\lambda_t\ge 2+\varepsilon/2\), so \eqref{eq:strict-negative-drift} applies and yields
\[
\lambda_{t+1}<\lambda_t<2+\varepsilon.
\]
Thus in either case,
\[
t\ge T_\varepsilon,\ \lambda_t<2+\varepsilon
\quad\Longrightarrow\quad
\lambda_{t+1}<2+\varepsilon,
\]
which proves \eqref{eq:automatic}. By induction on \(n\), \eqref{eq:automatic} immediately implies:
\begin{equation}\label{eq:absorbing-below-threshold}
t\ge T_\varepsilon,\ \lambda_t<2+\varepsilon
\quad\Longrightarrow\quad
\lambda_s<2+\varepsilon
\qquad\text{for all }s\ge t.
\end{equation}

\medskip
\noindent
\textbf{Step 5: conclude eventual boundedness by \(2+\varepsilon\).}

Assume, toward a contradiction, that there are infinitely many indices \(t\ge T_\varepsilon\) such that
\begin{equation}\label{eq:final-invariance}
\lambda_t\ge 2+\varepsilon.
\end{equation}
Choose one such index \(t_\ast\ge T_\varepsilon\). Step 2 shows that the trajectory cannot stay forever in the region \([2+\varepsilon/2,\infty)\). Hence there exists \(s\ge t_\ast\) such that
\[
\lambda_s<2+\frac{\varepsilon}{2}<2+\varepsilon.
\]
Since \(s\ge t_\ast\ge T_\varepsilon\), the absorbing property \eqref{eq:absorbing-below-threshold} applies and yields
\[
\lambda_t<2+\varepsilon
\qquad\text{for all }t\ge s.
\]
This means that after time \(s\), the trajectory can never again return to \([2+\varepsilon,\infty)\). But \eqref{eq:final-invariance} asserts that there are infinitely many such return times \(t\ge T_\varepsilon\), hence in particular some after \(s\), which is impossible. This contradiction proves that only finitely many \(t\) satisfy \(\lambda_t\ge 2+\varepsilon\).

Therefore
\[
\lambda_t<2+\varepsilon
\qquad\text{for all sufficiently large }t.
\]
This establishes \eqref{eq:eventual-upper-threshold}. Since \(\varepsilon>0\) was arbitrary,
\[
\limsup_{t\to\infty}\lambda_t\le 2.
\]
Finally, we already know that
$
\liminf_{t\to\infty}\lambda_t\ge 2.
$
Hence
$
\lambda_t\to 2.
$
\end{proof}

\begin{theorem}\label{thm:ct-square-to-zero-final}
We have
\[
\mathbb E\!\left[a_0(\alpha^{(t)},\beta^{(t)})^2\right]\to 0
\qquad\text{as }t\to\infty.
\]
\end{theorem}

\begin{proof}
Suppose, toward a contradiction, that
\[
\mathbb E\!\left[a_0(\alpha^{(t)},\beta^{(t)})^2\right]\not\to0.
\]
Then there exist \(\varepsilon>0\) and a subsequence \(t_k\to\infty\) such that
\begin{equation}\label{eq:ct-square-to-zero-final-lower}
\mathbb E\!\left[a_0(\alpha^{(t_k)},\beta^{(t_k)})^2\right]\ge \varepsilon
\qquad\text{for every }k.
\end{equation}
Passing to a further subsequence if necessary, we may assume
\begin{equation}\label{eq:ct-square-to-zero-final-lambda-subseq}
\lambda_{t_k}\to \lambda_\ast\in[0,\infty].
\end{equation}
By Theorem~\ref{thm:liminf-lambda-ge-2-final},
\begin{equation}\label{eq:ct-square-to-zero-final-lambda-ge-two}
\lambda_\ast\ge2.
\end{equation}
Lemma~\ref{lem:subseq-ct-limit-final} therefore gives
\[
a_0(\alpha^{(t_k)},\beta^{(t_k)})\to \mathbf 1_{\{1-U_\pi>\lambda_\ast\}}
\qquad\text{almost surely}.
\]
Since \(1-U_\pi\in(0,2)\) almost surely and \(\lambda_\ast\ge2\),
\[
\mathbf 1_{\{1-U_\pi>\lambda_\ast\}}=0
\qquad\text{almost surely}.
\]
Hence
\[
a_0(\alpha^{(t_k)},\beta^{(t_k)})\to0
\qquad\text{almost surely}.
\]
Because
$
0\le a_0(\alpha^{(t_k)},\beta^{(t_k)})^2\le1,
$
bounded convergence yields
$
\mathbb E\!\left[a_0(\alpha^{(t_k)},\beta^{(t_k)})^2\right]\to0,
$
contradicting \eqref{eq:ct-square-to-zero-final-lower}. Therefore
\[
\mathbb E\!\left[a_0(\alpha^{(t)},\beta^{(t)})^2\right]\to0.
\]
\end{proof}

\begin{theorem}\label{thm:population-loss-to-zero-final}
Under Assumption \ref{ass:distribution}, suppose that
$
d\ge 2$, $ N\ge 100\sqrt d$,  $0<\eta\le {1}/{840(N+1)^2}.
$
Then, under the conditions of Theorem \ref{thm:training convergence},
\[
\mathcal L(\alpha^{(t)},\beta^{(t)})\to0
\qquad\text{as }t\to\infty,
\]
or equivalently,
\begin{equation*}
    \mathcal{L}(\Wb^{(t)})\to0
\qquad\text{as }t\to\infty.
\end{equation*}
\end{theorem}

\begin{proof}
By Theorem~\ref{thm:ct-square-to-zero-final} we have
\[
\mathbb E\!\left[a_0(\alpha^{(t)},\beta^{(t)})^2\right]\to0.
\]
Using the exact loss reduction \eqref{eq:exact-loss-reduction-final},
\[
\mathcal L(\alpha^{(t)},\beta^{(t)})
=
\frac32\,\mathbb E\!\left[a_0(\alpha^{(t)},\beta^{(t)})^2\right]+o(1),
\]
we conclude that
\[
\mathcal L(\alpha^{(t)},\beta^{(t)})\to0.
\]
Equivalently, by definition, we have
\begin{equation*}
    \mathcal{L}(\Wb^{(t)})\to0
\qquad\text{as }t\to\infty.
\end{equation*}
\end{proof}

\section{Additional Experimental Results}

\subsection{Joint Training}\label{app:joint_training}

In Section \ref{sec:numerical experiment}, we fix the value matrix $\Vb$ according to the construction in Theorem~\ref{thm:expressive power} and optimize only the attention matrix $\Wb$. In this subsection, we provide an additional experiment in which both $\Wb$ and $\Vb$ are trained jointly from random initialization.
We use the same experimental setup as in the main experiment, except that both $\Wb$ and $\Vb$ are randomly initialized and optimized jointly.

\noindent\textbf{Learned Parameter Structure.}
Figures~\ref{fig:joint_v_heatmap} and \ref{fig:joint_w_heatmap} show the learned matrices $\Vb$ and $\Wb$ after training. Although $\Vb$ is no longer fixed, the learned $\Vb$ still matches the construction in Theorem~\ref{thm:expressive power}. In particular, all entries are close to zero except for the entry corresponding to the $y$ coordinate and the diagonal entries corresponding to the positional coordinates. Moreover, these nonzero entries are all close to $1$, which is consistent with the theoretical construction.

The learned $\Wb$ also recovers the construction in Theorem~\ref{thm:expressive power}. Specifically, the diagonal entries in the data coordinates are positive, whereas the diagonal entries in the positional coordinates are negative, with the remaining entries close to zero. Furthermore, the magnitude of the negative diagonal entries is approximately twice that of the positive diagonal entries, again matching the predicted construction. These results show that even when $\Vb$ is trained jointly with $\Wb$, the learned transformer still converges to the same parameter structure as predicted by the theory.

\begin{figure*}[t]
    \centering

    \begin{subfigure}[t]{0.31\textwidth}
        \centering
        \includegraphics[width=\linewidth]{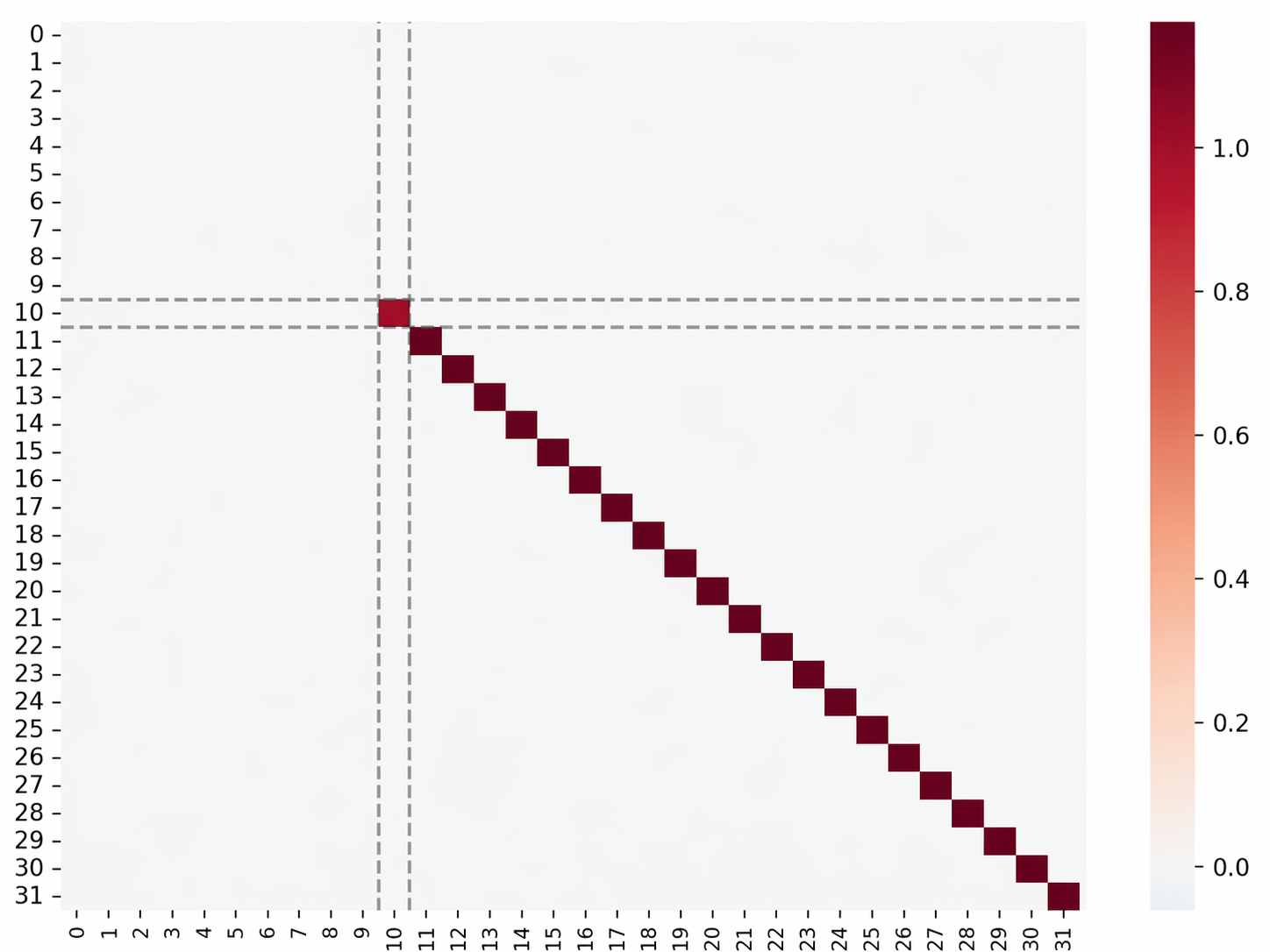}
        \caption{\textbf{Learned $\Vb$}}
        \label{fig:joint_v_heatmap}
    \end{subfigure}
    \hfill
    \begin{subfigure}[t]{0.31\textwidth}
        \centering
        \includegraphics[width=\linewidth]{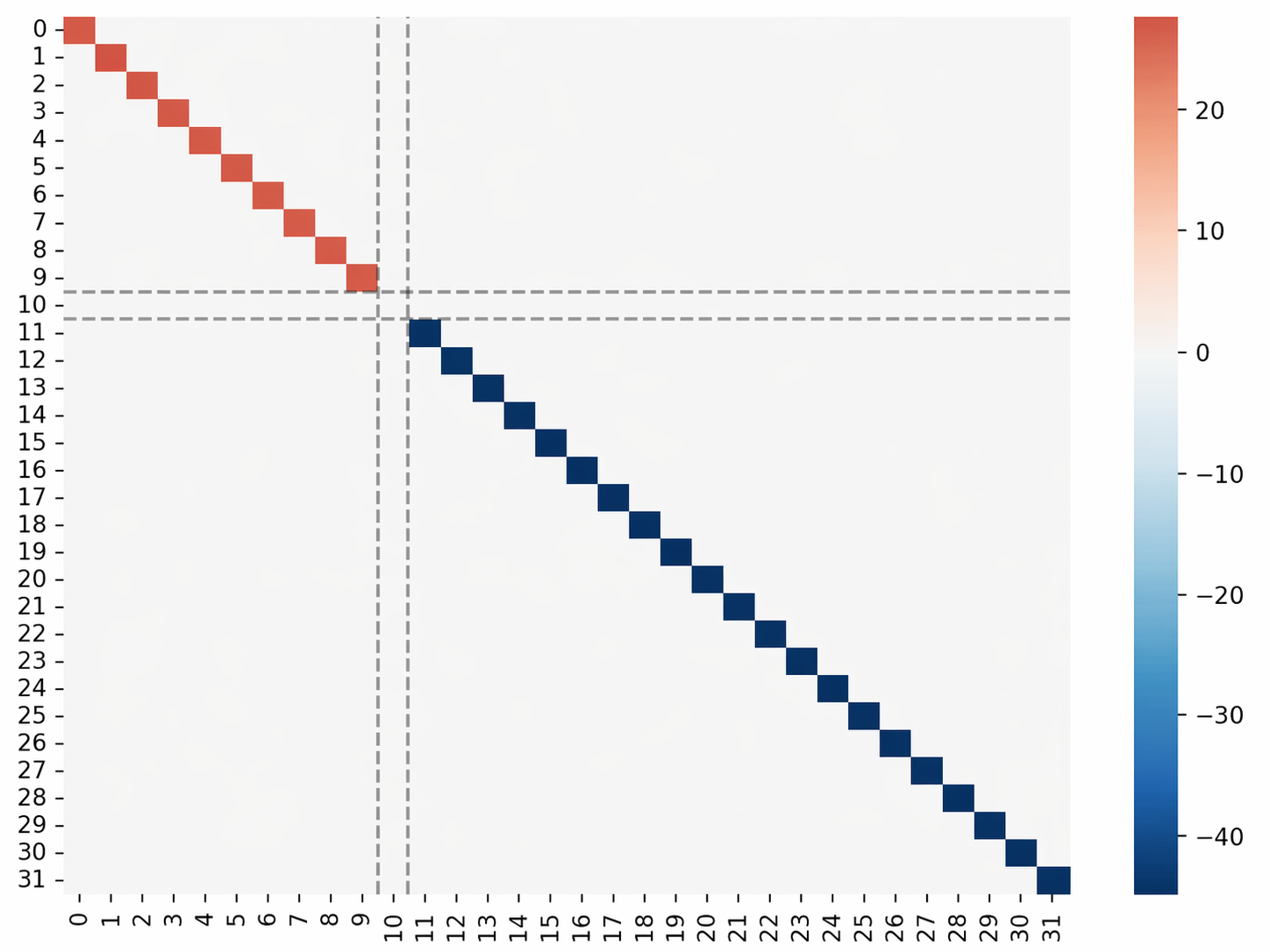}
        \caption{\textbf{Learned $\Wb$}}
        \label{fig:joint_w_heatmap}
    \end{subfigure}
    \hfill
    \begin{subfigure}[t]{0.36\textwidth}
        \centering
        \includegraphics[width=\linewidth]{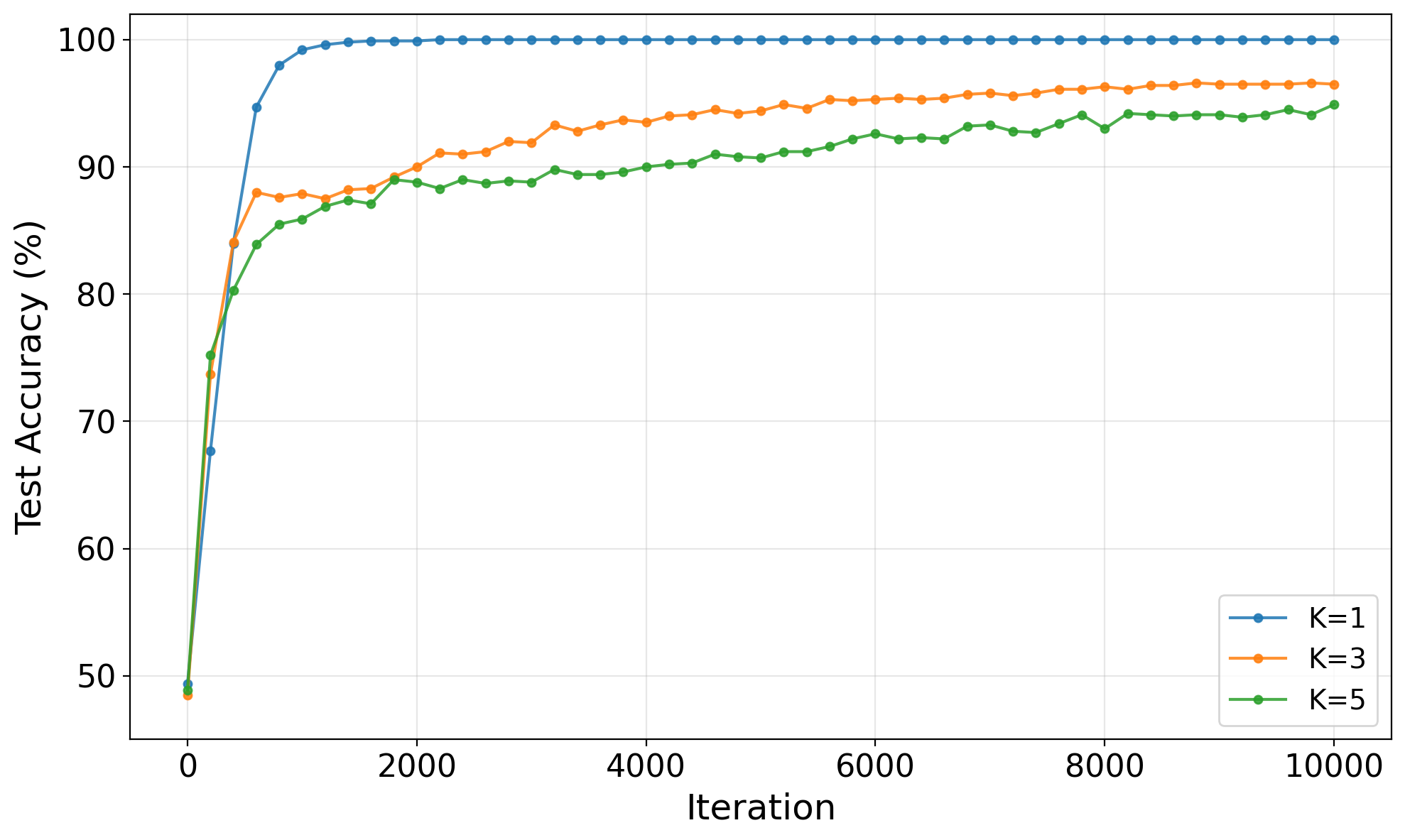}
        \caption{\textbf{Test accuracy}}
        \label{fig:joint_test_accuracy}
    \end{subfigure}

    \caption{\textbf{Joint training of $\Wb$ and $\Vb$.}
    We train both matrices jointly from random initialization. The learned $\Vb$ and $\Wb$ recover the construction in Theorem~\ref{thm:expressive power}, and the resulting transformer achieves high test accuracy on the $\mathrm{K}$-NN task.}
    \label{fig:joint_training_results}
\end{figure*}

\noindent\textbf{Test Accuracy.}
We further evaluate the jointly trained transformer on the $\mathrm{K}$-NN task for $K=1,3,5$. As shown in Figure~\ref{fig:joint_test_accuracy}, the test accuracy converges to nearly 1 across different values of $K$. This demonstrates that the jointly trained model not only recovers the construction in Theorem~\ref{thm:expressive power}, but also retains the multi-step CoT behavior.

\subsection{Direct-inference}\label{app:direct_inference}

In this subsection, we further examine direct inference using transformers with additional heads or layers and compare their performance with CoT. We also report results for the one-head, one-layer transformer with $N=40$ and $d=10$, complementing the corresponding results with $N=20$ reported in Section~\ref{sec:numerical experiment}.

Specifically, we set $N=40$ and $d=10$ and consider three architectures: a one-head, one-layer transformer, a four-head, one-layer transformer, and a one-head, two-layer transformer. For each architecture, we train a separate direct-inference baseline for each $K\in\{1,3,5\}$ for $10000$ iterations. During training, each model receives only the initial input and is supervised to predict the  final $\mathrm{K}$-NN target state $
\qb^{(K)}=\big[\xb_q^\top,\ \sum_{i=1}^{K} y_{\pi(i)},\ \pb_{N+1}^\top+\sum_{i=1}^{K} \pb_{\pi(i)}^\top\big]^\top
$ directly in a single forward pass.

At inference time, each trained direct-inference model performs only a single forward pass on the initial input, without intermediate-state generation or autoregressive rollout. We evaluate each direct model on the test task corresponding to the value of $K$ for which it was trained. In addition, for the models trained with $K=3$ and $K=5$, we also evaluate their predictions on the $K=1$ test set. This additional evaluation allows us to examine whether models trained for larger $K$ nevertheless retain behavior close to $1$-NN prediction.

We summarize the new experiments and compare them with the CoT results in table \ref{tab:cot_vs_direct}.

\begin{table}[t]
    \centering
    \caption{Test accuracy of CoT and direct models across different
    values of $K$.}
    \label{tab:cot_vs_direct}
    \small
    \setlength{\tabcolsep}{5pt}
    \begin{tabular}{llccc}
        \toprule
        Model & Training / Inference
        & $K=1$ & $K=3$ & $K=5$ \\
        \midrule
        
        \multirow{4}{*}{1-head, 1-layer}
        & CoT / CoT rollout
        & 100.0 & 99.8 & 98.3 \\
        & Direct ($K=1$) / Single-pass
        & 100.0 & -- & -- \\
        & Direct ($K=3$) / Single-pass
        & 100.0 & 74.7 & -- \\
        & Direct ($K=5$) / Single-pass
        & 99.6 & -- & 68.8 \\
        
        \midrule

        \multirow{4}{*}{1-head, 2-layer}
        & CoT / CoT rollout
        & 100.0 & 99.1 & 92.8 \\
        & Direct ($K=1$) / Single-pass
        & 100.0 & -- & -- \\
        & Direct ($K=3$) / Single-pass
        & 100.0 & 74.7 & -- \\
        & Direct ($K=5$) / Single-pass
        & 97.9 & -- & 69.2 \\
        \midrule

        \multirow{4}{*}{4-head, 1-layer}
        & CoT / CoT rollout
        & 100.0 & 99.7 & 98.2 \\
        & Direct ($K=1$) / Single-pass
        & 100.0 & -- & -- \\
        & Direct ($K=3$) / Single-pass
        & 100.0 & 75.1 & -- \\
        & Direct ($K=5$) / Single-pass
        & 99.9 & -- & 68.5 \\
        \bottomrule
    \end{tabular}
\end{table}

First, the results for the one-head, one-layer transformer with $N=40$ are consistent with those reported in Section \ref{sec:numerical experiment} for $N=20$. In particular, the direct-inference accuracies for $K=3$ and $K=5$ remain close to $75\%$ and $69\%$, respectively, while the corresponding CoT accuracies remain near perfect. Moreover, direct-inference learning of $\mathrm{K}$-NN remains difficult even with additional heads or layers: the direct models solve $1$-NN well, but their accuracies remain far below those of the CoT models for $K=3$ and $K=5$.

Second, the training of direct-inference models converges toward suboptimal $1$-NN-like solutions. Specifically, these models achieve nearly perfect $1$-NN accuracy but low $3$-NN and $5$-NN accuracies, even when supervised by $3$-NN or $5$-NN prediction rules. Since we use i.i.d. $\{\pm 1\}$ labels in the experiments, we can verify that the theoretical accuracies of a $1$-NN predictor on the $3$-NN and $5$-NN prediction tasks are $75\%$ and $68.75\%$, respectively. These values are close to our experimental results and explain the test accuracies achieved by the direct-inference models on the $3$-NN and $5$-NN tasks.


These results suggest that, during training, direct-inference models with additional heads or layers may continue to attend to the same context data point rather than retrieve multiple distinct neighbors. In contrast, CoT not only provides a natural sequential mechanism for the models to express $\mathrm{K}$-NN prediction rules, but also helps training successfully converge to such predictors. These results are consistent with both our expressive-power results and our training guarantees.

Although universal approximation results suggest that sufficiently large models may be able to represent the $\mathrm{K}$-NN prediction rule in principle, such results do not imply that gradient-based training will discover this mechanism, nor that the learned model will reliably generalize to test data. Such large-scale models are beyond the scope of this work.

\bibliography{ref}

\bibliographystyle{ims}

\end{document}